\documentclass[11pt]{article}
\usepackage[margin=1in]{geometry}
\usepackage[utf8]{inputenc}
\usepackage[T1]{fontenc}
\usepackage{textcomp}
\usepackage{amsmath,amssymb,amsthm,amsfonts}
\usepackage{booktabs}
\usepackage{graphicx}
\usepackage{xcolor}
\usepackage{array}
\usepackage{enumitem}
\usepackage{placeins}
\usepackage{tabularx}
\usepackage{xltabular}   
\newcolumntype{L}{>{\raggedright\arraybackslash}X}
\usepackage[hidelinks]{hyperref}
\usepackage[numbers,sort&compress]{natbib}
\usepackage{authblk}
\usepackage{caption}
\graphicspath{{figs/}}

\newtheorem{proposition}{Proposition}
\newtheorem{lemma}{Lemma}

\newtheorem{definition}{Definition}
\newtheorem{remark}{Remark}

\newcommand{\NESTED}{\ensuremath{\mathcal{L}_{\mathrm{nest}}}}
\newcommand{\PK}{\ensuremath{\mathcal{L}_{\mathrm{pk}}}}
\newcommand{\CF}{\textsc{cf}}
\newcommand{\MCF}{\textsc{mcf}}
\newcommand{\Ocf}{\ensuremath{\Omega_{\mathrm{cf}}}}
\newcommand{\Opk}{\ensuremath{\Omega_{\mathrm{pk}}}}
\newcommand{\Rze}[2]{\ensuremath{\mathcal{R}_{#1}(#2)}}
\definecolor{vExact}{HTML}{1B5E20}
\definecolor{vEst}{HTML}{2E7D32}
\definecolor{vNarrow}{HTML}{6B8E23}
\definecolor{vWeak}{HTML}{C55A11}
\definecolor{vNone}{HTML}{9E2A2B}
\newcommand{\chip}[2]{\colorbox{#1}{\textcolor{white}{\footnotesize\textbf{\strut #2}}}}

\newsavebox{\tblbox}

\newcommand{\match}{\ensuremath{\mathcal{A}_{\mathrm{match}}}}
\newcommand{\umfe}{\ensuremath{\mathcal{A}_{\mathrm{uMFE}}}}
\newcommand{\pfe}{\ensuremath{\mathcal{A}_{\mathrm{pfe}}}}

\title{\textbf{Auditing Discovery Claims:\\A Two-Sided Criterion for Agentic Science}\\[5pt]
\large With the negative side decidable, instantiated on RNA pseudoknot design,\\
where one fallible oracle credits $43$ designs and three predictors credit $1$}
\author[1]{Wenhui Chen\thanks{\texttt{mc35092@um.edu.mo}}}
\author[2]{Jianlin Chen\thanks{\texttt{202330450231@mail.scut.edu.cn}}}
\author[1]{Ziyao Lin\thanks{\texttt{mc35081@um.edu.mo}}}
\author[1]{Chi Man Vong\thanks{Corresponding author. \texttt{cmvong@um.edu.mo}}}
\affil[1]{University of Macau}
\affil[2]{South China University of Technology}
\date{}

\begin{document}
\maketitle

\begin{abstract}
\noindent When a self-improving AI-for-science system claims a new capability, the evidence is usually a
benchmark delta, a description-length gate, or a $p$-value. None separates a real gain from extra search, from
a changed verifier, or from adaptation to a fallible oracle. We build a two-sided audit whose negative side is
a formal fact. A pseudoknot-free oracle provably cannot represent a crossing base pair, so the prior
\emph{verifier's} range is bounded exactly, offline, before any run. Two conventions are weaker than the words
suggest: ``new'' is relative to the agent's \emph{prior self}, never to the base model, and ``solver-free''
means only that no packaged solver sits on the method path, since external folding oracles are queried
throughout.

\textbf{First, how far a single fallible oracle can inflate a capability claim.} An invented, solver-free
operator solves $43/60$ crossing RNA targets under the predictor it optimizes, above a context-free floor of
$0/60$; under three predictors, $1/60$ survives. The discriminating evidence is paired on the \emph{same} $43$
targets. A predictor the operator never saw confirms $2$ of its designs against $26$ for a
minimum-free-energy solver ($p{=}8\!\times\!10^{-7}$). No statistic computed from the system and its own
oracle sees that gap, and a compression gate or a benchmark delta is such a statistic.

\textbf{Second, that agent-written procedures can beat a human-written one under a judge no objective can
flatter, at a fraction of the compute.} Of six frontier models given only primitives, \emph{the two whose
operators ran without timeouts} were replicated at depth. On targets they and our operator both solved under
the same in-loop predicate they carry over at $0.293$ against our $0.095$ ($n{=}951$ paired units,
target-clustered $[{+}0.108,{+}0.297]$, $p{=}5\!\times\!10^{-5}$), while spending $4.6$--$10\times$
\emph{fewer} oracle calls. The claim covers \emph{those two operators}, not agent-written code as a class,
since the other four were excluded by a selection we cannot fully audit. Three rungs: difference under an
outside adjudicator (reached), not bought with compute (reached, in both directions), mechanism identified
and transferable (not reached; seven candidates tested, none moves the statistic).

\textbf{The ceiling on all of it is the panel itself.} Its three predictors share nearest-neighbour
thermodynamic parameters, and two agree at $\kappa{=}0.673$ on our own calibration set, so ``held out'' means
held out of the optimization loop and \emph{not} mechanistically independent. A predictor from a different
model class is the experiment this study most owes. The audit is as unsparing about our own system. Matched
undirected search is an exact zero, a search-free probe puts $84\%$ of our headline effect on targets a random
sequence already solves, and the honest description of our operator is a stochastic search with a better
success predicate and no better guidance. That is the instrument working on its author.
\end{abstract}

\section{Introduction}\label{sec:intro}
Self-improving agents claiming new scientific or design capabilities are proliferating, but the field
adjudicates the claim with soft statistics: a reward-model score, a held-out delta, an MDL gate, a $p$-value.
Such a number rises whether the agent acquired something or merely searched harder, and it rises equally when
the \emph{verifier} changed underneath. Rarely separated: memorized targets, luck, a borrowed external tool,
or an over-optimized fallible oracle. What is missing is a criterion that is two-sided, reporting ``can now''
and ``could not be \emph{certified} before'' as distinct claims, and decidable on at least one side. Which side is which matters. Only the statement about the prior \emph{verifier's} range is exact; the statement
about the prior \emph{policy} is an empirical floor at a finite budget, and we never upgrade the second to the
first.

\paragraph{Why RNA.} In RNA inverse design the target's topology is a position in the Chomsky hierarchy:
nested pairing is context-free, while \emph{selected unbounded families} of crossing (pseudoknotted)
structures are mildly context-sensitive \citep{rivaseddy2000,nebelweinberg,katosekikasami2006}. The
pseudoknot-free folding oracle cannot represent a crossing pair, which turns ``the prior repertoire could not
be certified through its pseudoknot-free regime'' into a decidable statement, not a failed search; the
stronger claim, that the prior policy produces no crossing design under a crossing-aware adjudicator, is a
separate empirical floor ($0/60$).

\emph{The difficulty axis is not ours to choose.} Where a target sits is a property of its topology,
fixed by \citet{rivaseddy2000,nebelweinberg,katosekikasami2006} before this system or this benchmark existed.
So the negative side cannot be moved by selecting a favourable target set, tuning a threshold, or recomputing
the criterion from an improved system's own fit. A description-length or
information-criterion gate is computed from the revised system and the evidence it was fitted on, and has no
such external anchor (\S\ref{sec:related}).

We claim no more, and two concessions follow. The exclusion is \emph{shallow}:
$\operatorname{im}\Ocf\subseteq\NESTED$ can be stated from the oracle's output space alone. And ``crossing''
here is a \emph{structural} predicate on a target, decidable in $O(|P|^2)$; the mildly-context-sensitive reading
attaches to the family $\mathcal L_{\mathrm H}$ and never to a design, so a reader who substitutes
``pseudoknot-free versus pseudoknotted'' throughout loses the family-level statement and nothing else.

\emph{What that family-level statement turned out to be worth} is the reason a $100$-target test set exists.
Pseudobase++ is structurally exhausted at $24$ development-disjoint clusters, and the grid of \S\ref{sec:exp}
is constructible only because $\mathcal L_{\mathrm H}$ is a formally characterized unbounded family whose
members are non-context-free by construction. The hierarchy's dividend is not depth in the exclusion but that
it makes the only unbounded axis in this study constructible and interpretable.

\paragraph{Thesis.} \emph{Under a pseudoknot-aware oracle, a self-improving design agent realizes held-out
crossing targets via a reusable, solver-free procedure, while its prior \CF-verified repertoire is confined by
its pseudoknot-free verifier to nested structures and does not realize the targets empirically ($0/60$). We
audit this as an empirical capability-acquisition event of the oracle-augmented system, reporting which
components it establishes and which it does not; the class-level $\CF\to\MCF$ reading is evidence from an
unbounded crossing family, not a language-class proof, since any finite target set is regular
(\S\ref{sec:formal}). The negative side pairs a decidable verifier-range confinement with a separate empirical
floor; the positive side is model-relative, adjudicated by an oracle panel.}

\paragraph{Which ``new''?} A self-improving agent runs on a base model that has read the literature, so a
capability may be latent even when the agent appears to acquire it. We do not claim novelty relative to the
base model: that notion is neither measurable nor necessary, since landmark results advance a frontier without
proving out-of-corpus originality. We certify an operational ``new'', relative to the agent's prior self,
whose repertoire provably could not be \emph{certified} to realize the target through its pseudoknot-free
regime and does not realize it empirically, with no packaged solver called. We do \emph{not} write
``robustly'': the one genuinely held-out predictor separates weakly, matched undirected search is not
separated, cross-source consistency is $0/60$, execution noise is of the same order as the method effect, and
the evaluation split is not structurally disjoint (Table~\ref{tab:verdict}).

\paragraph{Contributions.}
\begin{enumerate}[leftmargin=1.3em,itemsep=1.5pt]
\item A two-sided criterion for capability acquisition (\S\ref{sec:cert}). It does \emph{not} rule out search;
what it controls for is target memorization and one-off luck, matched-cap undirected search, delegation to a
packaged solver, and mistaking one fallible oracle's verdict for a capability.
\item A formalization of the object a pumping lemma can constrain, the realizable-target language
$\Rze{\Omega}{\pi}$, with a decidable confinement fact over it (Prop.~\ref{prop:confine}) discharged offline
rather than searched (\S\ref{sec:formal}). We \emph{invoke} the MCFG-2 characterization and verify only the
non-context-freeness of the family we construct ourselves, since no citation covers it.
\item An explicit treatment of the audit's asymmetry (\S\ref{sec:cert}, \S\ref{sec:oracle}): the negative side
has an exact verifier-range component and a separate empirical prior-policy floor; the positive side is
heuristic and oracle-dependent, and is adjudicated as such. Two impossibilities are kept apart,
repertoire-confinement and target-undesignability (\S\ref{sec:impossible}).
\item An empirical demonstration that the criterion catches oracle-gaming (\S\ref{sec:exp}): an invented,
solver-free operator scores $43/60$ under a single pseudoknot oracle and $1/60$ panel-unanimous, above a
context-free floor of $0/60$. Re-optimizing for cross-oracle agreement recovers $3$--$8/60$.
\item \textbf{A two-stage attribution gate, and evidence that its first stage is not enough}
(\S\ref{sec:exp}). Capability differences in learned systems are normally attributed by finding a property that differs between
the systems and correlates with the outcome. We ran that step under composition-adjusted controls, then
\emph{intervened} on every property that survived. Four times a statistically supported and mechanistically
plausible attribution was wrong, and a control fixed in advance caught it; the standard procedure alone would
have yielded two independent false mechanisms with clean $p$-values.
\item \textbf{A positive result on the same instrument, on three named rungs} (\S\ref{sec:exp},
\S\ref{sec:disc}). An audit nothing passes and an audit that cannot see produce the same table, so we measure an arm that does
clear the bar. Frozen LLM-written operators carry over at $0.293$ against the hand-built $0.095$ (paired
${+}0.199$, target-clustered $[{+}0.108,{+}0.297]$, $p{=}5\!\times\!10^{-5}$), at fewer oracle calls per
target: \emph{D1 and D2 established, D3 open}. The criterion we position against (\S\ref{sec:related})
reaches neither D1 nor D2 on its own reported evidence, and D3 is reached by neither line.
\end{enumerate}

\noindent Figure~\ref{fig:arch} shows how these compose into one pipeline. One caveat about its last step:
the $\pm$solver ablation cannot decide whether ``the ascent is the agent's own'', since the system leans
heavily on external crossing-aware folding oracles throughout and the ablation speaks only to the declared
solver class (Rem.~\ref{rem:autonomy}).

\begin{figure}[t]\centering
\includegraphics[width=\linewidth]{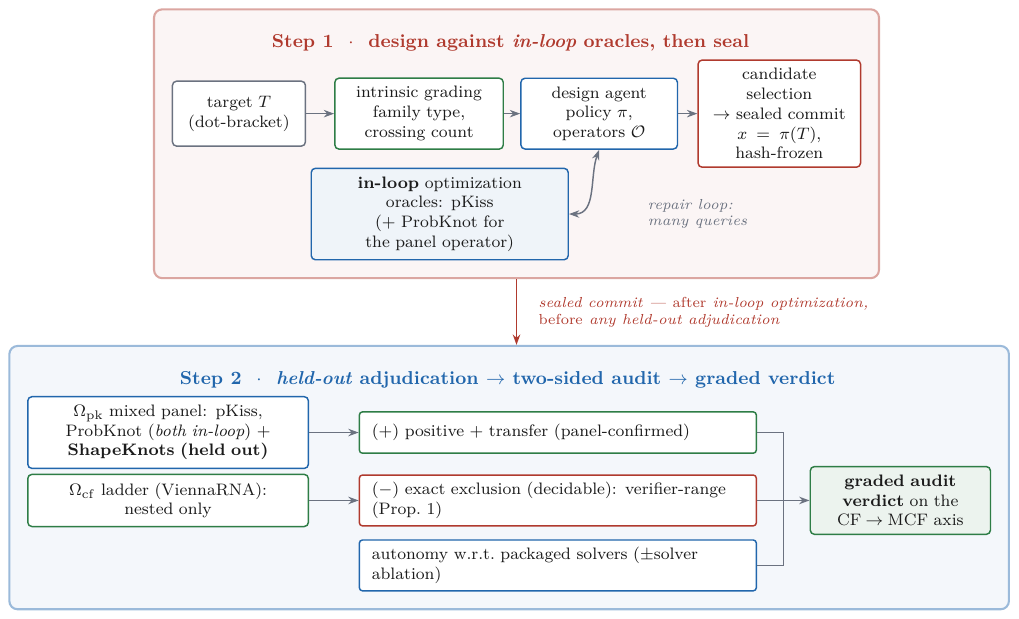}
\caption{\textbf{End-to-end architecture.} A target is graded by family type and crossing-stem count and assigned to a
pool of Table~\ref{tab:pools}. Step~1: the design agent queries crossing-aware oracles (pKiss, and ProbKnot for
the panel operator) inside its repair loop, then commits a candidate hash-frozen. The seal falls \emph{after}
in-loop optimization and \emph{before} any held-out adjudication, so it forecloses post-hoc selection on the
held-out oracle but not overfitting to the in-loop ones. Step~2: the pseudoknot-free ladder $\Ocf$ and the
mixed panel $\Opk$ adjudicate, abstaining on disagreement. The audit reports six findings
(Table~\ref{tab:verdict}), never a single bit. Only the exclusion is exact.}\label{fig:arch}
\end{figure}

\paragraph{Scope: three results, weighted.} Three things to take from this paper, and it is worth saying
plainly which carry how much. \textbf{(1) The instrument, and the $43\!\to\!1$ measurement it produces}
(\S\ref{sec:exp}). On the $60$-target development/evaluation pool a $43/60$ single-oracle rate collapses to
$1/60$ under three predictors, above a $0/60$ context-free floor, with the discriminating evidence paired on a
common $43$ targets. This is what we are most confident in. \textbf{(2) The positive control that turns into a
positive result} (Table~\ref{tab:rungs}). An agent-written procedure beats the human-written one under a
predictor outside every loop, at a fraction of the oracle calls, on a paired target-clustered comparison. We
report it at D1 and D2 and withhold D3. It is not secondary. It is the only evidence here that the panel can
be cleared at all, and therefore the only thing separating a stringent audit from a blind one. \textbf{(3) The
audit's verdict on our own system}, largely negative and reported in full. It separates from the gamer only
narrowly, matched undirected search is an exact zero, and $84\%$ of the headline difference sits where a
random sequence already succeeds. We regard (3) as the instrument working on its author, not as a result
withheld.

Two things are \emph{not} claimed: (i) a robust autonomous ascent by \emph{our} operator, since at $1/60$
panel-unanimous the instrument only exposes an oracle artifact; (ii) any scientific discovery of a new design
principle, by anyone here, which by our own standard requires D3 (Table~\ref{tab:rungs}) and which no arm
reaches.

\paragraph{The nearest criterion, and where the two differ.} The closest programme is categorical
self-revising discovery \citep{selfrevising}, an arXiv preprint at the time of writing, which formalizes
discovery as the residual beyond a functorial
transport of prior artifacts and gates it on an MDL/AIC criterion. We state the comparison here, in the terms
of Table~\ref{tab:rungs}, because it is the reason this paper is shaped as an audit; \S\ref{sec:related}
gives it in full. An information criterion is computed from the revised system and the evidence that system
was fitted on. It therefore has no predictor held outside the loop, which is exactly the quantity D1
requires and exactly the failure mode our $43\!\to\!1$ collapse exhibits, since a system can improve a statistic of
its own fit while a predictor it never optimized sees almost nothing. No compute control is reported, so D2
is not addressed. And by its authors' own statement the residual measures the \emph{additional bits} a
revision needs, not the prior regime's incapacity, so there is no counterpart to the exact exclusion of
Prop.~\ref{prop:confine}. We intend this as a difference in \emph{reported evidence}, not in ambition. A
categorical account of regime change does different work from an audit, and D3 is reached by neither line. The claim we do make is narrow and checkable. On the rungs where evidence can be demanded, this paper
supplies a held-out adjudicator, a compute control, matched baselines, ablations, intervals, and its own two
unestablished components.

\paragraph{Classification rule, and how to count the components.} \emph{Exact}: proved, discharged offline.
\emph{Established}: the effect's $95\%$ interval excludes zero at the deepest resampling level the component's
design contains. \emph{Weak}: directionally consistent, but the interval or the replication does not support
it. \emph{Not established}: the interval spans zero. By that rule this study returns \textbf{one exact, three
established, one weak, one not established}, and every count stated elsewhere in this paper is this one.

Two caveats on applying it. Where a component admits more than one nesting variant we name which is primary
before seeing the answer: for $\Delta_{\mathrm{gamer}}$ the primary estimator draws all
three executions inside each drawn seed, because that is the design that was run, and the
one-execution-per-seed variant is a conservative sensitivity. We hold no timestamped pre-registration, so
``fixed in advance'' is our word and the choice between variants is a degree of freedom a reader cannot audit;
the resolution we would defend is not the choice but the deterministic re-run that removes it, which has no
execution level to nest (Table~\ref{tab:det}). And the count should not be read without the second question
Table~\ref{tab:verdict} answers, flagged \textsc{no counterfactual} and \textsc{weak counterfactual}: whether a
component contrasts against an arm that could have come out otherwise. Two do not. $E_{\mathrm{range}}$ is a
theorem and $g_{\mathrm{aut}}$ records that a solver was never on the method path, so neither can be surprised
by data; $\Delta_{\mathrm{prior}}$ is weak in this sense because the prior arm optimizes a different objective.
Only $\Delta_{\mathrm{gamer}}$, $\Delta_{\mathrm{undirected}}$ and $\Delta_{\mathrm{held\text{-}out}}$ are
counterfactual contrasts, and of those one is narrow, one null and one weak. The rungs of
Table~\ref{tab:rungs} use the same words for a different arm and are not part of this count.

\section{Background and the Chomsky--RNA Correspondence}\label{sec:bg}
\paragraph{Formal languages and the Chomsky hierarchy.} A grammar generates a language of strings; the
Chomsky hierarchy \citep{chomsky1956,hopcroftullman1979} orders grammars by generative power. Context-free
grammars (CFGs) generate exactly the languages recognizable with one stack, of which balanced-bracket (Dyck)
languages are canonical. The pumping and Ogden lemmas \citep{ogden1968,hopcroftullman1979} give necessary
conditions for context-freeness, so a word family that cannot be pumped proves a language is not context-free.
Multiple context-free grammars (MCFGs) \citep{seki1991mcfg} sit just above, deriving tuples of substrings at
several positions, the machinery for crossing dependencies one stack cannot track, with dimension~2
sufficing for the patterns we treat; the canonical witness is the count language
$\{\,a^{n}b^{m}c^{n}d^{m}\,\}$, whose interleaved constraints force two substrings to grow in lockstep at
separated positions \citep{hopcroftullman1979}. \textbf{We invoke this literature rather than extend it}, and
use exactly two facts from it, both established elsewhere: the crossing family we target is non-context-free
\citep{nebelweinberg}, and specific pseudoknot classes sit in low-dimension MCFG subclasses
\citep{katosekikasami2006}. We make no claim that every pseudoknot topology is captured at dimension~2, and
nothing below turns on the MCFG machinery beyond those two citations. Their role is to make the family of
\S\ref{sec:exp} well defined, not to prove anything new.

\paragraph{RNA secondary structure, folding, inverse design.} An RNA is $x\in\{A,C,G,U\}^n$; a
\emph{secondary structure} is a set $P$ of base pairs (each index paired at most once). Pairs
$(i,j),(k,l)$ with $i<k$ \emph{cross} iff $i<k<j<l$; $P$ is \emph{nested} if none cross,
\emph{pseudoknotted} otherwise. A folding oracle maps a sequence to a structure under a thermodynamic model,
from base-pair maximization \citep{nussinov1980} through minimum-free-energy dynamic programming
\citep{zuker1981} and the equilibrium partition function \citep{mccaskill1990} to ViennaRNA
\citep{lorenz2011}. \textbf{We fix \Ocf{} as a named function, not as a software package}: it is
ViennaRNA's minimum-free-energy fold (\texttt{RNA.fold}, Turner parameters), whose dot-bracket output
alphabet is $\{\texttt{(},\texttt{)},\texttt{.}\}$. This matters for the scope of Prop.~\ref{prop:confine} and
we would rather state it here than have it read as an oversight. The same distribution ships
\texttt{RNAPKplex}, which \emph{does} emit crossings; it is never called anywhere on any path in this work,
and if a prior system's verifier were \texttt{RNAPKplex} then $\operatorname{im}\Omega\not\subseteq\NESTED$
and the exclusion would simply not apply to it. That is the correct behaviour of a scoped theorem, not a gap. The negative side is a statement about \emph{a declared verifier}, and an audit that let the verifier be
renegotiated after the fact would establish nothing. The \emph{empirical} prior-policy floor is unaffected
either way, since it is measured on the prior system as it actually ran. So \Ocf{} is pseudoknot-free: its
output space contains only nested structures ($\operatorname{im}\Ocf\subseteq\NESTED$), and it cannot represent a
crossing pair. Inverse design (given $T$, find $x$ with $\mathrm{fold}(x)=T$) is mature, with heuristic
\citep{busch2006inforna}, ensemble-defect \citep{zadeh2011nupack}, learned \citep{runge2019learna},
ensemble-optimization \citep{zhou2023samfeo}, and pseudoknot-capable \citep{wirecki2025desirna}
solvers. Pseudoknot-aware prediction oracles \Opk{}, pKiss \citep{janssen2015pkiss}, ProbKnot
\citep{bellaousov2010probknot}, and ShapeKnots \citep{hajdin2013shapeknots} in the lineage of
\citep{rivas1999pknots}, output crossings but are heuristic, not biophysical ground truth.
Pseudoknot-free folding is polynomial; pseudoknotted folding is NP-complete
\citep{lyngso2000pknot}, and inverse design is NP-hard even in the minimal
base-pair-maximization model \citep{bonnet2020hard}. The sequence$\to$structure map is many-to-one, with large
neutral networks \citep{schuster1994,aguirre2011}; undesignability (no sequence realizes $T$) is provable by
exhibiting, for every candidate, a rival structure at least as favorable \citep{zhou2024rigende,zhou2025motifs}.

\paragraph{The verifier ladder (a reused instrument).} Our companion work \citep{verifiereconomics} turns
``solved'' into a nested family of decidable predicates on one target,
$\match\supseteq\umfe\supseteq\pfe$, under \Ocf, with a separate pseudoknot-aware branch for crossing
targets. We reuse it as the adjudicating instrument and stake an orthogonal thesis (capability-class
acquisition, not verifier over-acceptance). It also measured the pseudoknot verdict's oracle-dependence: the
fraction of state-of-the-art pseudoknot designs judged ``formed'' swings across predictors
(pKiss/ProbKnot/ShapeKnots) over a $[\,15,60\,]\%$ envelope, which we carry into our positive-side design
(\S\ref{sec:oracle}).

\subsection{RNA Topology Is Native to the Chomsky Hierarchy}\label{sec:cf}
Reading biological sequence structure as a formal language is a long tradition \citep{searls1992,searls2002}.
Nested arcs are balanced brackets, context-free and generated by the RNA stochastic context-free grammar
$S\to(\,S\,)\,S\mid\mathord{.}\,S\mid\varepsilon$ \citep{rivaseddy2000,durbin1998}, with SCFGs the standard
nested-RNA model \citep{sakakibara1994,knudsenhein1999,dowelleddy2004} and thermodynamic pseudoknot folders
operating over the mildly context-sensitive class \citep{reedergiegerich2004}. Crossing arcs are provably not
context-free (an Ogden-lemma argument on the interleaved dependencies rules out any CFG \citep{nebelweinberg})
and are captured by an MCFG of dimension $\le 2$; the canonical pseudoknot grammar of \citet{rivaseddy2000} is
exactly such \citep{katosekikasami2006}. The pseudoknot-free/pseudoknot split a design pipeline already treats
as two verifier branches therefore straddles the $\CF\!\to\!\MCF$ boundary \emph{for the families we treat}; we do not claim the whole crossing set $\PK$ falls in one
fixed low-dimension MCFG
class.

\paragraph{Why one stack fails, concretely.} Reading a strand $5'\!\to\!3'$, the nested grammar
$S\to(\,S\,)\,S\mid\mathord{.}\,S\mid\varepsilon$ folds with a single stack: each opening base is pushed and
its partner pops it, and well-nesting guarantees the most-recently-opened pair closes first. A canonical
H-type pseudoknot violates this. In dot-bracket form with two pair classes, $(((([[[[))))]]]]$, the
square-bracket stem opens while the round-bracket pairs are still unclosed and must be remembered past the
closing of the round stem, requiring a second, independent stack. This is the copy-language pattern, not the
Dyck pattern: an Ogden-lemma argument rules out every CFG \citep{nebelweinberg}, whereas an MCFG carrying the
two stems as a coordinated pair of substrings generates it at dimension~2 \citep{katosekikasami2006}. The two
verifier branches a design pipeline already maintains thus straddle a generative-class boundary between the
nested family and the non-CF H-type family, developed formally in \S\ref{sec:formal}.

\paragraph{Generative power is the main line; folding complexity is a distinct axis.} The Chomsky hierarchy
measures generative power (which structures a grammar produces); computational complexity measures the cost of
parsing and folding. Pseudoknotted folding is NP-complete and inverse design is NP-hard
\citep{lyngso2000pknot,bonnet2020hard}, but that hardness is a \emph{distinct} axis, depending on the
allowed pseudoknot topology, energy model, and optimization objective, and is not implied by the generative
class (fixed fan-out MCFGs, for instance, parse in polynomial time). Our sole main-line quantity is the
generative jump $\CF\!\to\!\MCF$; we treat the two axes as related but independent.
\section{The Two-Sided Capability Audit}\label{sec:cert}
Each target receives an intrinsic difficulty grade (crossing-stem count and nested-vs-crossing type), fixed
independently of what the agent solves, which avoids survivorship bias. Crossing targets (members of the MCFG-characterized crossing families) split by
grade into the pools named in Table~\ref{tab:pools}.

\begin{definition}[Audited capability acquisition]\label{def:cert}
\emph{``Capability acquisition'' names the class of claim this instrument grades, not a claim we make.} By our
own standard (Table~\ref{tab:rungs}) no arm in this paper reaches D3, so nothing here is asserted to be an
acquired design principle; what is asserted is a measured difference and the alternatives it excludes.
\emph{The audit returns a vector of six findings, not a bit.} A pass/fail conjunction would be the wrong
output: a single threshold on a positive rate is passed by one lucky target, and encodes neither the
separation from a gaming baseline, nor the separation from matched undirected search, nor spread, nor
uncertainty, nor held-out-oracle generalization. We therefore define the audit's output as the evidence vector
\begin{equation}\label{eq:vector}
  \mathcal A=\bigl(\,E_{\mathrm{range}},\ \Delta_{\mathrm{prior}},\ \Delta_{\mathrm{gamer}},\
  \Delta_{\mathrm{undirected}},\ \Delta_{\mathrm{held\text{-}out}},\ g_{\mathrm{aut}}\,\bigr),
\end{equation}
whose components are of \emph{different epistemic kinds}: one is a theorem, the rest are effect sizes with
intervals. A verdict is a statement about which components are established, never a single bit. The phrase
$\CF\!\to\!\MCF$ names a difficulty axis over target \emph{families}; a positive finding is a per-target
verified event, not a proof that the policy's realizable-target \emph{language} changed Chomsky class. A
finite set of realized targets is regular, so no finite experiment can establish that (\S\ref{sec:formal}). Underlying all four
$\Delta$'s is the held-out positive rate itself. On held-out crossing targets the system must produce designs
a crossing-aware predictor panel confirms fold to the target's crossing pair set, reported with its
uncertainty and its spread across targets, never as a single threshold.

\begin{equation}\label{eq:cert}
\begin{aligned}
E_{\mathrm{range}} &= \bigl[\,\Rze{\Ocf}{\pi_{\mathrm{before}}}\cap\PK=\varnothing\,\bigr]
  &&\text{exact (Prop.~\ref{prop:confine}), discharged before any run}\\
\Delta_{\mathrm{prior}} &= \widehat p^{\mathrm{panel}}_{\mathrm{after}}-\widehat p^{\mathrm{panel}}_{\mathrm{before}}
  &&\text{before/after under one \emph{shared} adjudicator}\\
\Delta_{\mathrm{gamer}} &= \widehat p^{\mathrm{panel}}_{\mathrm{after}}-\widehat p^{\mathrm{panel}}_{\mathrm{gamer}}
  &&\text{vs.\ single-oracle gaming, matched access}\\
\Delta_{\mathrm{undirected}} &= \widehat p^{\mathrm{panel}}_{\mathrm{after}}-\widehat p^{\mathrm{panel}}_{\mathrm{undir}}
  &&\text{vs.\ matched-nominal-cap undirected search}\\
\Delta_{\mathrm{held\text{-}out}} &= r_{\mathrm{ho}}^{\mathrm{after}}-r_{\mathrm{ho}}^{\mathrm{gamer}}
  &&\text{on the one genuinely held-out predictor}\\
g_{\mathrm{aut}} &= \widehat p_{\mathrm{after}}-\widehat p^{-\mathrm{solver}}_{\mathrm{after}}
  &&\text{verified non-use of the solver class }\mathcal T .
\end{aligned}
\end{equation}
Only $E_{\mathrm{range}}$ is decidable and discharged before any run; the five $\Delta$'s are measured,
model-relative, and each is reported with an effect size and an interval. \textbf{A study may establish some
components and fail to establish others, as ours does.} The audit's job is to say which, not to collapse
them. Table~\ref{tab:verdict} is this paper's own filled-in $\mathcal A$, stated up front so no reader has to
reconstruct it from the results section.
\end{definition}

\paragraph{How to read a mostly-negative vector.} It is tempting to read Table~\ref{tab:verdict} as a report
card our system failed. The system audited here was \emph{built to pass this audit}: it optimizes a
two-oracle conjunction because a single oracle is gameable. It still leaves two of six components
unestablished or weak. But our own failure to saturate the instrument says nothing about the
instrument until some arm succeeds. The discriminating power is established by the positive control of
\S\ref{sec:exp}, not by our low score.

\paragraph{What each side controls for, and what none of them does.} The limit first, since it is the most
easily overclaimed. \emph{This audit does not separate capability from search.} The operator we audit
\emph{is} a stochastic search against folding oracles, and our own ablation attributes its gain to
oracle-guided repair (Table~\ref{tab:det}). Nor should the two be opposed. A finite-description search
procedure that reliably realizes held-out family members is a perfectly good candidate capability. What the
audit controls for is narrower, and each component excludes a named alternative: target memorization and
one-off luck (only partly, since our split is disjoint by target identity and not by structure), matched-cap
undirected search, single-oracle gaming, and delegation to a packaged solver. No single component supports the
word \emph{acquisition}, and neither does their conjunction when some are unestablished, so the output is the
vector of Eq.~\eqref{eq:vector} read component-wise. The structure parallels the retrieval/search/discovery
distinction of \citet{selfrevising}, with one amendment. We place our result at \emph{search}, and the audit's
job is to \emph{ask} whether that search is transferable, separated from undirected search and undelegated,
not to certify that it is. On this study only the last comes back positive, and only as verified non-use.

\paragraph{Terminology.} \textbf{Audit} names the whole protocol of Def.~\ref{def:cert}, never a single result.
Within it, \textbf{exact exclusion} is the verifier-range theorem (Prop.~\ref{prop:confine}), the only exact
component; \textbf{empirical gain} is the before/after difference under a shared adjudicator; \textbf{verified
non-use} is the solver-removal result, relative to a declared tool class and an audit of the call graph rather
than a counterfactual (Rem.~\ref{rem:autonomy}); \textbf{cross-predictor consistency} is what the panel
measures, never physical truth; and \textbf{scientific confirmation}, the wet-lab rung, is not attempted.
\textbf{We reserve ``certified'' for the exact exclusion alone}: a panel verdict is \emph{adjudicated} or
\emph{mixed-panel-consistent}, and no design here is called certified. For the same reason we write
\emph{mixed-panel-consistent crossing} and not ``panel-robust ascent'', since two of three predictors are inside
the operator's objective and the third has $\sim\!1\%$ native recall.

\paragraph{Three metrics, named and kept apart.} A single ``panel'' number conflates quantities of different
evidential value, so we name all three and report them separately throughout:
\begin{equation}\label{eq:metricnames}
\begin{aligned}
  r_{\mathrm{in}}  &:\ \text{pKiss}\wedge\text{ProbKnot \ (both \emph{in} the objective)}
                   &&\Rightarrow\ \text{in-objective consistency}\\
  r_{\mathrm{ho}}  &:\ \text{ShapeKnots alone \ (the one oracle held out)}
                   &&\Rightarrow\ \text{held-out-oracle generalization}\\
  r_{\mathrm{mix}} &:\ \text{all three \ (two in-loop, one held out)}
                   &&\Rightarrow\ \text{mixed in-loop/held-out}.
\end{aligned}
\end{equation}
For the panel operator, $r_{\mathrm{in}}$ scores its own objective and so cannot evidence generalization;
$r_{\mathrm{ho}}$ is the only genuine held-out-oracle rate and is \emph{weak} (\S\ref{sec:exp});
$r_{\mathrm{mix}}$, the rate we report as primary because it is the most stringent, is a
\emph{mixture} of two in-loop conjuncts plus one held-out, and must not be read as three-way independent
confirmation. Every headline number below is labelled with which of the three it is. Relatedly, \textbf{``held-out'' in this paper
means held out of the optimization \emph{loop}} (the ShapeKnots adjudicator), and never ``held out of
development''. The primary $n{=}60$ pool is a \emph{development/evaluation} pool: the operator and every
hyperparameter were fixed on it. Pools are named in Table~\ref{tab:pools} and we use those names.

\begingroup
\centering\footnotesize
\setlength{\abovecaptionskip}{4pt}
\setlength{\tabcolsep}{5pt}\renewcommand{\arraystretch}{1.2}
\begin{xltabular}{\linewidth}{@{}>{\raggedright\arraybackslash}p{2.5cm}>{\raggedright\arraybackslash}p{2.0cm}L L@{}}
\caption{\textbf{The audit's output for this study: the evidence vector $\mathcal A$ of Eq.~\eqref{eq:vector},
filled in.} All rates are $r_{\mathrm{mix}}$ (Eq.~\eqref{eq:metricnames}) unless stated. This is the verdict
on \emph{our own system}: \emph{partial audit success: diagnostic crossing evidence, not robust capability
acquisition}. Table~\ref{tab:rungs} reports the arm that does clear the bar, which this vector has no row for.
Each row is a pointer, not an argument: the classification rule is in \S\ref{sec:cert} and every number is
derived where it is cited.}\label{tab:verdict}\\
\toprule
\textbf{component} & \textbf{verdict} & \textbf{basis} & \textbf{what it does \emph{not} show}\\
\midrule
\endfirsthead
\multicolumn{4}{@{}l}{\footnotesize Table~\ref{tab:verdict}, continued.}\\
\toprule
\textbf{component} & \textbf{verdict} & \textbf{basis} & \textbf{what it does \emph{not} show}\\
\midrule
\endhead
$E_{\mathrm{range}}$ \newline exact exclusion & \chip{vExact}{exact} &
Prop.~\ref{prop:confine}, discharged offline before any run &
Bounds the prior \emph{verifier's} range, not the prior policy's capacity. \textsc{No counterfactual}: a
theorem cannot be surprised by data\\
\addlinespace[2pt]
$\Delta_{\mathrm{prior}}$ \newline prior floor & \chip{vEst}{established} &
Prior policy $0/60$ at \emph{every} one of $24$ executions; after, $1$--$9/60$ (mean $3.29$) under the shared
panel &
The quoted ${+}0.017$--${+}0.15$ is the \textbf{range of point estimates}, not a $95\%$ interval. So this row
rests on the floor being an exact zero, not on the interval rule. \textsc{Weak counterfactual}: the prior arm
optimizes a different objective\\
\addlinespace[2pt]
$\Delta_{\mathrm{gamer}}$ \newline vs.\ single-oracle gaming & \chip{vNarrow}{narrow} &
\emph{Established, narrowly.} ${+}0.067$ $[{+}0.023,{+}0.120]$ bit-reproducibly, gamer $0/60$ at all five seeds
(Table~\ref{tab:det}); operator strictly higher in $14/15$ wall-clock executions and $15/15$ on the
structurally disjoint pool (\S\ref{sec:exp}) &
A stable \emph{magnitude}. The estimate fell ${+}0.080\!\to\!{+}0.056\!\to\!{+}0.048$ as replication was
added, and determinism removes the execution level, not the target set. A search-free probe puts $84\%$ of it
where a random compatible sequence already succeeds, leaving ${+}0.016$. Between-system; matched on oracle
\emph{access}, not \emph{use} ($2.3\times$)\\
\addlinespace[2pt]
$\Delta_{\mathrm{undirected}}$ \newline vs.\ matched undirected search & \chip{vNone}{not est.} &
\emph{Not established.} Exactly ${+}0.0000$ on $300$ paired bit-reproducible cells, $[{-}0.060,{+}0.050]$ (Table~\ref{tab:det}) &
An underpowered null. Execution noise can no longer be blamed, so this is \emph{measured absent} at this pool
size. Deleting the objective does not lower $r_{\mathrm{mix}}$ at all; only deleting the repair loop does.
Matched on allowance, not on realized use\\
\addlinespace[2pt]
$\Delta_{\mathrm{held\text{-}out}}$ \newline held-out oracle & \chip{vWeak}{weak} &
$r_{\mathrm{ho}}$ higher at every seed but significant in $1$ of $3$ (\S\ref{sec:exp}) &
A small \emph{absolute} signal: against its own ceiling the judge confirms our designs at $8.0\%$ where it
recovers only $11.3\%$ of natives, against the gamer's $3.7\%$ (App.~\ref{app:adjud}). Run in an unvalidated
de-novo mode; the comparison is unstratified\\
\addlinespace[2pt]
$g_{\mathrm{aut}}$ \newline verified non-use of packaged solvers & \chip{vEst}{established} &
No packaged inverse-design solver is ever called &
Anything counterfactual: the solver was never on the method path, so removing it changes nothing. Solver-free
is \emph{not} oracle-free, and not evidence that the gain is the system's own\\
\bottomrule
\end{xltabular}
\par\endgroup
\medskip

\begingroup
\centering\footnotesize
\setlength{\abovecaptionskip}{4pt}
\setlength{\tabcolsep}{5pt}\renewcommand{\arraystretch}{1.2}
\begin{xltabular}{\linewidth}{@{}>{\raggedright\arraybackslash}p{2.5cm}>{\raggedright\arraybackslash}p{2.0cm}L L@{}}
\caption{\textbf{The instrument's other output, which the vector above cannot carry.}
Table~\ref{tab:verdict} is the audit's verdict on \emph{our} system, so by construction it has no row for the
arm that does clear the bar. That arm is the frozen LLM-written operators of App.~\ref{app:invent}, which
optimize the same pKiss$\,\wedge\,$ProbKnot predicate and are scored by the ShapeKnots predictor held out of
every arm's loop. We report them on the three independently reachable rungs of \S\ref{sec:disc}. It is
simultaneously the positive control that shows the panel is not merely insensitive and the sharpest bound on
our own result. Quantifier: these are the $2$ of $6$ invented operators with no runtime failures, so the
claim is about \emph{some} agent-written operators, not the class.}\label{tab:rungs}\\
\toprule
\textbf{rung} & \textbf{verdict} & \textbf{basis} & \textbf{what it does \emph{not} show}\\
\midrule
\endfirsthead
\multicolumn{4}{@{}l}{\footnotesize Table~\ref{tab:rungs}, continued.}\\
\toprule
\textbf{rung} & \textbf{verdict} & \textbf{basis} & \textbf{what it does \emph{not} show}\\
\midrule
\endhead
(D1) \newline difference under an adjudicator no objective can flatter & \chip{vEst}{established} &
Paired and difficulty-matched on the $951$ units \emph{both} arms solved under the shared predicate:
$0.293$ vs.\ $0.095$, difference ${+}0.199$, target-clustered $[{+}0.108,{+}0.297]$ over $42$ clusters,
cluster-permutation $p{=}5\!\times\!10^{-5}$. Re-derived per design through a second code path with zero
disagreements &
Correctness, and not a within-system contrast. Carry-over measures agreement with a partly correlated panel
(\S\ref{sec:oracle}). We quote the clustered pair, not the exact McNemar $p{=}5\!\times\!10^{-37}$ the same
counts give, since the $951$ units come from $60$ targets. One pool; the model-level selection above\\
\addlinespace[2pt]
(D2) \newline not bought with compute & \chip{vEst}{established} &
Measured in both directions. Tripling the pKiss allowance ($82.2\!\to\!196.2$ calls/target) moves carry-over
$0.108\!\to\!0.108$. And the agents' \emph{realized} use is $\approx\!28$ and $\approx\!13$ calls/target
against this operator's $128.0$; an arm held to that volume ($26.0$) carries over at $0.107$ where the agent
reaches $0.367$ &
That no resource explains it. Only that oracle-call volume does not, in either direction. Resources we did
not count remain open\\
\addlinespace[2pt]
(D3) \newline mechanism identified and transferable & \chip{vNone}{not est.} &
\emph{Not established, informatively.} \textbf{Seven} mechanisms transplanted one at a time,
deterministically at $5$ seeds: four legible in the source, a compute control in each direction, and \emph{two}
properties that differ between the arms at matched composition and predict or explain carry-over
observationally. Every interval contains the baseline; none reaches the agents' lower bound of $0.275$. The two
observational leads are the informative part: each would have been reported as the mechanism by an analysis
that stopped before intervening (\S\ref{sec:exp}) &
What the difference \emph{is}. We withhold the word ``discovery'' until this rung is reached, and note in
\S\ref{sec:related} that the criterion we position against reaches neither D1 nor D2\\
\bottomrule
\end{xltabular}
\par\endgroup
\medskip

\paragraph{The audit is asymmetric.} The negative side has two components of different epistemic status:
an \emph{exact}, decidable verifier-range bound (the representational limit of \Ocf{}, Prop.~\ref{prop:confine})
and a \emph{separate}, empirical prior-policy floor ($0/60$ under the shared crossing-aware panel). Only the
first is a theorem. The positive side is also not exact: \Opk{} is a
heuristic, model-relative predictor whose verdict is oracle-dependent (\S\ref{sec:bg},\ref{sec:oracle}). We
therefore claim a decidable negative side and a model-relative, panel-adjudicated positive side, not a
uniformly ``exact'' audit, and report the positive rate with its cross-oracle envelope
(\S\ref{sec:oracle}).

\paragraph{Sealed-commitment protocol.} We distinguish three oracle roles: the \emph{optimization} oracles the
operator queries inside its repair loop (pKiss, and ProbKnot for the panel operator), the \emph{selection} of a
final candidate, and the \emph{held-out adjudicator} (chiefly ShapeKnots), held out of every loop. For each
target the agent commits a hash-frozen prediction, either a candidate sequence or an explicit abstention claim
(\S\ref{sec:impossible}), before the held-out adjudicator is run, so the prediction is risky against an oracle
it never optimized and post-hoc selection there is foreclosed. Sealing does \emph{not} prevent overfitting to
the in-loop oracles; our $43/60\!\to\!1/60$ collapse is exactly such overfitting. \emph{So what does sealing buy
if the dominant failure is one it does not prevent?} The ability to see that failure. Without a seal an arm may
select among candidates using the held-out predictor, and any collapse it would have shown is selected away
before adjudication, making $43\!\to\!1$ unobservable in principle rather than merely unobserved. Sealing does not stop in-loop overfitting. It makes in-loop overfitting \emph{measurable}. This operationalizes a Lakatosian
risky prediction \citep{popper1959,lakatos1978}, the design-time analogue of pre-registration and of agentic
falsification loops \citep{huang2025popper,liu2024aigs}, but with an exact refuter, not a statistical one.

\paragraph{The certification procedure and its soundness.} Auditing an ascent requires each component to be computed by a stated procedure, not judged. $E_{\mathrm{range}}$ is discharged offline by
Prop.~\ref{prop:confine}: no search, no budget, no seed. Each $\Delta$ is a paired difference between two arms
scored by the \emph{same} adjudicator on the \emph{same} targets inside the \emph{same} execution, so
common-mode noise cancels; the pairing, resampling unit and interval type are fixed in App.~\ref{app:repro}.
$g_{\mathrm{aut}}$ is a $\pm$solver ablation over the declared tool class. Soundness is conditional on the
panel: a panel that errs systematically moves every positive component with it, so it abstains on disagreement, why one predictor is kept out of every loop, and why \S\ref{sec:exp} reports a positive control.
No procedure here converts a model-relative verdict into a physical one.

\noindent The verdict is a demarcation, not a probability, and it is component-wise:
$E_{\mathrm{range}}$ is a theorem (Prop.~\ref{prop:confine}), whereas the five $\Delta$'s are model-relative
measurements whose uncertainty we quantify (\S\ref{sec:exp}, Tables~\ref{tab:sep},~\ref{tab:det}) and two of
which this study does not establish (Table~\ref{tab:verdict}).

\begin{table}[t]\centering\small
\setlength{\tabcolsep}{4.5pt}\renewcommand{\arraystretch}{1.15}
\caption{\textbf{$r_{\mathrm{mix}}$ separated from single-oracle gaming}, matched design ($n{=}60$
development/evaluation pool; panel-unanimous $=$ pKiss\,$\wedge$\,ProbKnot\,$\wedge$\,ShapeKnots, all exact).
Gamer and panel operator run on the \emph{same} five seeds and targets. $b$ solved only by the panel operator,
$c$ only by the gamer; $\widehat\Delta_s$ is the per-seed paired risk difference. Note $c{=}0$ at every seed.
Single $60$-target seeds are underpowered, and this ${+}0.080$ is the \emph{first} of four estimates of
$\Delta_{\mathrm{gamer}}$; \S\ref{sec:exp} gives the sequence and what it does and does not settle.
(\texttt{stats\_separation.py}, \texttt{large\_pool\_stats.py}.)}\label{tab:sep}
\begin{tabularx}{\linewidth}{@{}Xcccc@{}}
\toprule
\textbf{seed} & \textbf{gamer $r_{\mathrm{mix}}$/60} & \textbf{panel $r_{\mathrm{mix}}$/60} & \textbf{paired McNemar ($b,c,p$)} & $\widehat\Delta_s$\\
\midrule
seed 0 & $1$ & $8$ & $b{=}7,c{=}0,p{=}0.016$ & $+0.117$\\
seed 1 & $0$ & $3$ & $b{=}3,c{=}0,p{=}0.25$ & $+0.050$\\
seed 2 & $0$ & $5$ & $b{=}5,c{=}0,p{=}0.0625$ & $+0.083$\\
seed 3 & $0$ & $5$ & $b{=}5,c{=}0,p{=}0.0625$ & $+0.083$\\
seed 4 & $0$ & $4$ & $b{=}4,c{=}0,p{=}0.125$ & $+0.067$\\
\midrule
pooled & $1/300$ & $25/300$ & \multicolumn{2}{l}{$\Delta{=}{+}.080$; CI $[{+}.040,{+}.127]$ / two-way $[{+}.030,{+}.140]$}\\
\bottomrule
\end{tabularx}
\end{table}

\subsection{The Positive Side Is Model-Relative}\label{sec:oracle}
Only the negative side's verifier-range bound is exact (Prop.~\ref{prop:confine}); its prior-policy floor is empirical, and the positive side is neither exact nor decidable. A ``formed pseudoknot'' verdict comes from a heuristic predictor
over a restricted pseudoknot class, not from biophysical ground truth, and it is oracle-dependent. Our
companion measured the fraction of state-of-the-art pseudoknot designs judged to form their target swinging
across a $[\,15,60\,]\%$ envelope over pKiss, ProbKnot and ShapeKnots, with only fair pairwise agreement
\citep{verifiereconomics}. That disagreement is mechanistic and not noise, since the three cover different
pseudoknot subclasses and optimize different objectives. pKiss enumerates H-type and kissing-hairpin motifs
under a heuristic folding model \citep{janssen2015pkiss}, ProbKnot assembles maximum-expected-accuracy pairs
from a partition function \citep{bellaousov2010probknot}, and ShapeKnots folds under experimental probing
constraints \citep{hajdin2013shapeknots}. A single-predictor rate is therefore class-biased, and a rate under
one lenient predictor would be an oracle artifact. Abstaining when the panel disagrees declines to convert a
modeling artifact into a claimed capability. We call the predictors \emph{distinct} and not statistically
independent, since all three are computational secondary-structure models sharing thermodynamic and
structural assumptions, so their errors can correlate. Panel-unanimity is cross-predictor consistency under
the declared computational panel, not independent physical confirmation.

\paragraph{How many independent voices does the panel have? (Measured, not argued.)} We measure pairwise
agreement on the $80$ native bpRNA pseudoknots of our calibration set, where each predictor's call is whether
folding a native's \emph{own} sequence recovers its known crossing (\texttt{panel\_agreement.py}). We use Cohen's $\kappa$, not raw agreement, because at these marginal rates (pKiss $20\%$, ProbKnot $4\%$,
ShapeKnots $11\%$) two predictors that mostly say ``no'' agree $\sim\!84\%$ of the time by construction. The
panel is \textbf{not three exchangeable voices}: pKiss and ShapeKnots agree at $\kappa=0.673$ $[0.422,0.875]$,
an interval clear of zero, so they behave as a correlated block, while the two pairings involving ProbKnot are
$\kappa=0.270$ $[0.000,0.524]$ and $\kappa=0.117$ $[-0.074,0.421]$. We resist reading this as ``two voices, not
three'': ProbKnot recovers only $3$ of $80$ natives, so every $\kappa$ involving it is poorly determined and
whether it is a genuinely independent third voice cannot be settled at $n{=}80$. The defensible statement is
weaker and still consequential: at least two of three members are substantially correlated, so unanimity buys
less independent confirmation than its arity suggests. This also gives the low $r_{\mathrm{mix}}$ a mechanism
instead of leaving it a bare number, since the three-way conjunction is bound largely by ProbKnot, which is
exactly where the gamer's designs die ($2/43$, Table~\ref{tab:contrast}).

\emph{Does that correlation undercut D1?} Not the comparison, and the reason is structural. Both arms in the D1
comparison optimize the \emph{same} predicate and the comparison is conditioned on both having satisfied it on
the same target, so whatever agreement the correlation hands out is handed to both arms equally and cancels in
the paired difference. What it does bound is the \emph{absolute} reading: carry-over measures agreement with a
partly correlated panel, not correctness, so $0.293$ is not an estimate of how often those designs really fold
as intended. The same distinction disposes of the mirror worry about $\Delta_{\mathrm{gamer}}$, where the two
arms optimize \emph{different} predicates so the correlation does not cancel, which is one more reason we treat
that component as narrow and this one as established.

\emph{Then why not make the least-correlated pair the primary metric?} The answer is structural, not a
preference. ProbKnot is the member least correlated with the others, so pKiss$\,\wedge\,$ProbKnot is indeed
the most mechanistically diverse pairing available. But that conjunction \emph{is} $r_{\mathrm{in}}$, the
panel operator's own success predicate, so for that arm it cannot serve as a held-out measure at all; scoring
an arm on its own objective is the failure this whole paper is about. Where ProbKnot \emph{is} held out we
already use it as primary, and it carries the paper's sharpest single result: the gamer optimizes pKiss alone,
so its $2/43$ against the minimum-free-energy solver's $26/43$ is measured on exactly that least-correlated
predictor (Table~\ref{tab:contrast}). The pairing the question asks for is therefore used wherever it is
available and unavailable for the one arm whose objective it is. What no rearrangement of three thermodynamic
predictors fixes is the underlying problem, which is why owed experiment~(3) asks for a predictor from a
different model family rather than a different conjunction of these.

We adjudicate the positive side accordingly. A held-out crossing target counts as realized only under an
explicit rule, and we report all three: (i) \emph{panel-unanimous} (all three predictors confirm the exact
crossing pair set) as the primary rate; (ii) the reference (pKiss) rate; and (iii) the cross-predictor
envelope $[\min,\max]$. Targets on which predictors disagree are abstained rather than counted either way.
The audit is thus asymmetric by construction: a negative side exact and decidable on its verifier-range component (plus an empirical floor) and a model-relative,
panel-adjudicated, abstaining positive side.

\section{Formalization: the Realizable-Target Language}\label{sec:formal}
A pumping lemma constrains languages of structures; an agent manipulates sequences. The object the lemma acts
on is therefore not the agent's operators but the set of target structures the agent realizes under an
oracle. This fixes a category error in a naive ``operators place pairs'' account: an operator edits a
sequence, and whether the result is nested or crossing is a property of the fold, decided by the oracle, not
of the operator.

\paragraph{Policies, oracles, and the realizable-target language.}
\begin{definition}[Structures, oracles, policies]\label{def:setup}
Fix $\Sigma=\{\mathrm A,\mathrm C,\mathrm G,\mathrm U\}$ and let $\mathcal T=\bigcup_n\mathcal T_n$ be the
dot-bracket structures, partitioned into the \emph{nested} language $\NESTED$ (no two base pairs cross) and
the \emph{crossing} language $\PK=\mathcal T\setminus\NESTED$. A \emph{folding oracle} is a length-preserving
map $\Omega:\Sigma^\ast\!\to\mathcal T$; a \emph{design policy} is a length-preserving map
$\pi:\mathcal T\!\to\Sigma^\ast$, realized by a search over a repertoire of reusable components $\mathcal O$.
The \emph{realizable-target language} of $\pi$ under $\Omega$ and the \emph{designable set} of $\Omega$ are
\begin{equation}\label{eq:realizable}
  \Rze{\Omega}{\pi}=\{\,T\in\mathcal T:\Omega(\pi(T))=T\,\},\qquad
  \mathcal D_\Omega=\operatorname{im}\Omega=\{\,T:\exists\,x\in\Sigma^\ast,\ \Omega(x)=T\,\}.
\end{equation}
$\Rze{\Omega}{\pi}$ is a language over dot-bracket strings, the object a Chomsky-class statement must range
over, since a pumping lemma constrains languages of \emph{structures}, not the sequence edits an operator
performs.
\end{definition}

\begin{lemma}[Realizability is capped by the oracle range]\label{lem:bound}
For every policy $\pi$ and oracle $\Omega$, $\ \Rze{\Omega}{\pi}\subseteq\mathcal D_\Omega=\operatorname{im}\Omega$,
independently of the policy's search budget or repertoire.
\end{lemma}
\begin{proof}
If $T\in\Rze{\Omega}{\pi}$ then $T=\Omega(\pi(T))$, so $T$ is in the image of $\Omega$; i.e.\
$T\in\operatorname{im}\Omega=\mathcal D_\Omega$. The bound quantifies over \emph{all} $\pi$, so no amount of
search enlarges $\Rze{\Omega}{\pi}$ beyond $\operatorname{im}\Omega$.
\end{proof}
\noindent Lemma~\ref{lem:bound} is the engine of the negative side: it makes confinement a property of the
oracle's representation, not of any operator, and reduces Prop.~\ref{prop:confine} to computing
$\operatorname{im}\Ocf$.

\paragraph{Where this object sits in designability theory.} $\Rze{\Omega}{\pi}$ is bounded above by the
designable structures under $\Omega$ (those admitting some sequence that folds to them), and inherits the
non-uniformity of the sequence-to-structure map: designable structures form vast neutral networks, while
many structures are realized by no sequence \citep{schuster1994,aguirre2011}. Its extreme is undesignability,
where no sequence realizes $T$ under the energy criterion, provable by construction \citep{zhou2024rigende}
and the complement of the designable set. Hence a protocol consequence (\S\ref{sec:exp}): a held-out
crossing pool must be pre-screened for designability, so an intrinsically unrealizable target is scored as an
impossibility (\S\ref{sec:impossible}), not a capability failure. The realizable-target language is thus the
intersection of what is designable in principle and what $\pi$ actually finds; an observed crossing gain concerns
the second factor, over targets the first admits (Fig.~\ref{fig:lang}).

\begin{figure}[t]\centering
\includegraphics[width=0.9\linewidth]{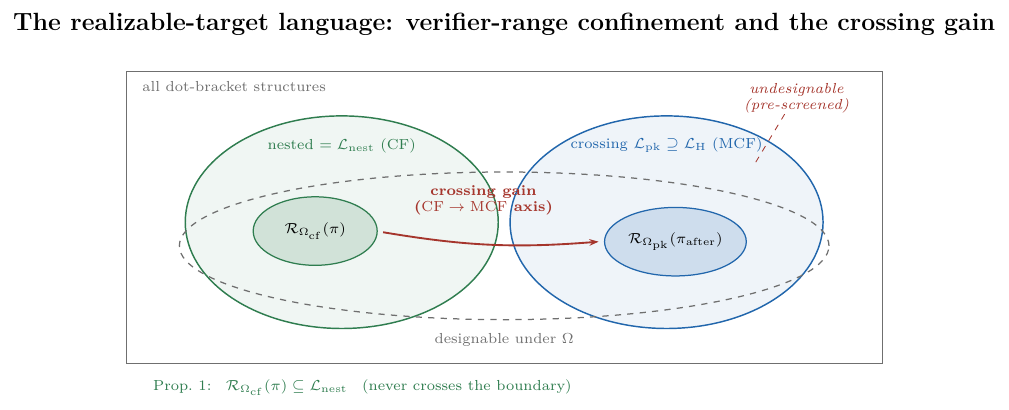}
\caption{\textbf{The realizable-target language, confinement, and the ascent.} Nested structures form a
context-free region $\NESTED$ and crossing structures a disjoint region $\PK$ (whose H-type sub-family
$\mathcal L_{\mathrm H}$ is non-\CF, strictly higher); the designable set (dashed) is the subset realized by
some sequence. Under the pseudoknot-free oracle, the
realizable-target language $\Rze{\Ocf}{\pi}$ is confined to $\NESTED$ (Prop.~\ref{prop:confine}) and cannot
enter $\PK$; an empirically observed crossing-family gain moves the measured $\Rze{\Opk}{\pi_{\mathrm{after}}}$ into the designable part of $\PK$,
while undesignable crossing targets are pre-screened out.}\label{fig:lang}
\end{figure}

\begin{proposition}[Decidable \CF-confinement]\label{prop:confine}
The pseudoknot-free oracle satisfies $\operatorname{im}\Ocf\subseteq\NESTED$ (every single-structure output it
returns is nested, admitting no crossing pair; the image is in fact the strictly smaller \emph{designable}-nested subset, but $\subseteq$ is
all the negative side needs). Hence, by Lemma~\ref{lem:bound}, $\Rze{\Ocf}{\pi}\subseteq\NESTED$ for \emph{every} policy $\pi$, and
for every crossing target $T\in\PK$ we have $T\notin\Rze{\Ocf}{\pi}$. This non-membership is
\emph{decidable} in $O(|P|^{2})$ (test whether any two pairs of $T$ cross; App.~\ref{app:proofs}) and is
independent of the policy's search budget.
\end{proposition}
\noindent The proof (App.~\ref{app:proofs}) is immediate from the oracle's output space: the negative side
of Def.~\ref{def:cert} is a theorem discharged before any run, not a failed search.

\paragraph{What the Chomsky framing buys, and what it does not.} It buys three things: an
\emph{externally fixed} difficulty axis that we did not choose to flatter the result, a decidable negative
side, and an unbounded family over which extrapolation is meaningful. It does \emph{not} buy four others, and
we name them because each has been read into work of this kind. (i) It does not show the agent's policy
changed Chomsky class: any finite set of realized designs is regular, so no finite experiment can. (ii) It
does not make the positive side exact, since crossing realization is adjudicated by heuristic predictors, and only
the exclusion is a theorem. (iii) It does not license ``the agent learned a grammar'': we measure realized
targets, not induced rules. (iv) It does not make the crossing-pair count a grammar dimension; the count is a
structural covariate that happens to track difficulty. The framing is a source of exactness on one side and a
vocabulary on the other, and we keep those roles apart throughout.

\paragraph{The boundary is a genuine class boundary (invoked for the general results, verified only for our
own family).} $\NESTED$ is context-free
\citep{rivaseddy2000,durbin1998}. We use $\PK=\mathcal T\setminus\NESTED$ for \emph{all} crossing structures
(a structural set, used as such on the negative side). The formal-language claim is scoped to the specific
unbounded family: the H-type family $\mathcal L_{\mathrm H}\subseteq\PK$ is non-context-free
\citep{nebelweinberg} and MCFG of dimension $2$ \citep{rivaseddy2000,katosekikasami2006}; we do \emph{not}
claim the whole crossing set $\PK$ is uniformly MCFG-$2$ (richer topologies may need higher dimension,
\S\ref{sec:bg}). We take these as established and tie the agent to them via Prop.~\ref{prop:confine}:
``could not before'' is $\Rze{\Ocf}{\pi_{\mathrm{before}}}\subseteq\NESTED\not\ni T$; ``can now'' is
$T\in\Rze{\Opk}{\pi_{\mathrm{after}}}$ for a crossing $T\in\PK$ (the class-level claim, \S\ref{sec:formal},
concerns the non-CF family $\mathcal L_{\mathrm H}$).

\begin{definition}[Empirical crossing-family capability gain (fixed oracle)]\label{def:ascent}
Let $\mathcal D_{\mathrm{ho}}\subseteq\PK$ be a held-out pool of crossing targets (their designability under
$\Opk$ is generally unknown; the exact screen of \S\ref{sec:exp} removes only \emph{provably}-undesignable
targets, so an unsolved target may be a search failure \emph{or} oracle-relative undesignability). A policy
pair $(\pi_{\mathrm{before}},\pi_{\mathrm{after}})$ exhibits an empirical crossing-family capability gain,
adjudicated by the \emph{single, fixed} crossing-aware oracle $\Opk$, if
\begin{equation}\label{eq:ascent}
  \Rze{\Opk}{\pi_{\mathrm{before}}}\cap\mathcal D_{\mathrm{ho}}=\varnothing
  \qquad\text{and}\qquad
  \Rze{\Opk}{\pi_{\mathrm{after}}}\cap\mathcal D_{\mathrm{ho}}\neq\varnothing .
\end{equation}
Both sides are scored by the \emph{same} $\Opk$, so the ascent is not an artifact of exchanging oracles; the
decidable Prop.~\ref{prop:confine} gives the stronger, oracle-independent
$\Rze{\Ocf}{\pi_{\mathrm{before}}}\cap\PK=\varnothing$ as a pre-run bound. The empirical instance measured in
\S\ref{sec:exp} is the left condition at $0/60$ (floor) against the right at $3$--$8/60$ (panel operator).
\end{definition}

\begin{definition}[Transferable crossing-family capability]\label{def:cap}
The agent's \emph{crossing capability} is a set $\Delta\mathcal{O}=\mathcal{O}_{\mathrm{after}}\setminus
\mathcal{O}_{\mathrm{before}}$ of reusable components such that, with $\Delta\mathcal{O}$ enabled and the
external pseudoknot solver disabled, $\Rze{\Opk}{\pi_{\mathrm{after}}}$ contains \emph{held-out} targets of
$\PK$. It is \emph{transferable} if this holds on targets never used to form $\Delta\mathcal{O}$, and
\emph{autonomous relative to a tool class $\mathcal T$} if it holds with every member of $\mathcal T$ disabled
(here $\mathcal T=$ packaged inverse-design solvers; the crossing-aware folding oracles are \emph{not} in
$\mathcal T$ and remain available; see Rem.~\ref{rem:autonomy}).
\end{definition}

\begin{remark}[``Solver-free'' would have been the better word]\label{rem:autonomy}
In this paper ``autonomous'' means exactly one thing, that no packaged inverse-design solver is on the method
path, and that is far weaker than the word carries in the agent literature. Our system queries crossing-aware
folding oracles continuously and could not function without them. Where the distinction matters we write
\emph{solver-free}, and a reader who substitutes it everywhere loses nothing we intend. Two consequences.
Before/after differ in feedback instrument as well as policy: the prior repertoire is scored through a
pseudoknot-free channel and the after-policy queries crossing-aware predictors, so acquiring a new feedback
instrument is part of what happened. And the entity credited with a gain is the oracle-augmented system under a
declared protocol, not a policy in isolation, because the before/after step changes the policy, the in-loop
objective, oracle access and the candidate-selection rule at once. Our design does not decompose that bundle,
and \S\ref{sec:exp} shows a matched ablation failing to isolate any single component under the primary metric.
Every ``capability'', ``gain'' and ``new relative to the prior self'' below is therefore a system-level
empirical gain under augmented oracle access, and specifically not an identification of autonomous learning
inside a policy. The gamer-versus-operator comparison runs at identical oracle access precisely so the
crossing-aware channel is held fixed.
\end{remark}

\begin{remark}[Component attribution strengthens the audit; it is not a necessary condition]\label{rem:attrib}
We do not require that ablating $\Delta\mathcal O$ destroy the held-out rate. A component can be genuinely
acquired and load-bearing as a bundle while no single element is individually necessary, so an all-or-nothing
attribution test would reject real acquisitions. What we do instead is report which attributions we ran and how
they came out, which for this study is mostly negatively: the agreement objective is decisive under
$r_{\mathrm{in}}$ and not distinguishable from its deletion under $r_{\mathrm{mix}}$; the decompose/compose
scaffolding is inert under both; and the undirected control runs at a matched nominal cap rather than matched
realized compute (\S\ref{sec:exp}, Table~\ref{tab:det}). Component attribution is therefore an optional
strengthening, graded like the rest of the verdict. The audit's necessary conditions remain Eq.~\eqref{eq:cert},
in which no per-component clause appears.
\end{remark}

\begin{remark}[What we do and do not specify]\label{rem:scope}
We define the \emph{audit} that recognizes a transferable, autonomous crossing capability; we do \emph{not} here
specify the algorithm that \emph{produces} $\Delta\mathcal{O}$ (how the agent invents reusable components).
That algorithm is the \emph{object under test}: Def.~\ref{def:cap} is exactly what distinguishes a genuine
invented capability from a memorized instance (fails transfer) or a borrowed one (fails autonomy). This is
the honest boundary between an instrument (ours) and a system (to be built and then measured). We cross that
boundary in \S\ref{sec:exp}, where an LLM agent, given only primitives, \emph{writes} a $\Delta\mathcal O$ that
the audit then evaluates.
\end{remark}

\subsection{Two Impossibilities, Kept Separate}\label{sec:impossible}
The audit's negative side and undesignability are different statements and must not be conflated.
Confinement (Prop.~\ref{prop:confine}) is about the repertoire: $\Rze{\Ocf}{\pi_{\mathrm{before}}}$ excludes
$T$ because the \CF{} oracle cannot represent it. Undesignability \citep{zhou2024rigende} is about the target
and energy model: no sequence realizes $T$ under a criterion. It is provable by rival-structure
construction: from the target's motifs one exhibits a rival at least as favorable as $T$ for every
compatible sequence, so $T$ is never the unique optimum \citep{zhou2024rigende,zhou2025motifs}. Since
deciding ``is $T$ designable?'' is NP-hard in general \citep{bonnet2020hard}, a positive undesignability
certificate, when the construction succeeds, is a valuable decidable witness on an otherwise hard question.
Undesignability binds the ascended agent too: no repertoire realizes a target that no sequence realizes. The
ascent's negative side uses confinement, not undesignability.

\paragraph{An optional, distinct capability.} Undesignability enables a separate two-sided predictive
capability: per held-out target the agent commits either a design or an undesignable claim, adjudicated
respectively by the oracle panel or by an undesignability certificate. We delineate it from the crossing axis
and do not pursue it here, noting only that a committed undesignability claim is a computation under sealed commitment, not a hypothesis in the usual empirical sense.

\section{Experiment: Measuring Oracle-Gaming, and Auditing the Crossing Claim It Inflates}\label{sec:exp}
This section does three things. It \emph{measures} how far a single fallible oracle can inflate a capability
claim: an invented, solver-free operator reaches $43/60$ under the predictor it optimizes and $1/60$ under
three, with the discriminating comparison paired against an external solver that optimizes something else. It
\emph{audits} the system we then built to survive that measurement, reporting each component of the evidence
vector separately, including the two it does not establish. And it establishes that the instrument
\emph{discriminates}, by measuring an arm that clears the bar our own does not. That is where this section's
one clearly positive result lives, and why it is not filed as a supplement. The three arguments are kept in
separate subsections because they are answerable separately and a reader may want only one:
\S\ref{sec:exp:gaming} is the measurement, \S\ref{sec:exp:ours} the audit of our own operator,
\S\ref{sec:exp:pc} the positive control and the seven interventional attribution tests,
\S\ref{sec:exp:pools} what the target sets can support, and \S\ref{sec:exp:stats} the statistical
conventions and remaining robustness checks. Everything is on the $n{=}60$ Pseudobase++ crossing pool unless a larger or structurally disjoint
pool is named; every rate is one of the three metrics of Eq.~\eqref{eq:metricnames}, never a blended
``panel'' number.

\paragraph{Setup: an exact difficulty axis, a validated panel, and an exact designability screen.}
\emph{(i) The grade is intrinsic and exact.} Each target's family type and crossing-stem count are properties
of the target, decidable in $O(|P|^2)$ and never inferred from whether the agent solves it, so the graded
frontier carries no survivorship bias and the boundary is the established $\CF\!\to\!\MCF$ line
(\S\ref{sec:cf}), not a learned proxy. \emph{(ii) Verifier validity is checked before verdicts are trusted.}
The \CF{} ladder and the pseudoknot panel are reused instruments \citep{verifiereconomics}; we run the
specificity half here (skeleton designs, expected zero crossing recall) and inherit native recall from that
companion. Without it a lenient panel could credit skeletons as pseudoknots and manufacture an ascent.
\emph{(iii) Undesignable targets are excluded only on an exact certificate}, by rival-structure construction
\citep{zhou2024rigende,zhou2025motifs} and never on solver failure, so hard-but-designable targets stay in; on
the primary pool it excludes none. We also do not infer designability from native provenance, since our
calibration shows $\Omega(x_{\mathrm{nat}})=T$ usually fails ($\sim\!1\%$ native recall). Oracle-relative
designability of the pool is therefore \emph{unknown}, and an unsolved target may be a search failure or an
oracle-relative impossibility.

\paragraph{Positive-side adjudication under oracle-dependence.} Because the pseudoknot verdict is
oracle-dependent (\S\ref{sec:bg}), a target counts as realized only under an explicit rule: (i)
mixed-panel-consistent, the primary rate; (ii) reference (pKiss); and (iii) envelope $[\min,\max]$ across
predictors, reported alongside. We fix two acceptance predicates explicitly, because conflating them is the
classic pseudoknot-evaluation error. Writing $P(\cdot)$ for a pair set and $P_{\mathrm{nest}}(\cdot)$ for its
\emph{nested projection} (crossing pairs deleted, which is what a pseudoknot-free parser sees),
\begin{equation}\label{eq:predicates}
  A_{\mathrm{full}}(x,T)=\mathbf 1\!\left[P(\Omega(x))=P(T)\right],
  \qquad
  A_{\mathrm{nest}}(x,T)=\mathbf 1\!\left[P_{\mathrm{nest}}(\Omega(x))=P_{\mathrm{nest}}(T)\right].
\end{equation}
\textbf{Every rate in this paper uses $A_{\mathrm{full}}$}: a predictor confirms iff its predicted pair set
equals the target's \emph{entire} pair set: nested \emph{and} crossing pairs, after bracket-type
normalization, with no extra or missing pairs (\texttt{forms\_target}); the mixed panel requires this of all
three. $A_{\mathrm{nest}}$ is reported only as the \emph{blind rule whose false accepts we count}
(Table~\ref{tab:panel}): for a crossing target $T\in\PK$, Prop.~\ref{prop:confine} makes
$A_{\mathrm{full}}(x,T)=0$ under \Ocf{} for every $x$, whereas $A_{\mathrm{nest}}$ can and does return $1$. So
$A_{\mathrm{nest}}$ acceptances are evaluator errors, never evidence that a target needs no crossing.
(The $\sim\!1\%$ native recall of \S\ref{sec:oracle} is measured under a \emph{laxer}
crossing-pair-\emph{recovery} predicate, so the panel's recall under $A_{\mathrm{full}}$ is no
higher.) Disagreements trigger abstention; we never treat a single predictor as ground truth.

%


\subsection{The Measurement: One Fallible Oracle Credits $43$ Designs, Three Credit $1$}\label{sec:exp:gaming}

\paragraph{The naive operator crosses only by gaming the oracle.} We instantiate the full audit on
the $n=60$ development/evaluation pool with an invented operator (\textsc{pkdecomp}): decompose the target into its
context-free skeleton, sub-solve it with a pseudoknot-free designer, compose the crossing stem by
complementarity, and repair against the pseudoknot oracle, the one move the \{REMC, DesiRNA\} toolbox never
makes. It never calls an external pseudoknot solver, so its autonomy is structural. Under the single
reference oracle (pKiss) the outcome looks like a decisive ascent: $43/60$ solved, matching the external SOTA solver while calling nothing, far above the $0/60$
floor. The negative side is firm on both its components: the $\Ocf$-certification bound is exact (Prop.~\ref{prop:confine})
and the empirical \CF-confined floor is $0/60$, with every skeleton design a naive-MFE-style false positive the
panel refutes (a $0/60$ panel credit on the E0 skeleton control below).

The positive side does not survive the panel. The operator optimized pKiss in the loop; re-adjudicated by
a held-out predictor, only $2/43$ of its solves are confirmed and one is panel-unanimous
(Table~\ref{tab:panel}). The apparent $43/60$ ascent collapses to
$1/60$ (Fig.~\ref{fig:collapse}): an oracle artifact, the operator having learned the blind spots of the predictor it was scored on
\citep{skalse2022,gao2023overopt}. A single-oracle evaluation would have reported a new autonomous
capability; under the abstaining panel (\S\ref{sec:oracle}) almost nothing survives. The gap between $43$
and $1$ is what a one-sided criterion cannot see.

\emph{What the ratio is, and what carries the claim.} We name the collapse by its ratio because that is the
quantity a practitioner cares about, but the ratio's denominator is a single target and we will not pretend
otherwise: $1/60$ has an exact binomial interval of $[0.0004,0.089]$ against $43/60$'s $[0.586,0.825]$, so the
data are compatible with an inflation anywhere from roughly $7\times$ upward. Had two designs survived instead
of one it would read $21.5$-fold. \textbf{Nothing in this paper rests on the ratio.} What carries the finding is the paired contrast below, on the same $43$ targets: $2$ confirmed against $26$,
exact McNemar $p{=}8.1\!\times\!10^{-7}$. Its uncertainty is small because it is paired and its denominator is
$43$, not $1$.

\begin{table}[t]\centering\small
\setlength{\tabcolsep}{6pt}\renewcommand{\arraystretch}{1.15}
\caption{\textbf{Cross-oracle panel adjudication of the invented operator's $43$ pKiss-solves.} Read the top block down
the count column: $43\to2\to1$ as adjudication moves from the oracle the operator optimized to predictors it
never saw. The lower block is the exclusion side, and its two numbers are \emph{not} confirmations of the
design: all $43$ targets are crossing ($3$--$9$ crossing pairs each), so a pseudoknot-free fold
realizes none, and the $4$ counts how often the pseudoknot-blind predicate would have been fooled.}\label{tab:panel}
\begin{tabularx}{\linewidth}{@{}>{\raggedright\arraybackslash}p{5.0cm}cL@{}}
\toprule
\textbf{adjudicator / predicate} & \textbf{count\,/\,43} & \textbf{what the number counts}\\
\midrule
\multicolumn{3}{@{}l}{\emph{crossing-aware adjudicators: the collapse}}\\
pKiss, $A_{\mathrm{full}}$ & $43/43$ & the in-loop oracle, the rate the operator was optimized to produce\\
ProbKnot, $A_{\mathrm{full}}$ & $2/43$ & first crossing-aware predictor the gamer never saw (in-loop for the
panel operator)\\
all three predictors, $A_{\mathrm{full}}$ & $\mathbf{1/43}$ & mixed panel, $2$ in-loop $+\,1$ held out;
$1/60$ of the pool, the audited outcome\\
\addlinespace[2pt]
\multicolumn{3}{@{}l}{\emph{pseudoknot-free adjudicators: the exclusion side, not confirmations of the design}}\\
pseudoknot-free MFE, $A_{\mathrm{full}}$ & $\mathbf{0/43}$ & a theorem, not a measurement
(Prop.~\ref{prop:confine}: $\operatorname{im}\Ocf\subseteq\NESTED$)\\
pseudoknot-free MFE, \emph{blind} $A_{\mathrm{nest}}$ & $4/43$ & designs the nested-projection rule would
\emph{falsely accept}: that rule's own error rate, not a property of the targets\\
\bottomrule
\end{tabularx}
\end{table}

\begin{figure}[t]\centering
\includegraphics[width=0.78\linewidth]{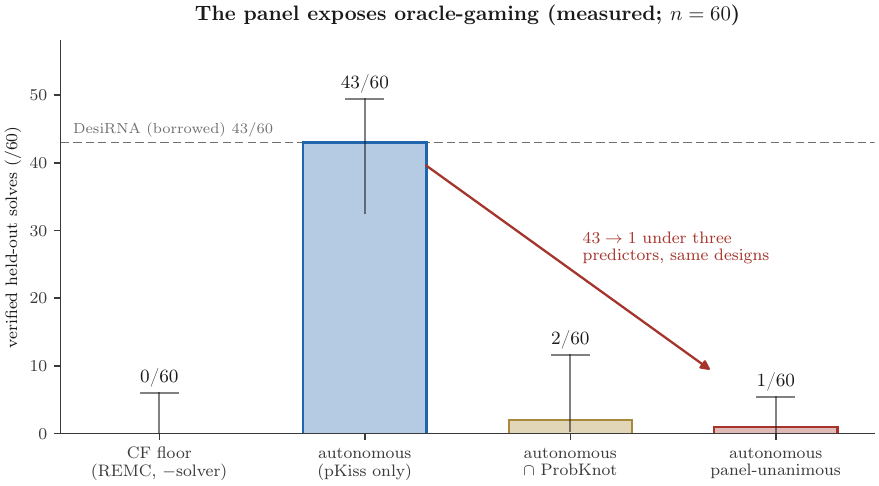}
\caption{\textbf{The panel exposes oracle-gaming (measured, $n=60$).} The invented operator's single-oracle
solve rate ($43/60$, matching the borrowed DesiRNA ceiling) collapses to $2$ under a held-out predictor
and to $1$ panel-unanimous, above the decidable context-free floor of $0/60$. Bars carry exact binomial
intervals, and the collapse is annotated as $43\!\to\!1$ on the \emph{same} designs, not as a ``$43\times$'' ratio: the last bar is a single design, so that quotient is unstable and nothing in this paper
rests on it (\S\ref{sec:exp}). What a one-sided criterion cannot see is the gap, not its
quotient.}\label{fig:collapse}
\end{figure}

\paragraph{Oracle-gaming vs.\ oracle-robustness: a contrast with the borrowed solver.} Is the collapse
specific to our operator? Re-adjudicating the external solver's $43$ pKiss-solves under the same panel
(Table~\ref{tab:contrast}): DesiRNA optimizes minimum free energy, not the pseudoknot oracle, and $26/43$ of
its solves survive ProbKnot, whereas our operator, which optimized pKiss, retains only $2/43$.

\emph{Is that a paired comparison?} Two ``$43/60$'' figures need not name the same $43$ targets, in which case
the ratio would mix method with target composition and difficulty. We checked (\texttt{contrast\_paired.py}), and the check comes out cleanly: the two pKiss-solved sets are
\emph{identical} (intersection $43$, Jaccard $1.000$, neither arm solving a target the other misses), so the
comparison is exactly paired by construction, on targets with an identical difficulty profile (mean crossing-pair
count $5.26$ for both, against a pool mean of $5.48$). On that common subset the paired McNemar is
$b{=}25$ (DesiRNA only), $c{=}1$ (ours only), exact $p{=}8.1\!\times\!10^{-7}$, paired risk difference
${+}0.558$ with bootstrap $95\%$ CI $[{+}0.395,{+}0.721]$; at the composition-free full $n{=}60$ denominator it
is $26/60$ vs.\ $2/60$, ${+}0.400$ $[{+}0.267,{+}0.533]$. So target composition is excluded as an explanation.

\emph{What that still does not license.} The two arms are two different methods, not a randomized
manipulation of one method's training oracle, so the design is observational in the arm: pKiss-training is
confounded with every other difference between DesiRNA and our operator (search algorithm, energy model,
candidate quality, budget). The defensible reading is therefore \emph{association}, not causal localization:
the result is \emph{consistent with substantially stronger oracle-specific over-optimization by the
pKiss-trained method}, and we do not claim it localizes gaming \emph{to} that training. Two further caveats.
Under the full three-predictor unanimous panel neither arm is robust ($0/43$
for DesiRNA, $1/43$ for ours): ShapeKnots, run de novo without probing data, rejects almost everything, so
the cleaner comparison is the two de-novo predictors ($\text{pKiss}\wedge\text{ProbKnot}$), where the
$2$-vs-$26$ gap stands. And that even a SOTA solver is only $\sim\!60\%$ cross-predictor-robust confirms the
thesis: ``forms the pseudoknot'' is deeply model-relative.

\begin{table}[t]\centering\small
\setlength{\tabcolsep}{6pt}\renewcommand{\arraystretch}{1.15}
\caption{\textbf{Oracle-gaming vs.\ cross-predictor carry-over} (each arm's $43$ pKiss-solves, re-adjudicated). The
pKiss-optimizing operator is $13\times$ less ProbKnot-robust than the MFE-optimizing external solver. The comparison is paired: the two pKiss-solved sets are the same $43$ targets (Jaccard $1.000$, matched
mean crossing-pair count $5.26$), so target composition cannot explain the gap. Paired McNemar $b{=}25$,
$c{=}1$, exact $p{=}8.1\!\times\!10^{-7}$, risk difference ${+}0.558$ $[{+}0.395,{+}0.721]$
(\texttt{contrast\_paired.py}). Because the arms are two different methods rather than a randomized
manipulation of one method's training oracle, we read this as an \emph{association}. Under the full unanimous
panel neither arm is robust.}\label{tab:contrast}
\begin{tabular}{@{}lccccc@{}}
\toprule
\textbf{method (objective)} & \textbf{pKiss} & \textbf{ProbKnot} & \textbf{ShapeKnots} & \textbf{unanimous}
& \textbf{ProbKnot\,/\,60}\\
\midrule
DesiRNA (minimum free energy) & $43/43$ & $\mathbf{26/43}$ & $1/43$ & $0/43$ & $26/60$\\
invented operator (pKiss in-loop) & $43/43$ & $\mathbf{2/43}$ & $3/43$ & $1/43$ & $2/60$\\
\midrule
\multicolumn{6}{@{}l@{}}{\footnotesize\emph{paired, same $43$ targets}: $b{=}25$, $c{=}1$, exact McNemar
$p{=}8.1\!\times\!10^{-7}$, $\widehat\Delta{=}{+}0.558$ $[{+}0.395,{+}0.721]$}\\
\bottomrule
\end{tabular}
\end{table}

\subsection{The Audit of Our Own System, Which It Largely Fails}\label{sec:exp:ours}

\paragraph{It is a search, and the ablation says which part of it works.} We do not pose this as capability
\emph{versus} search: the operator is a search, and the useful question is which ingredient carries the gain.
We answer it in the bit-reproducible regime only (Table~\ref{tab:det}; App.~\ref{app:repro} gives the fix, and
the earlier wall-clock ablations it supersedes are not reported, since every conclusion below is stronger
without them). Five arms, the same operator at an identical $150$-call pKiss allowance, differing in exactly
one component, five noise-free seeds, paired on target and seed.

Three findings, all sharp, and only the first flattering. \emph{(i) The repair loop is the operator.} Delete
it and $r_{\mathrm{in}}$ falls $37.0\to1.0$ ($\Delta{=}{+}0.600$ $[{+}0.487,{+}0.707]$) and $r_{\mathrm{mix}}$
to $0$. The crossing capability is stochastic search against the pseudoknot oracle, which is precisely why it
games: search against a scorer finds its blind spots. \emph{(ii) The objective is load-bearing only on the
metric it optimizes.} Deleting it costs ${+}0.140$ $[{+}0.063,{+}0.220]$ on $r_{\mathrm{in}}$ and
\emph{exactly zero} on $r_{\mathrm{mix}}$: ${+}0.0000$, $[{-}0.060,{+}0.050]$, on $300$ paired cells. With
execution noise removed this is a measured absence, not an underpowered null. So ``optimize agreement'' is
supported as a \emph{choice of success predicate} and not as a search objective that earns its keep.
\emph{(iii) The \CF{} decomposition is load-bearing after all}, ${+}0.093$ $[{+}0.023,{+}0.173]$ on
$r_{\mathrm{in}}$, where the noisier estimate had it inert: the one place removing execution noise reversed a
conclusion, and we flag it instead of quietly adopting the new sign.

One matching caveat survives determinism and cannot be removed. Equal \emph{allowance} is not equal
\emph{use}: the full operator often solves and stops, so realized pKiss calls run $52.6$--$105.5$ across arms
(Table~\ref{tab:det}, last column). Equalizing realized use would mean penalizing success. The ablated arms
consume \emph{more} compute, so their lower $r_{\mathrm{in}}$ is not a budget artifact; the comparison is
nonetheless matched on access, not on use.

\emph{How much does (ii) actually say?} Less than the interval suggests, and we would rather state that than
let a reader find it. $r_{\mathrm{in}}$ is pKiss\,$\wedge$\,ProbKnot and the undirected walk optimizes
neither, so ``deleting the objective costs ${+}0.140$ on $r_{\mathrm{in}}$'' is close to tautological: the
arm fails the conjunct it stopped pursuing. The informative half of (ii) is the other number, the
\emph{exact} zero on $r_{\mathrm{mix}}$, which is not tautological, because $r_{\mathrm{mix}}$ adds a
predictor neither arm optimizes, and because a two-sided interval of $[{-}0.060,{+}0.050]$ on $300$ paired
noise-free cells is a measurement rather than a failure to find one.

\emph{What this licenses about the success predicate, stated precisely.} The natural conclusion, that the success predicate and not the search is the causal lever, is not supported in that form. Changing the
predicate from one predictor to a two-predictor conjunction is \emph{associated} with a higher
$r_{\mathrm{mix}}$ in the gamer-versus-operator comparison, while deleting the objective from the operator
does not lower $r_{\mathrm{mix}}$ at all. Those reconcile only if something other than the objective
distinguishes the two systems, and this design cannot say what. We report the association and stop.

\emph{The uncomfortable reading, stated in our own words.} Put (i) and (ii) together and the description of
our operator that the data actually supports is: \textbf{a stochastic search against a fallible oracle, with a
better success predicate and no better guidance}. Not ``a design principle''. We prefer to write that
sentence ourselves than to shelter behind ``a search procedure can itself be a capability'', which is true
but is not what these numbers show. It also sharpens what the positive control is for. The agent-written
operators clear the \emph{same} predicate on the \emph{same} targets, so the predicate is held fixed between
them and us; whatever separates them is in the search, which is exactly the dimension our own ablation says
we did not move.

\begin{table}[t]\centering\small
\setlength{\tabcolsep}{3.5pt}\renewcommand{\arraystretch}{1.15}
\caption{\textbf{Component ablation and the headline comparison, bit-reproducibly at equal allowance}: $5$ nominal
seeds, one noise-free run each, gamer included. Counted budgets replace wall-clock caps and \textsc{ViennaRNA}'s
generator is seeded per attempt, so a seed fixes the run bit-for-bit; every arm gets the same $150$-call pKiss
allowance. $\Delta$ is the paired difference against the reference with a two-level bootstrap over targets and
seeds ($B{=}20{,}000$); \textbf{bold} excludes zero. $r_{\mathrm{in}}$ is each arm's \emph{own} predicate,
pKiss alone for the gamer, so it is not comparable across those groups; $r_{\mathrm{mix}}$ is the cross-arm
metric. ``real pKiss'' is realized calls per target, which differ because the full operator often solves and
stops. (\texttt{panel\_deterministic.sh}, \texttt{det\_stats.py}.)}\label{tab:det}
\begin{tabularx}{\linewidth}{@{}Lccccc@{}}
\toprule
& \multicolumn{2}{c}{$r_{\mathrm{in}}$ (own predicate)} & \multicolumn{2}{c}{$r_{\mathrm{mix}}$ (all three)} & \\
\cmidrule(lr){2-3}\cmidrule(lr){4-5}
\textbf{arm} & \textbf{mean/60} & $\Delta$ [CI] & \textbf{mean/60} & $\Delta$ [CI] & \textbf{real pKiss}\\
\midrule
gamer (single-oracle) & $42.6$ & \emph{n/c} & $0.00$ & \textbf{ref} & $52.6$\\
full operator (reference) & $37.0$ & --- & $4.00$ & --- & $82.2$\\
\midrule
\multicolumn{6}{@{}l}{\emph{headline:} full operator $-$ gamer on $r_{\mathrm{mix}}$ $=\mathbf{+.067}$ $[{+}.023,{+}.120]$, gamer $0/60$ at all five seeds}\\
\midrule
$-$ objective (undirected) & $28.6$ & $\mathbf{+.140}$ $[{+}.063,{+}.220]$ & $4.00$ & $+.000$ $[{-}.060,{+}.050]$ & $99.5$\\
$-$ objective, $-$ bias (blind) & $25.4$ & $\mathbf{+.193}$ $[{+}.107,{+}.287]$ & $2.40$ & $+.027$ $[{-}.017,{+}.077]$ & $105.5$\\
$-$ repair loop entirely & $1.0$ & $\mathbf{+.600}$ $[{+}.487,{+}.707]$ & $0.00$ & $\mathbf{+.067}$ $[{+}.023,{+}.123]$ & $1.0$\\
$-$ \CF{} decomposition & $31.4$ & $\mathbf{+.093}$ $[{+}.023,{+}.173]$ & $3.40$ & $+.010$ $[{-}.033,{+}.050]$ & $100.1$\\
$-$ complementarity composition & $34.4$ & $+.043$ $[{-}.030,{+}.137]$ & $5.40$ & $-.023$ $[{-}.080,{+}.030]$ & $90.3$\\
\bottomrule
\end{tabularx}
\end{table}

\paragraph{The panel-optimizing operator, and its held-out rate.} The operator whose repair loop optimizes
agreement between pKiss and ProbKnot solves $34$--$37/60$ on its own two oracles across five seeds against the
gamer's ProbKnot-robust $2/60$; $3$--$8$ of those (mean $5.0$) are also panel-unanimous, against $\sim\!1/60$
for the gamer and $0/60$ for the SOTA solver. Read every band of this form as the range of the runs named
beside it: pooling all $24$ executions gives $30$--$39/60$ and $1$--$9/60$ (mean $3.29$). Panel-unanimity puts
two of its three conjuncts inside the operator's objective, so it cannot by itself show generalization. On the
one genuinely held-out predictor the operator is higher at every seed but significant in one of three ($n{=}125$
ID-disjoint targets; $\widehat\Delta = {+}0.008, {+}0.056, {+}0.064$). We record that as \emph{weak} and do not
build on it. In \emph{absolute} terms it is not small once given its reference class: the same adjudicator
recovers only $11.3\%$ of native pseudoknots from their own sequences and confirms our designs at $8.0\%$,
about $70\%$ of its own ceiling, against the gamer's $3.7\%$ (App.~\ref{app:adjud}).

\begin{figure}[t]\centering
\includegraphics[width=\linewidth]{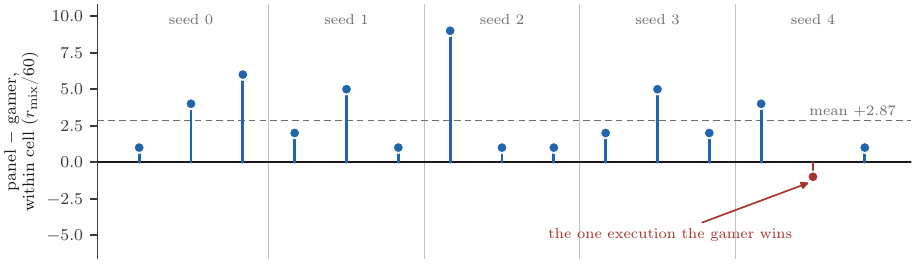}
\caption{\textbf{The statement about the headline effect that uses no interval at all.} Within-cell paired
difference in $r_{\mathrm{mix}}$, panel operator minus single-oracle gamer, for each of the $15$ wall-clock
executions ($5$ nominal seeds $\times$ $3$ executions) on the $60$ common targets. Fourteen point one way and
one points the other. This is the form the headline result is most defensible in, because it assumes nothing
about the resampling unit and nothing about which variance components a bootstrap should nest: it is a sign test on independent executions, and that is why \S\ref{sec:exp:ours} settles the sign while declining to settle
the magnitude. Bar heights are the per-cell McNemar discordance $b-c$, parsed from
\texttt{gvp\_stats.txt}.}\label{fig:paired}
\end{figure}

\paragraph{$\Delta_{\mathrm{gamer}}$: separated, but by how much depends on the estimator, and we say so.}
Run on the \emph{same} seeds and targets as the operator, the gamer is panel-unanimous on
$\{1,0,0,0,0\}/60$ against $\{8,3,5,5,4\}/60$, and per seed the paired McNemar has $c{=}0$: the gamer never
solves a target the operator misses. \textbf{The quantity is estimated four times in this paper and the point
estimate falls every time the design is tightened}: ${+}0.080$, ${+}0.056$, ${+}0.048$ as execution
replication is added, then ${+}0.067$ $[{+}0.023,{+}0.120]$ in the bit-reproducible regime with the gamer at
$0/60$ at every seed (Table~\ref{tab:det}). App.~\ref{app:repro} gives each design, its intervals, and the
conservative one-execution-per-seed variant that reaches $0.000$. Three readings, and only the first is
comfortable. Determinism removes the \emph{execution} level of variation but not the \emph{target-set} level
the decline pointed at, and its narrower interval reflects one fewer variance component, not more evidence. The most robust statement uses no interval at all (Fig.~\ref{fig:paired}): the operator is strictly higher in $14$ of $15$
independent executions and $15$ of $15$ on the structurally disjoint pool. And ${+}0.067$ sits at $77\%$
power for this design (\S\ref{sec:exp}), i.e.\ at its resolution limit, which is what ``established,
narrowly'' means in units a reader can check. What is settled is the sign; the magnitude is not, and the
search-free probe below removes most of it on a different ground.

\emph{Two confounds this comparison exposes rather than removes.} The arms are matched on oracle
\emph{access} and not on oracle \emph{use}: per target the operator makes $128.0$ pKiss and $63.5$ ProbKnot
calls against the gamer's $55.2$ and $0$, because the gamer's single-predictor objective is satisfied at
$43/60$ and stops. Realized compute is a \emph{mediator} of the objective change, so
$\Delta_{\mathrm{gamer}}$ measures ``better designs'' and ``searches $2.3\times$ longer'' summed, and this
design cannot split them. And it is a between-\emph{system} contrast: two different methods, not a randomized
manipulation of one method's oracle.

\paragraph{$r_{\mathrm{mix}}$ is not metric-neutral, so we redo the comparison on a statistic that is.}
$r_{\mathrm{mix}}=r_{\mathrm{in}}\wedge\text{ShapeKnots}$, and $r_{\mathrm{in}}$ is \emph{the panel operator's
own success predicate}. Scoring both arms on $r_{\mathrm{mix}}$ therefore scores one of them on its objective
plus one term, which by our own argument in \S\ref{sec:oracle} cannot by itself evidence generalization. The
information $r_{\mathrm{mix}}$ adds over $r_{\mathrm{in}}$ is the conditional term, so we make that term the
statistic:
\begin{equation}\label{eq:carry}
\begin{aligned}
  C(\text{arm}) \;=\; \Pr\bigl[\,&\text{a predictor never in \emph{that arm's} loop confirms the design}\\[-1pt]
  &\;\big|\; \text{the arm's own success predicate is satisfied}\,\bigr].
\end{aligned}
\end{equation}
$C$ is neutral because every arm is scored on how far \emph{its own} successes carry to a predictor \emph{it}
never optimized, with its own predicate as the denominator. ShapeKnots is in no arm's loop, so carry-over to
ShapeKnots is defined identically for all of them (Fig.~\ref{fig:carry}). Three readings follow, and they do
not all point the same way.

\emph{The gamer comparison survives the change of metric.} Difficulty-matched to the $528$ (target, seed,
execution) cells where both arms satisfy their own predicate, ShapeKnots confirms $48$ of the panel operator's
designs against $22$ of the gamer's ($b{=}41$, $c{=}15$, exact $p{=}6.9\!\times\!10^{-4}$, difference
${+}0.049$). That the neutral statistic reproduces the $r_{\mathrm{mix}}$ value of ${+}0.048$ almost exactly is
reassuring rather than circular: the two are computed from different denominators. One conditioning does
\emph{not} agree and we report it: if both arms are instead held to the panel operator's bar
($r_{\mathrm{in}}$ for both), the gamer's rate is nominally higher ($5/32$ against $48/530$). That comparison
conditions the gamer on having passed ProbKnot, the very carry-over at issue, and retains $32$ of its $643$
designs, so we do not treat it as primary; but with an interval of $[0.053,0.328]$ it does not refute the
primary reading either, and a reader who prefers it should read $\Delta_{\mathrm{gamer}}$ as unestablished.

\emph{The objective ablations do not survive it.} The undirected walk and the oracle-blind walk share the panel
operator's success predicate and differ only in guidance, and on $C$ they are indistinguishable from it
($0.095$ and $0.110$ against $0.094$; paired $p{=}0.57$ and $p{=}0.29$, both point estimates \emph{against}
the full operator). So what the two-oracle conjunction gains comes from the \emph{predicate}, not from the search that pursues it. That is the same conclusion $\Delta_{\mathrm{undirected}}$ reached on
$r_{\mathrm{mix}}$, now on a metric that cannot be accused of favouring either arm.

\emph{Cross-predictor robustness is pairwise, not global.} DesiRNA carries over to ProbKnot at $26/43$ against
our gamer's $2/43$ (\S\ref{sec:exp}, Table~\ref{tab:contrast}), and to ShapeKnots at $1/43$, the
\emph{lowest} rate of any arm in Fig.~\ref{fig:carry}. An arm can be $13\times$ more robust under one
un-optimized predictor and the least robust under another. ProbKnot and DesiRNA's minimum-free-energy
objective are both built on the same nearest-neighbour thermodynamic parameters, and ShapeKnots run de novo is
not; the natural reading is that ``robustness'' here tracks shared model family, not correctness. That is a limitation of the whole panel, not of one arm, and it is the strongest argument in this
paper for adding a predictor from a different model class (\S\ref{sec:disc}).

\begin{figure}[t]\centering
\includegraphics[width=\linewidth]{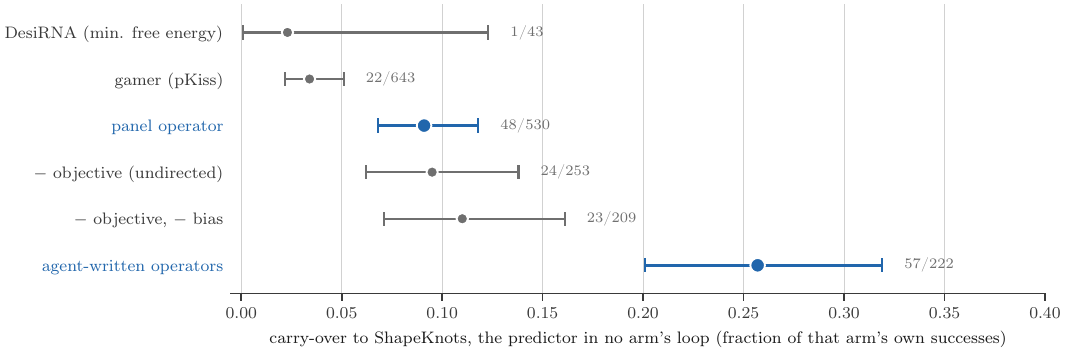}
\caption{\textbf{Conditional carry-over $C$ (Eq.~\eqref{eq:carry}): the metric-neutral comparison.} Each arm is scored
on the fraction of its own successes that a predictor never in its loop confirms, so the denominators differ by
design and that is what makes the statistic neutral. Bars are exact binomial intervals; counts are successes /
own successes. The two coloured rows are the comparison the figure exists to make: the hand-built panel
operator against the frozen agent-written operators of App.~\ref{app:invent}, which optimize the same
predicate. The ProbKnot direction is meaningful only for the two arms that do not optimize it
(Table~\ref{tab:contrast}). The agent row pools six single runs; its two zero-failure models replicated at
$5\times3$ give $0.302$ $[0.275,0.330]$.}\label{fig:carry}
\end{figure}

\subsection{Does the Instrument Discriminate? A Positive Control That Becomes a Result}\label{sec:exp:pc}

\paragraph{Does the instrument discriminate, or is it merely insensitive? A positive control.} We have said
more than once that nothing we build saturates the panel and offered that as evidence of stringency. It is
not: an audit nothing passes and an audit that cannot detect anything produce the same table, and at
$\sim\!1\%$ recall on natives insensitivity is the live alternative. Separating them needs an arm that
\emph{does} clear the bar, measured identically. We already had one, filed as a supplement. The frozen
LLM-written operators of App.~\ref{app:invent} optimize pKiss\,$\wedge$\,ProbKnot, the panel operator's own
predicate, so ShapeKnots is held out from them in the same sense.

Replicated at the depth of the headline experiment (the two operators with no runtime failures, frozen code
re-executed at $5$ seeds $\times$ $3$ executions, \texttt{agent\_replicate.py}), their carry-over is
$333/1104=\mathbf{0.302}$ $[0.275,0.330]$ against the hand-built operator's $48/530=0.091$ $[0.068,0.118]$;
the intervals do not overlap. The three executions inside a seed are \emph{bit-identical} in $5/5$ seeds, so
unlike our own operator these are reproducible runs.

\emph{Paired, at the standard we impose on the transplant arms.} Those are two conditional rates on two
different denominators, and Fig.~\ref{fig:attrib}'s caption forbids exactly that reading for every other arm,
so we apply the same standard here, not a weaker one, to the arm that carries the paper's discriminating-power claim. Both arms optimize the same pKiss$\,\wedge\,$ProbKnot predicate and are scored by
ShapeKnots, held out of both, so restricting to the (target, cell) units where \emph{both} arms cleared
that shared bar gives a paired comparison. Same target, same nominal seed, same execution index,
same in-loop bar met; the only open question is whether the held-out predictor confirms. On those
$n{=}951$ units the agent operators carry over at $0.293$ against $0.095$, with $b{=}218$ discordant pairs
favouring the agents against $c{=}29$ favouring us: a paired difference of ${+}0.199$.

\emph{We do not quote the exact McNemar these counts give.} It is $5.1\!\times\!10^{-37}$, and it assumes $951$
independent units when they come from $60$ targets contributing ${\sim}16$ correlated observations each.
Fig.~\ref{fig:attrib}'s caption rules out that anticonservatism for every transplant arm, so applying it to the
arm carrying this paper's positive result would be indefensible. Resampling whole \emph{targets} gives
$[{+}0.108,{+}0.297]$ over $42$ clusters; a permutation test flipping each target's discordant pairs together
gives $p{=}5\!\times\!10^{-5}$. Per model the clustered intervals are $[{+}0.151,{+}0.382]$ and
$[{+}0.059,{+}0.241]$. The correction moves the $p$-value by $32$ orders of magnitude and changes no conclusion.
Matching costs the point estimate almost nothing ($0.302$ to $0.293$), which is informative: the
advantage was not an artifact of the two arms solving different targets (\texttt{arep\_matched.py}). One
quantifier the matching does not repair: these are the $2$ of $6$ invented operators with no runtime failures,
a model-level survivorship selection on top of the target-level one, so the supported claim is that
\emph{some} agent-written operators carry over better.

\emph{How much compute do the agents actually spend? We stopped arguing from their source and counted.} The
allowance control above is only informative if it reaches the regime it is meant to rule out, and the frozen
sources make that doubtful: one asks for $15$ restarts $\times$ $3000$ iterations against our $6\times800$.
But a constant in the source is what the code \emph{requests}, not what it spends under the $20$\,s per-target
cap the replication imposed, so we instrumented the two crossing-aware primitives and re-ran the frozen code
in that regime (\texttt{agent\_budget.py}). The measurement inverts the worry. Realized pKiss calls per target
are $\approx\!28$ for gemini-3.1-pro and $\approx\!13$ for kimi-k3, against the hand-built operator's
$\mathbf{128.0}$ in the same regime: $4.6\times$ and $9.8\times$ \emph{fewer}, with the $3\times$-allowance
control sitting $7\times$ above the hungriest of them. This licenses a stronger statement than ``not bought
with compute'': \textbf{the agent operators reach $0.293$ carry-over while spending a fraction of the oracle
calls of the operator they beat at $0.095$}. (Two independent runs of the instrumented measurement gave
$28.6/12.9$ and $27.8/13.1$; our harness runs it in-process and serially, so it does not reproduce the
fork-per-call isolation of the replication, and the small spread is that. It does not touch the conclusion.)

\emph{The inverse control, and the compute-matched comparison it buys.} The measurement also generated a
hypothesis and let us kill it in the same pass. If the agents are not searching harder, perhaps they benefit
from searching \emph{less}. kimi-k3's source contains an explicit \texttt{budget = [30]}, a self-imposed
oracle-call cap, and a search forced to commit early might accept candidates from a different region of
sequence space. We ran the compute control the other way: the hand-built operator at a $30$-call pKiss
allowance, nothing else changed, deterministic, five seeds (\texttt{tinypk.sh},
\texttt{tinyctl\_stats.py}). Carry-over is $8/75=\mathbf{0.107}$ $[0.030,0.203]$ against the baseline's
$0.108$, while $r_{\mathrm{in}}$ collapses from $37.0$ to $15.0$ per $60$. By the rule fixed before the
numbers, that a mechanism counts only if carry-over rises \emph{and} $r_{\mathrm{in}}$ holds, that is a fifth
null, and a clean one: the rate did not even inflate in the artifactual way a halved denominator invites.

That arm is worth more than the hypothesis it refutes, because it realizes $26.0$ pKiss calls per target
against gemini-3.1-pro's $27.8$. \textbf{At matched realized compute, within $7\%$ on the coordinate that
matters, the hand-built operator carries over at $0.107$ and the agent-written one at $0.367$.} This
discharges the compute-matched (rather than cap-matched) control the previous revision listed as owed, and it
closes the compute question in both directions: the advantage does not appear when we are given three times
the allowance, and it does not appear when we are held to the agents' own.

\emph{A structural property of the designs, gated then transplanted.} Composition and compute are excluded, so
the surviving hypothesis is that the accepted designs differ in \emph{where} their nucleotides sit relative to
the target's pair set. We pre-stated six such features and two gates: (a) the feature differs between the arms,
and (b) it predicts carry-over \emph{within} the hand-built arm, where the program is held fixed. Overall GC
went in as a negative control, and earned its place at once by passing gate (b) outright (AUC $0.749$,
$p{=}1.9\!\times\!10^{-10}$), so an unadjusted gate admits anything GC-correlated. With both gates adjusted for
GC by logistic regression and a permuted feature as the adjusted null ($p{=}0.813$), \texttt{gc\_cross}
($p{=}0.558$) and \texttt{a\_loop} ($p{=}0.176$) fall to fingerprints, and two features are undefined on $513$
of $530$ designs because a bare H-type pseudoknot has no non-crossing pairs. One survives: \texttt{decoy}, the
count of off-target windows complementary to the crossing stem's $5'$ arm, of which agent designs carry $0.003$
each against our $0.060$.

That earned a transplant, and the transplant refutes it. \textbf{Our first attempt does not count and we say
so}: adding \texttt{decoy} as an acceptance tiebreak left the accepted designs' decoy count unchanged
($0.0757\!\to\!0.0761$), and by the standard of Fig.~\ref{fig:attrib} a transplant that changes nothing
measurable is a bug. Moving the intervention to the step that determines the feature, choosing the crossing
stem's bases to minimize decoys and still calling no oracle, engages properly: decoys fall to $0.0452$ and the
designs carrying any halve from $12$ to $6$. Carry-over then goes the wrong way, $16/177=0.090$
$[0.034,0.157]$ against the baseline's $0.108$, with a paired difference of ${-}0.031$ on the $162$ cells both
arms solved ($p{=}0.33$) and $r_{\mathrm{in}}$ at $35.4/60$.

\emph{A seventh test, distributional rather than scalar.} A reviewer asked the sharpest remaining question: a
difference could live in the \emph{whole distribution} of accepted designs while every hand-picked scalar came
back null. Answering it took three designs and the first two failed their own controls. Unstratified, all six
pre-stated quantities came back significant, the signature of a test measuring composition six times over.
GC-decile stratification did not repair it: the negative control, GC against itself within GC bins, came back
significant. Caliper-matching on GC ($\pm0.01$, no replacement, $399/530$ pairs) finally calibrates it, with the
control at $p{=}0.98$ and matched mean GC $0.4643$ against $0.4640$. \textbf{And a difference survives at
matched composition}: the agent designs' ensemble gap $(\mathrm{MFE}-\mathrm{EFE})/n$ is $0.005$ against our
$0.012$ (KS $D{=}0.466$), with longer unpaired loops ($\max$ $13.9$ against $10.8$). They concentrate the
Boltzmann ensemble on a single fold, and a dominant fold is one any model confirms
(\texttt{d3\_distributional.py}).

That is a mechanism-shaped lead, and the transplant kills it. Minimizing the ensemble gap where it is
determined, using a nested partition function outside every arm's objective and outside the counted allowance,
engages properly ($0.0126\to0.0080$ per design, $60\%$ of the way to the agents' value). Carry-over does not
follow: $15/151=0.099$ $[0.043,0.167]$ against $0.108$, and on the $141$ cells both arms solved the paired
difference is \textbf{exactly zero} ($b{=}7$, $c{=}7$, $p{=}1.0$), with $r_{\mathrm{in}}$ falling $37.0\to30.2$
(\texttt{dominant.sh}, \texttt{dom\_stats.py}). Two independent leads, both composition-controlled, both dead
under intervention.

\begin{figure}[t]\centering
\includegraphics[width=\linewidth]{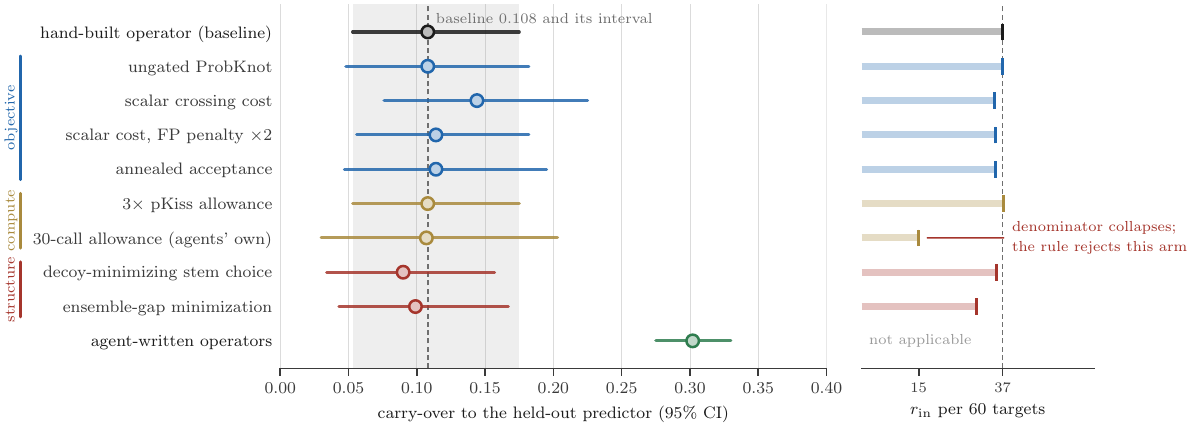}
\caption{\textbf{Nine interventions on the hand-built operator, against the arm they are trying to
reproduce.} Each row is that operator with exactly one thing changed, run deterministically at $5$ seeds.
Carry-over is $P(\text{ShapeKnots confirms}\mid\text{the arm's own pKiss}\wedge\text{ProbKnot predicate
confirmed})$; intervals are percentile bootstraps resampling \emph{targets}, the resampling unit throughout,
because units inside a target are correlated. Three standards are fixed here and invoked by reference
elsewhere. (i) Two conditional rates on different denominators are not a comparison, so every cross-arm claim
is made paired on the cells both arms solved. (ii) An independence-assuming $p$-value on these units is
anticonservative and is not quoted. (iii) A mechanism counts only if carry-over rises \emph{and}
$r_{\mathrm{in}}$ holds, since a rate bought by collapsing its own denominator is not an improvement; the
$30$-call arm is the case that rule exists to catch, and a transplant that changes nothing measurable is a
bug rather than a null. Every interval contains the baseline's and none reaches the agents' lower bound of
$0.275$. Exact counts are in the committed outputs the panels are drawn from
(\texttt{attr\_stats.txt}, \texttt{tinyctl\_stats.txt}, \texttt{ndc\_stats.txt},
\texttt{dom\_stats.txt}).}\label{fig:attrib}
\end{figure}

\paragraph{What seven interventional tests bought, which is not what we set out to buy.} The mechanism hunt
failed: the advantage is robust, is not compute in either direction, is not legible in the source, and is not
reproduced by either property that distinguishes the arms at matched composition. We had been reporting that as
a confession. It is better read as the paper's most portable result, and reading it that way costs no evidence.

\emph{Four times in this study a statistically supported and mechanistically plausible attribution was wrong,
and a control fixed in advance caught it.} (i) \texttt{decoy} differs between the arms at
$p{=}2.1\!\times\!10^{-4}$ and predicts carry-over \emph{within} the hand-built arm at
$p{=}2.6\!\times\!10^{-3}$ after adjusting for composition; biasing toward it moves carry-over from $0.108$ to
$0.107$ and reverses the sign. (ii) The ensemble gap differs at matched composition with $D{=}0.466$ under a
calibrated null and has a clean mechanism, and biasing toward it gives a paired difference of \emph{exactly}
zero. (iii) Two successive designs of the distributional test were confounded by GC, and in both cases the
negative control, not our judgement, said so. (iv) The agent source asks for $15$ restarts $\times$ $3000$
iterations against our $6\times800$, which reads as ``they searched harder''; counting realized calls shows
they spend $4.6$--$10\times$ \emph{fewer}.

The standard way to attribute a capability difference in a learned system is step one of what we did: find a
property that differs between the systems and correlates with the outcome, then name it the mechanism. On this
dataset that procedure would have produced two confident, well-powered, mutually independent false mechanisms,
each with a publishable $p$-value, plus two further false conclusions from a plausible confound and a source
constant. What separated them from real mechanisms was \emph{intervening} on each and re-measuring under a rule
fixed in advance. We therefore report the two-stage gate as a method rather than as our misfortune:
observational screen, then transplant, with the transplant's \emph{engagement} verified before its outcome is
read. Its cost is one extra sweep per candidate; its benefit is that seven candidates are excluded by
measurement, and no exclusion rests on our judgement about which properties matter.

None of this identifies the mechanism, and we do not soften that: a \emph{discovery}-grade claim on this axis
needs a transferable design principle, and one cannot transfer what one cannot identify. So we continue to
claim no attribution, and no transfer either: on the structurally
disjoint pool $9.7/27$ and $1.1/27$ targets error out per cell, leaving a denominator filtered towards easy
targets, and rates there are recorded in App.~\ref{app:invent} and used for nothing. Only two of six models
had no runtime failures; the other four time out on $19$--$32$ of $60$ targets, and that exclusion is itself a
selection we cannot fully audit.

Fairness requires putting these operators on the same axis we used to bound our own result, and they survive it
better than we do. Of the panel operator's $r_{\mathrm{mix}}$ wins $85\%$ fall on zero-margin targets; for the
two replicated agent operators the figures are $76\%$ and $63\%$, so on the $30$ targets a random compatible
sequence never reaches they take $45$ and $54$ wins against our $7$. The advantage is therefore not an artifact
of picking easier targets. It is \emph{larger} on the stratum where search is actually required
(\texttt{runs/cc\_ascent/pool\_margin.py}).

\subsection{What the Target Sets Will and Will Not Bear}\label{sec:exp:pools}

\paragraph{How much of this is the method, and how much is which targets are easy?} Every rate in this
section is a property of a method \emph{and} of a target set, and we had no way to separate them, because the
only difficulty statistic we carried, the crossing-pair count, describes the target's topology and not how hard
it is to hit. So we built a search-free probe. Draw $60$ sequences merely \emph{compatible} with
$T$ (Watson--Crick or GU at every paired position, uniform elsewhere), fold each under pKiss, and record
\begin{equation}\label{eq:margin}
  \Delta_{\mathrm{bp}}(T)\;=\;\min_{x\ \text{compatible}}\ \bigl|\,\mathrm{pKiss}(x)\ \triangle\ P(T)\,\bigr|
  \;\big/\;|P(T)| ,
\end{equation}
how close the best of them lands with no optimizer involved (\texttt{runs/cc\_ascent/pk\_margin.py}).
The $60$ draws are a budget, not a property of the target, and the direction that matters is known: sampling
more can only \emph{lower} $\Delta_{\mathrm{bp}}$, never raise it, so a larger probe would move targets into
the $\Delta_{\mathrm{bp}}{=}0$ stratum and never out of it. The $84\%$ below is therefore a \emph{lower}
bound on how much of $\Delta_{\mathrm{gamer}}$ sits where search is not required, and the residual
${+}0.016$ an \emph{upper} bound on what the operator buys. Both errors run against us, which is why we did
not spend more compute enlarging the probe.

It is a better predictor of our own failures than anything we optimized. $\Delta_{\mathrm{bp}}$ separates the
targets we solve under pKiss from those we do not at $\mathrm{AUC}=0.824$, against $0.780$ for length and
$0.624$ for the crossing-pair count, and it is not a restatement of that count
($r={+}0.124$). The consequence for the headline rate is blunt: \textbf{the median of our $43$
pKiss-solves has $\Delta_{\mathrm{bp}}=0$}. Thirty of the $43$ are hit \emph{exactly} by at least one of $60$
random compatible sequences. For most of what we count as solved, the operator is not what got us there.

\begin{figure}[t]\centering
\includegraphics[width=0.92\linewidth]{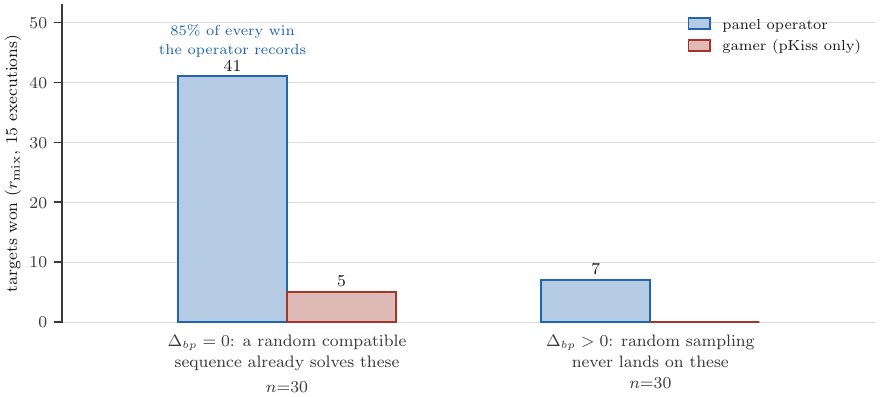}
\caption{\textbf{The largest caveat this paper raises against its own positive result.} Discordant wins over
the $15$ executions, with the $60$ targets split by whether any of $60$ merely \emph{compatible} random
sequences already hits the target exactly. Almost the whole separation between the panel operator and the
single-oracle gamer sits on the half where random sampling has already succeeded, so on most of what we count
as solved the operator is not what got us there. The residual on the other half ($7$ to $0$, exact McNemar $p{=}0.016$) is what
\S\ref{sec:exp:pools} reports as the size of the effect. Strata and counts parsed from
\texttt{pk\_margin.txt}.}\label{fig:margin}
\end{figure}

Stratifying $\Delta_{\mathrm{gamer}}$ by whether the probe can reach the target at all (Fig.~\ref{fig:margin}): on the $30$ targets
with $\Delta_{\mathrm{bp}}=0$ the panel operator takes $41$ discordant wins to the gamer's $5$
(${+}0.080$); on the $30$ with $\Delta_{\mathrm{bp}}>0$ it takes $7$ to $0$ (${+}0.016$, exact McNemar
$p{=}0.016$). So $84\%$ of the headline effect sits where a random sequence already succeeds. Something does
survive where search is genuinely required, and we now report \emph{that} as the size of the effect. This is
the fourth independent measurement pointing the same way: undirected search matches guided search, the
decomposition scaffolding is inert, the objective ablation is unattributable, and now the separation is
concentrated on zero-margin targets. The consistent reading is that the dimension we spent our design effort
on is not the one that decides the outcome. What this does \emph{not} touch is the oracle-gaming measurement:
$43\!\to\!1$ and the paired $2/43$ versus $26/43$ carry-over are statements about adjudication, and whether a
design survives a second predictor is a different question from whether its target was easy to hit.

\paragraph{Which pool is which, and a structurally disjoint evaluation.} ``Held-out'' and ``transfer'' have
been doing too much work, so the pools are named once (Table~\ref{tab:pools}) and those names used throughout.
Two facts belong in the open. The \emph{primary $n{=}60$ pool is a development/evaluation pool}: the
operator's design and every hyperparameter were fixed on it, so its rates are in-development performance. And
the ID-disjoint pool $[126{:}250]$ is disjoint by target \emph{identity}, not by structure: $39/125$ of its
targets are exact structural duplicates of a touched target and $94/125$ lie within $15\%$
(App.~\ref{app:data}). So results there are an ID-held-out evaluation, not a transfer claim.

So we built the split the evidence calls for: cluster all $251$ crossing targets by single linkage at $15\%$
normalized dot-bracket distance and discard every cluster containing \emph{any} target from ids $0$--$125$.
\emph{The first thing this measures is the benchmark.} $251$ targets $\to$ $60$ clusters $\to$ \textbf{$24$
clean ones holding $27$ targets} (minimum realized distance $0.167$): Pseudobase++ is \emph{structurally
exhausted}, and no clean transfer test on it can be well powered at a $\sim\!5\%$ base rate, where the
expected count is $\approx\!1.4/27$. That is a property of the benchmark, not of the method, and we suspect it
constrains other work reporting on it.

\emph{Within that limit the direction survives cleanly, and this is the strongest single piece of evidence for
it.} Over $5$ seeds $\times$ $3$ executions the panel operator is mixed-panel-consistent on $1.87/27$
($6.9\%$) and the gamer on $\mathbf{0.00/27}$: zero in \emph{all fifteen} executions; pooled discordance
$b{=}28$ against $c{=}\mathbf{0}$, exact McNemar $p{=}7.5\!\times\!10^{-9}$, cell-level sign test $15/15$
($p{=}6.1\!\times\!10^{-5}$). The operator's rate on structurally novel targets is \emph{not lower} than on
development ($5.3\%$), and the search-free probe says this pool is measurably harder (median
$\Delta_{\mathrm{bp}}$ $0.176$ against $0.059$; $15\%$ of targets reachable by a random compatible sequence
against $50\%$). Scoring higher on a harder pool is the opposite of the memorization signature a structural
split is built to detect. We deliberately do \emph{not} lean on the interval ($[{+}0.030,{+}0.116]$ at
$n{=}27$), which is underpowered by construction; and the sign test licenses ``a fresh execution on
\emph{these} structurally novel targets'', not ``on a fresh structurally novel target set''. The honest
summary: $\Delta_{\mathrm{gamer}}$'s \emph{sign} now has a structurally disjoint replication, its
\emph{magnitude} remains unestimable on any pool this benchmark can supply.

\begin{table}[t]\centering\small
\setlength{\tabcolsep}{4.5pt}\renewcommand{\arraystretch}{1.15}
\caption{\textbf{The pools, named by their actual contact with development.} We use these names throughout and
retire the earlier label ``strictly untouched'', which invited a structural reading the index split does not
support. The last row is the structure-aware split the evidence calls for; we built and ran it, and its size is a finding about the benchmark. The $\Delta_{\mathrm{bp}}$ column is the pool's median
search-free margin (Eq.~\eqref{eq:margin}, $60$ random compatible sequences per target): the pools are
\emph{not} difficulty-matched, and two readings in this paper turn on that
(\texttt{runs/cc\_ascent/pool\_margin.py}).}\label{tab:pools}
\begin{tabularx}{\linewidth}{@{}>{\raggedright\arraybackslash}p{2.7cm}>{\raggedright\arraybackslash}p{1.3cm}>{\centering\arraybackslash}p{1.15cm}L@{}}
\toprule
\textbf{name used here} & \textbf{ids\,/\,$n$} & \textbf{med.\ $\Delta_{\mathrm{bp}}$} &
\textbf{contact with development, and what it can support}\\
\midrule
development / evaluation pool & $0$--$59$ ($60$) & $0.059$ & operator design and \emph{all} hyperparameters fixed here.
Supports in-development comparison between arms; \textbf{not} a held-out benchmark, and we do not call it one.
Half of it ($30/60$) has $\Delta_{\mathrm{bp}}{=}0$ (Eq.~\eqref{eq:margin}): reachable by a random compatible
sequence, and carrying $84\%$ of $\Delta_{\mathrm{gamer}}$\\
inspected secondary pool & $60$--$109$ ($50$) & --- & inspected during development but not tuned on. Supports a
same-regime replication check\\
invention training pool & $110$--$125$ ($16$) & --- & the agent-invention loop's training set (App.~\ref{app:invent})\\
ID-disjoint follow-up pool & $126$--$250$ ($125$) & $0.000$ & untouched by development, tuning and model feedback, but
\emph{structurally overlapping} with the touched set ($39/125$ exact duplicates, $94/125$ within $15\%$;
App.~\ref{app:data}). Supports an \emph{ID-held-out evaluation}; does \textbf{not} support a transfer claim\\
cross-source pool & bpRNA ($60$) & $\mathbf{0.581}$ & independent source; both our operator and the SOTA
solver score $0/60$, on a pool an order of magnitude harder than development on the search-free
coordinate, which is most of why\\
\midrule
\textbf{structurally disjoint pool} & $24$ clusters ($27$) & $\mathbf{0.176}$ & whole structure clusters
assigned, keeping only
clusters with \emph{no} touched member; every target $>15\%$ from every developed-on target. Supports the
\emph{direction} of $\Delta_{\mathrm{gamer}}$ ($15/15$ executions, gamer $0/27$ throughout;
\S\ref{sec:exp}); too small for its magnitude, and that smallness is a fact about Pseudobase++, whose
$251$ targets contain only $60$ structural clusters. \emph{Harder} than development ($15\%$ of its targets
reachable by a random compatible sequence against $50\%$), which is why the panel operator's higher rate there
counts for more, not less\\
\bottomrule
\end{tabularx}
\end{table}

\subsection{What Every Rate Here Is Bounded By, and the Unbounded Family}\label{sec:exp:stats}
\paragraph{The panel is a low-recall instrument, and that bounds every number above.} On $80$ native bpRNA
pseudoknots, folding each native's \emph{own} sequence, per-predictor crossing recall is pKiss $20\%$, ProbKnot
$4\%$, ShapeKnots $11\%$, and the three-way conjunction $\mathbf{1\%}$ ($1/80$). Two consequences we carry
throughout: a low $r_{\mathrm{mix}}$ is weak evidence of \emph{absence} of capability, and $r_{\mathrm{mix}}$
is a high-stringency consistency measure, not a lower bound on true capability. The predictors also have
\emph{no fixed stringency order}: on natives it is pKiss\,$>$\,ShapeKnots\,$>$\,ProbKnot and on designs it
inverts, because agreement measured on designs is dominated by model-family sharing between the predictor and
the objective that produced the design.

\emph{What can an instrument this insensitive resolve?} Two things must be kept apart. \emph{Resolution} is a
property of $n$, the pairing and the base rate, and recall does not enter: simulating the design actually run
($60$ targets $\times$ $5$ seeds, target-clustered bootstrap, base rate $4/60$; \texttt{resolution.py}) gives
$80\%$ power at an observed gap of $\approx\!6/60$, so it separates $4/60$ from $\approx\!8/60$ and $5/60$ from
$15/60$ at power $1.00$, but not $4/60$ from $5/60$. That places our own headline: $\Delta_{\mathrm{gamer}}$ at
${+}0.067$ sits at $77\%$ power, at the design's resolution limit, which is ``established, narrowly'' in units
a reader can check. \emph{Attenuation} is where recall enters: an observed gap $d$ stands for a true gap
$d/\rho$, so low recall rescales the comparison rather than blurring it and makes every rate conservative. The
rescaling assumes $\rho$ is common to both arms, which holds for $r_{\mathrm{mix}}$ and not for carry-over
whose point is that $\rho$ differs by arm; we use it only for the former.

\emph{Generalization, and its limit.} On a second disjoint Pseudobase pool the regime repeats ($0/50$ floor,
$5/50$ panel-unanimous); on a cross-source bpRNA pool our operator \emph{and} the SOTA solver both reach
$0/60$. The margin probe says most of that is the pool: its median $\Delta_{\mathrm{bp}}$ is $0.581$ against
the development pool's $0.059$, and $5\%$ of its targets are reachable by a random compatible sequence against
$50\%$ there. The negative side generalizes across datasets; the positive claim is limited to the Pseudobase++
distribution. Statistical conventions, the leave-one-target-out check, the quasi-held-out/exploratory split and
the per-oracle cross-source breakdown are in App.~\ref{app:repro} and App.~\ref{app:robust}.

\paragraph{The one axis that is not structurally exhausted, run at scale.} Every magnitude above is
unestimable on Pseudobase++, and the resolution calculation says why in units: the design reaches $80\%$ power
at an observed gap of $\approx\!6/60$, and $\Delta_{\mathrm{gamer}}$ is $4/60$. The parameterized family
$\mathcal L_{\mathrm H}=\{T_{n,m}\}$ has no such ceiling. It is non-context-free for every $m\ge1$, it is the
family the class-level reading is about, and it is unbounded, so power there is a budget decision rather than a
property of a benchmark. An earlier revision ran $16$ of its members at one seed and filed them in an appendix
as not relied on; that was the wrong call. We run the full grid $n,m\in\{3,\dots,12\}$: $100$ targets, five
seeds, deterministic, $1500$ cells (\texttt{family\_grid.py}, \texttt{family\_stats.py}).

The regime reproduces and sharpens. The \CF{} floor is $\mathbf{0/300}$, so the negative side's empirical
counterpart holds at five times the pool size. The gamer reaches $500/500$ under the oracle it optimizes and
$9/500$ under three, on a set where every target is non-\CF{} by construction. The panel operator reaches
$167/500$. Paired within (target, seed) and clustered on target, $\Delta_{\mathrm{gamer}}={+}0.316$
$[{+}0.258,{+}0.378]$, $b{=}159$ against $c{=}1$, exact McNemar $p{=}2.2\!\times\!10^{-46}$. \textbf{That is the
first estimable magnitude in this paper}, five times the Pseudobase++ estimate, which is a statement about the
two target sets and not about the operator.

\begin{figure}[t]\centering
\includegraphics[width=\linewidth]{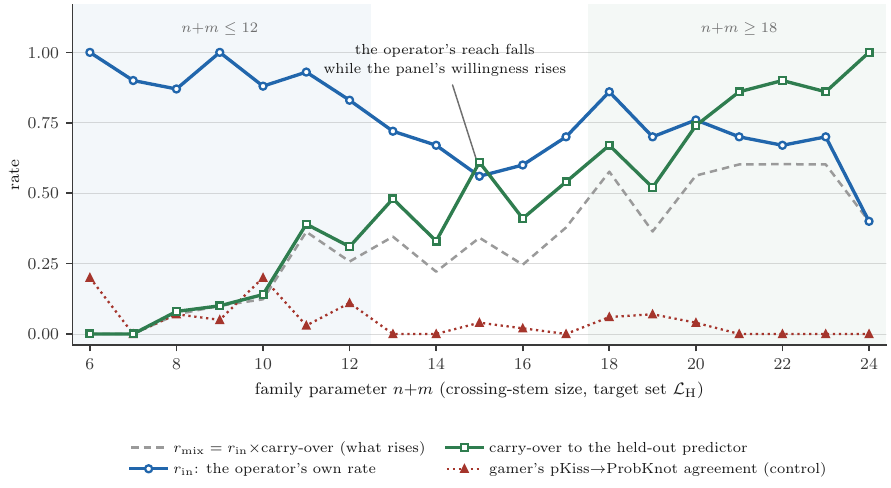}
\caption{\textbf{Why the rising confirmed rate on $\mathcal L_{\mathrm H}$ does not mean what it looks like.}
The panel-unanimous rate $r_{\mathrm{mix}}$ (grey) rises with the family parameter, which reads as a
fixed-description operator holding up as the family grows. It is a product, $r_{\mathrm{mix}}=r_{\mathrm{in}}
\times\text{carry-over}$, and only the second factor rises: the operator's own reach \emph{falls} between the
shaded strata while carry-over to the held-out predictor more than triples, because longer crossing stems are
thermodynamically dominant enough that a predictor which never saw them agrees anyway. The gamer's
pKiss$\to$ProbKnot agreement (red) is the control that blocks the obvious objection, since it moves the
\emph{other} way: ``longer stems make predictors agree'' is not a blanket effect, and satisfying two
predictors on a long stem selects sequences whose crossing stem dominates where satisfying one does not. $100$ targets $\times$ $5$ seeds, deterministic (\texttt{family\_grid.py},
\texttt{family\_stats.py}).}\label{fig:family}
\end{figure}

\emph{The success-vs-size curve, and the diagnostic that stops us mis-reading it.} Panel-unanimity rises with
the family parameter, from $0.186$ at $n{+}m\le12$ to $0.529$ at $n{+}m\ge18$. That reads as a
fixed-description operator holding up as the family grows, and Fig.~\ref{fig:family} refutes it: the
operator's own success rate \emph{falls} with size ($0.900\to0.736$) while carry-over to the held-out
predictor rises sharply ($0.206\to0.718$). What grows is not the procedure's reach but the panel's
willingness to confirm. The honest reading is narrower and still useful: there is no size
cliff in the confirmed rate up to $n{+}m{=}24$, which is what a uniform procedure needs and a memorized table
cannot supply, but the operator does have a size-related limitation and it is visible in $r_{\mathrm{in}}$.
The gamer's pKiss$\to$ProbKnot agreement moves the \emph{other} way with size ($0.093\to0.036$), which rules
out a blanket size effect on predictor agreement and is the clearest support in this paper for the
success-predicate reading.

\emph{Scope, stated beside the numbers rather than under them.} These rates are not comparable to the
Pseudobase++ rates and must not be quoted alongside them. This is one synthetic distribution with a uniform
topology and no natural sequence constraints; it is also the exact topology the operator's decompose--compose
scaffold was built for, a home-field advantage we name rather than discount; and small-size failures may be
undesignability, since we ran no designability reference on this grid. What the family buys is the one thing
Pseudobase++ cannot supply at any $n$: a magnitude, on targets sharing no structural cluster with anything we
developed on. It buys no generalization.

\paragraph{Where this leaves the ascent.} An operator optimizing panel agreement beats the gamer's $1/60$, if
modestly: $3$--$8/60$ across five seeds, $1$--$9/60$ over all $24$ executions, and
$2$--$17/60$ for agent-invented operators (App.~\ref{app:invent}), above a decidable $0/60$ floor. These are
mixed-panel-consistent, autonomous crossings on the Pseudobase++ distribution; we do not claim they are the
first in the literature, only the first we observe under this protocol. On those designs the target is the
pKiss MFE for all $17$ whose suboptimal enumeration completed, dominant by a median $1.2$\,kcal/mol, so the
positive verdict is not a marginal coincidence of pair sets, within the pKiss energy model, and on the
shorter designs only. The remaining gap ($\sim\!5/60$) is itself the finding: robust autonomous pseudoknot
design is hard, and two-oracle agreement generalizes to a third only partially.

\section{Discussion, Related Work, and Conclusion}\label{sec:discussion}

\subsection{Related Work}\label{sec:related}
Our novelty is a conjunction; we position against each line and claim the bundle.

\paragraph{Self-improving agents, and the one-sided statistics that adjudicate them.} Agentic systems propose,
run and refine scientific work end-to-end \citep{coscientist2025,agenticscience2025survey}, self-improving
agents grow skill libraries or edit their scaffolds \citep{voyager2023}, and categorical self-revising
discovery \citep{selfrevising} gates on an MDL criterion. Throughout, evidence for a capability gain is a
benchmark delta or a novelty score: one-sided statistics. Four adjacent lines share that shape.
\emph{Demarcation}: falsifiability demarcates science \citep{popper1959,lakatos1978}, and AI results have been
argued ``not born scientific'' until falsified \citep{falsify2025}; sequential falsification
\citep{huang2025popper} and falsification-centred discovery \citep{liu2024aigs} operationalize this with
deliberately inexact evidence, where our negative side is a decidable theorem. \emph{Compression}: a deep
tradition equates discovery with compression \citep{grunwald2007mdl}, but a shorter description improves fit
without establishing impossibility for a prior repertoire, which decidable non-membership supplies.
\emph{Verifiable rewards}: RLVR and process verification \citep{lightman2023verify} assume a reward cannot be
gamed, yet reward hacking and proxy over-optimization are well documented
\citep{amodei2016concrete,skalse2022,gao2023overopt}; our $43/60\!\to\!1/60$ collapse is an instance.
\emph{Exact-verifier construction search}: FunSearch \citep{funsearch}, AlphaEvolve \citep{alphaevolve},
AlphaTensor \citep{alphatensor} and AlphaDev \citep{alphadev} couple a model to an exact evaluator for scalar
improvement over a fixed objective, with no negative side (AlphaTensor emits upper bounds later beaten, never
a lower bound).

\paragraph{The nearest criterion, placed on the D1--D3 rungs.} \S\ref{sec:intro} states the comparison with
categorical self-revising discovery \citep{selfrevising}; three points complete it. \emph{Mechanism}: an
information criterion is computed from the revised system and the evidence that system was fitted on, so no
predictor is held outside the loop, and a system that improves its own compression cannot be distinguished
from one that adapted to the evidence it is scored on, which is the $43\!\to\!1$ collapse, exactly.
\emph{Reach}: as \citet{selfrevising} state, the residual measures the \emph{additional bits} a revision
needs, not the prior regime's incapacity, so there is no counterpart to our exact exclusion; this is a
difference in reported evidence, not in ambition, and D3 is reached by neither line. \emph{The checkable
part}: on the rungs where evidence can be demanded, this paper supplies a held-out adjudicator, a compute
control, matched baselines, ablations, intervals and its own two unestablished components; the case studies of
the categorical line report neither baselines nor ablations nor significance tests across roughly five runs.
We compare what each line reports, on one axis neither was designed against, and a reader who thinks regime
change is the more important question should discount this accordingly.

\paragraph{Independent-model validation, and the precedent we inherit.} Adjudicating a design with a model
that was not in the generator's loop is standard in protein design, not something we introduce: hallucination
and diffusion pipelines generate with one network and validate by \emph{self-consistency} against a separately
trained structure predictor \citep{anishchenko2021hallucination,watson2023rfdiffusion}. Our positive side
\emph{is} that protocol. What we add is (i) a formal reading of what such a check can license, (ii) a
\emph{measurement} of how far the in-loop and held-out verdicts diverge, the quantity that protocol leaves
unreported, and (iii) a decidable negative side, which those pipelines have no analogue of because their
verifiers have no exactly characterized range. Abstaining on disagreement connects to selective prediction and
conformal abstention, used here as a convention rather than a calibrated guarantee.

\paragraph{Chomsky hierarchy, formal RNA, and design complexity.} The hierarchy has benchmarked which
architectures generalize on which tier \citep{deletang2023}: a static characterization of models, not an
audited gain by a design system. A concurrent line maps an agent's memory architecture to automaton classes
\citep{agentsautomata2025}, classifying the computational substrate; we instead classify the \emph{target
family}. RNA's formal-language structure is classical (SCFGs \citep{rivaseddy2000,durbin1998}, non-context-free
pseudoknots \citep{nebelweinberg}, the MCFG characterization \citep{katosekikasami2006}) and to our knowledge
has not been repurposed as an audit axis for a claimed capability gain. Pseudoknot folding is NP-complete
\citep{lyngso2000pknot}, inverse design NP-hard \citep{bonnet2020hard}, and undesignability admits exact
certificates \citep{zhou2024rigende}; we use these as ground truth and reuse the verifier ladder
\citep{verifiereconomics}, whose measured oracle-dependence is what forces our panel-plus-abstention positive
side.

\subsection{Discussion and Limitations}\label{sec:disc}
\paragraph{What would overturn each claim.} \emph{Negative side}: only a \CF-verified repertoire certified to
realize a crossing target, impossible since $\operatorname{im}\Ocf\subseteq\NESTED$. This does \emph{not} say
the prior policy structurally cannot generate a crossing; that floor is a separate, empirical, fixed-budget
$0/60$. \emph{Positive side}: crossing performance matching the \CF{} floor, or holding only under one lenient
predictor. \emph{Verified non-use}: a packaged solver found on the method path. \emph{Family-level class
claim}: $\mathcal L_{\mathrm H}$ turning out context-free, which it is not, by the MCFG-2 witness. We claim no
individual realized crossing is non-\CF, since a finite set of designs is regular; that claim is about
$\mathcal L_{\mathrm H}$ and never about a design.

\paragraph{What a cross-predictor-consistent result requires, and what we did not get.} Making the success
predicate a conjunction does help, and the help is small and does not survive attribution. At matched caps and
matched access an undirected walk is indistinguishable from the directed operator on both $r_{\mathrm{mix}}$
and carry-over, so ``optimize agreement'' is supported as a \emph{choice of success predicate} and not as a
search objective that earns its keep. Nor did we get generalization: the held-out predictor separates weakly,
cross-source agreement is $0/60$ for our operator and the SOTA solver alike, and carry-over follows shared
model family, not correctness. A stronger result needs a fourth predictor from a different model class, a
second structurally independent corpus, and a compute-matched control. The first two remain unavailable on
this benchmark; the third is now run (\S\ref{sec:exp}), and an arm held to the agents' own realized oracle
volume carries over at $0.107$ where they reach $0.367$.

\paragraph{What we mean by ``new'', and the three rungs.} ``New'' is relative to the prior self, not the base
model, whose latent competence is unauditable and beside the point. Our operator is oracle-guided search over
known moves, and at $1/60$ panel-unanimous we have not established even a robust acquisition on this pool.
But ``discovery-grade'' is not one bar. The evidence separates into rungs reachable independently
(Table~\ref{tab:rungs}), and the agent-written operators reach \textbf{D1} and \textbf{D2} and not
\textbf{D3}. D1 is paired and difficulty-matched on the $951$ units both arms solved under the shared in-loop
predicate: $0.293$ against $0.095$, target-clustered $[{+}0.108,{+}0.297]$, cluster-permutation
$p{=}5\!\times\!10^{-5}$. Two bounds go with it. The claim is about the $2$ of $6$ operators that ran without
failures, and D1 is a between-system association on one pool.

\emph{Why the rung split is not special pleading.} It would be, if we had invented the rungs after seeing
which we cleared. The order is the reverse: D3 is the bar an earlier draft set for itself, and we still fail
it. Splitting it changes no verdict about our own system, which reaches none of the three, since its $0.091$
is what the control is measured \emph{against}. It lets an arm that clears two rungs be reported as clearing
two rungs.

\paragraph{Complementary to the categorical self-revising line.} \S\ref{sec:related} places the two on the
D1--D3 rungs. The complementarity is concrete. A fit-plus-description-length criterion \emph{accepts} our
$43/60$ designs, because they fit the optimized oracle exactly and are produced by a short procedure, and only
$1/60$ survives three predictors. That $42$-target gap is invisible to any statistic computed from the system
and its own oracle, so a one-sided gate and a two-sided audit are layers and not competitors: the categorical
line equips the revision layer, this work the audit layer.

\paragraph{Why this is one paper and not two.} Splitting the audit framework and the $43\!\to\!1$ measurement
from the RNA operator's ablation and attribution would damage both halves. An instrument paper whose only
demonstration is that nothing passes cannot distinguish stringency from insensitivity, which is the argument
of \S\ref{sec:exp}'s positive control, and that control \emph{is} the application. Conversely the operator
results are uninterpretable without the instrument that grades them, since the same $43/60$ reads as an ascent
or as an artifact depending on which adjudicator is asked.

\paragraph{An audit layer for AI4S, and what it degrades to.} The lesson generalizes because the vulnerability
does. Most scientific oracles a self-improving system optimizes against are fallible surrogates: a docking
score, a force field, a learned reward, a structure predictor, a theorem-search heuristic. A closed loop can
improve the oracle's score without improving the capability, as ours drove pKiss to $43/60$ while a
three-predictor panel saw $1/60$. The protocol ports to molecular generation adjudicated against DFT or
assays, to protein design against binding readouts, and to theorem discovery against an exact checker. Where
no decidable range relation exists it degrades mostly to ordinary good practice. Without a verifier whose
image can be characterized, $E_{\mathrm{range}}$ has no analogue, the negative side collapses to an empirical
floor, and what remains is held-out adjudication, matched ablations, a compute control and interval reporting.
Most capability claims in this literature report none of those, and the $43\!\to\!1$ collapse is visible with
them alone, so the residue is not empty. It is not this paper's distinctive claim, though, and ``two-sided,
one side decidable'' should be read as verified in exactly one domain. The exact side ports where an exact
checker exists, as in theorem search or SAT-style constraint classes, and it does not port to a docking score
or a binding assay.

\paragraph{Limitations, in rough order of how much they bound the contribution.} \emph{(i) One benchmark,
structurally exhausted, plus one synthetic family that repairs power and nothing else.} The unbounded
$\mathcal L_{\mathrm H}$ grid gives an estimable magnitude but is a single uniform topology built for the
operator's scaffold, so it speaks to power, not generalization. Every natural-corpus number is on
Pseudobase++, whose structure-aware split has $27$ targets in $24$ clusters, so directions replicate and
magnitudes do not, and the cross-source pool gives $0/60$ for our operator and the SOTA solver alike.
\emph{(ii) The panel is mechanistically narrow, and we measured that.} All three predictors share
nearest-neighbour parameters; carry-over follows that shared family, not correctness (Fig.~\ref{fig:carry});
and pKiss and ShapeKnots form a correlated block ($\kappa=0.673$), so a unanimous verdict is worth less than
three votes. Whether ProbKnot is an independent third voice is undetermined at $n{=}80$. No wet-lab
confirmation is attempted. \emph{(iii) Attribution is partial and much of $\Delta_{\mathrm{gamer}}$ is
target-side}: a search-free probe puts $84\%$ of it on targets a random compatible sequence already solves,
leaving ${+}0.016$ where search is required. \emph{(iv) Most numbers were produced by operators that were not
bit-reproducible}, and our first diagnosis of why was wrong (App.~\ref{app:repro}); the ablation and headline
comparison are re-run bit-reproducibly, the pool-level and cross-source numbers are not. \emph{(v) The
invention step is not fully reproducible}: the operators were written by frontier models behind a routed alias
exposing no version hash, so a re-run may not query the same weights. Every measurement re-executes frozen
source, so nothing reported depends on it, but the claim ``an agent can invent such an operator'' does.
\emph{(vi) The audit does not separate capability from search}, by construction.

\paragraph{The experiments this study owes, in priority order.} \textbf{(1) A second, structurally
independent crossing corpus.} Pseudobase++ holds $251$ targets in $60$ structural clusters, only $24$ disjoint
from our development, so the \emph{magnitude} of every transfer claim on it is unestimable and no reanalysis
fixes that. \textbf{(2) Power}, partly discharged and only where the benchmark allows it. The unbounded family
gives $100$ targets $\times$ $5$ seeds and an estimable ${+}0.316$ $[{+}0.258,{+}0.378]$ (\S\ref{sec:exp});
on the natural corpus the limit is $24$ clusters, not compute. \textbf{(3) A non-thermodynamic predictor},
which is the one we would run first. All three current members share nearest-neighbour parameters, and
\S\ref{sec:oracle} measures the cost, with pKiss and the held-out ShapeKnots forming a correlated block at
$\kappa=0.673$. A learned predictor would break that block rather than add to it, and would test the one thing
our panel structurally cannot, namely whether ``carries over'' means anything beyond ``agrees with a
thermodynamic model''. None is installed in the environment these runs were produced in, and we do not report
an untested integration as an arm. \textbf{(4) An explanation for the positive control}, narrowed rather than
open. Measurement has excluded oracle-call volume in both directions, target selectivity, GC composition,
denominator collapse, four mechanisms legible in the source, and two properties that differ at matched
composition and then reproduce nothing under intervention (\S\ref{sec:exp}). Of the directions put to us,
finer-grained structural features are what those last two already were; information-theoretic measures are
different in kind but underpowered at $185$ designs per arm, which makes them a consequence of~(2); and
accumulating statistics on the unbounded family is the one that dissolves that obstacle. We name no further
candidate we cannot test here. What the difference \emph{is} is the open problem this paper hands on.

\subsection{Conclusion}\label{sec:conclusion}
We proposed an audit that \emph{separately reports} an empirical gain and an exact verifier-range exclusion
instead of collapsing them: on one side a measured before/after difference under a shared adjudicator, on the
other a decidable proof that the prior repertoire's pseudoknot-free verifier could not certify the targets at
all. We fixed the object a pumping lemma can constrain (the realizable-target language under an oracle),
proved a decidable confinement theorem for the negative side, defined the transferable solver-free crossing
capability, separated repertoire-confinement from target-undesignability, and were explicit that the positive
side is model-relative.

\paragraph{The result, stated at the strength the evidence supports.} Under a single pseudoknot oracle an
invented, solver-free operator reaches $43/60$ on crossing targets, above a decidable context-free floor of
$0/60$; under three predictors $1/60$ survives. We report that collapse as \emph{oracle-specific
over-optimization}, and the discriminating evidence is paired: of the operator's $43$ pKiss solves a held-out
predictor confirms $2$, against $26$ for a minimum-free-energy solver on the same targets. That contrast is an
\emph{association} between what an arm optimizes and how far its designs carry, not a causal localization,
since the arms are two methods and not a randomized manipulation of one method's oracle. Re-optimizing for
cross-oracle agreement recovers a small set of mixed-panel-consistent crossings, and that gain is narrow,
shrinks as replication is added, and is not attributable to the search's guidance: matched undirected search
is indistinguishable from the full operator, and a search-free probe puts $84\%$ of it on targets random
sampling already reaches. Carry-over follows shared model family, which bounds what any panel of
thermodynamically related predictors can certify. The exclusion side is exact and concerns the prior
\emph{verifier's} range, never the prior policy's capacity.

\paragraph{And the positive side, at the same strength.} The instrument is not merely insensitive, and
establishing that produced this paper's one clearly positive measurement. Scored by the predictor held out of
every arm's loop, and compared on the targets they and our operator both solved under the same in-loop bar,
the frozen agent-written operators carry over at $0.293$ against our $0.095$ ($n{=}951$ paired units,
target-clustered $[{+}0.108,{+}0.297]$, cluster-permutation $p{=}5\!\times\!10^{-5}$), while spending
$4.6$--$10\times$ \emph{fewer} oracle calls per target. On the rungs of Table~\ref{tab:rungs} that is D1 and
D2 established and D3 open: seven candidate mechanisms transplant into ours without moving the statistic,
including the two that differ at matched composition, so what the difference \emph{is} remains unidentified
and we do not call it a discovery. It is bounded by the two of six operators that ran without failures and by
a single pool. An unexplained, compute-controlled, held-out-verified difference between an agent-written
procedure and a human-written one is a more useful thing to put on the record than either a flat null or a
claim the evidence does not carry.

\appendix
\section{Proofs}\label{app:proofs}
\begin{proof}[Proof of Proposition~\ref{prop:confine}]
The pseudoknot-free oracle \Ocf{} returns a minimum-free-energy secondary structure in the pseudoknot-free
ensemble; its representation admits no crossing pair, so $\Ocf(x)\in\NESTED$ for every sequence $x$. Hence
for any policy $\pi$, every $T\in\Rze{\Ocf}{\pi}$ satisfies $T=\Ocf(\pi(T))\in\NESTED$, i.e.\
$\Rze{\Ocf}{\pi}\subseteq\NESTED$. If $T\in\PK$ then $T$ contains a crossing pair, so $T\notin\NESTED$ and
therefore $T\notin\Rze{\Ocf}{\pi}$. Testing $T\in\PK$ (a crossing pair among $|P|$ pairs) is $O(|P|^2)$ and
independent of any search over sequences; the non-membership is thus decided offline.
\end{proof}

\smallskip\noindent\textbf{Decision procedure for the negative side (and its complexity).} Given $T$ with
pair set $P(T)$, the crossing test returns \textsc{crossing} iff $\exists\,(i,j),(k,l)\in P(T)$ with
$i<k<j<l$. The naive double loop is $O(|P|^2)$; sorting $P(T)$ by left endpoint and sweeping with a stack of
open right-endpoints (pushing $j$ at $i$, popping at $j$, flagging a violation if a new pair opens inside an
open interval but closes outside it) decides it in $O(|P|\log|P|)$. The verdict for the entire held-out pool
is computed \emph{once, offline}, before any sequence is designed. No folding, search, or evaluation over
$\Sigma^\ast$ is involved. This is the precise sense in which the negative side is a theorem discharged
before any run (contrast the positive side, which requires folding each committed design).
\noindent\textbf{Class boundary (invoked), and the grammar witness (worked).} That $\NESTED$ is context-free,
and that the crossing families we treat are non-context-free yet admit MCFG grammars of dimension $\le2$, are
established \citep{rivaseddy2000,nebelweinberg,katosekikasami2006}; we do not reproduce their proofs, and we
make no claim that every pseudoknot topology is captured at dimension~2. We do make the witness explicit for
the family the operator targets.

\smallskip\noindent\emph{The parameterized crossing family.} Let
\begin{equation}\label{eq:family}
  \mathcal L_{\mathrm H}=\bigl\{\,\texttt{(}^{\,n}\,\ell\,\texttt{[}^{\,m}\,\ell\,\texttt{)}^{\,n}\,\ell\,\texttt{]}^{\,m}\ :\ n,m\ge 1\,\bigr\},
\end{equation}
the canonical H-type pseudoknots ($\ell$ a fixed unpaired linker): a nested round stem of $n$ pairs crossed by
a square stem of $m$ pairs. Its pairing skeleton is the interleaving $\{a^{n}b^{m}c^{n}d^{m}\}$ under
$a{=}\texttt{(},\,b{=}\texttt{[},\,c{=}\texttt{)},\,d{=}\texttt{]}$.

\smallskip\noindent\emph{Non-context-freeness, and what is invoked versus verified here.} The general results
are invoked, not re-proved: that crossing arcs are not context-free is \citet{nebelweinberg}, and that
pseudoknot classes sit in low-dimension MCFG subclasses is \citet{katosekikasami2006}. What we do verify is the
witness for \emph{our} family, since we constructed it and no citation covers it: $\mathcal L_{\mathrm H}$
carries the two equalities $\#a{=}\#c$ and $\#b{=}\#d$ of the count language, and no
$uvwxy$ decomposition with $|vwx|\le p$ can span more than two of four blocks, so pumping breaks one equality
\citep{ogden1968,hopcroftullman1979}; hence no single-stack generator realizes the family uniformly. This
matters for one thing only, and it is the thing \S\ref{sec:exp}'s grid rests on: it is what licenses calling a
synthetic ladder a test of the same capability rather than an unrelated construction.

\smallskip\noindent\emph{An MCFG of dimension~2 that does.} Two dimension-2 nonterminals each emit a
\emph{pair} of substrings (the two ends of one stem), and the start rule interleaves them:
\begin{align}\label{eq:mcfg}
  R(a\,x_1,\ c\,x_2) &\leftarrow R(x_1,x_2); &  R(\varepsilon,\varepsilon) &\leftarrow \ \ \text{(base)};
  \nonumber\\
  Q(b\,y_1,\ d\,y_2) &\leftarrow Q(y_1,y_2); &  Q(\varepsilon,\varepsilon) &\leftarrow \ \ \text{(base)};\\
  S(x_1\,\ell\,y_1\,\ell\,x_2\,\ell\,y_2) &\leftarrow R(x_1,x_2),\,Q(y_1,y_2). && \nonumber
\end{align}
$R$ derives the round-stem pair $(a^n,c^n)$, $Q$ the square-stem pair $(b^m,d^m)$, and $S$ places them at four
interleaved slots, exactly the crossing a single yield cannot produce. Dimension~$2$ (each nonterminal
coordinates two substrings) is sufficient (the grammar above) and necessary for this family (dimension~$1$ is
exactly \CF, excluded by the pumping argument), i.e.\ the minimal machinery above
\CF{} \citep{seki1991mcfg,katosekikasami2006}.

\smallskip\noindent\emph{Definition, result, and inference, kept apart.} The sentence ``realizing
$\mathcal L_{\mathrm H}$ uniformly is an \MCF{} capability'' is a \textbf{definition}, and it quantifies over
\emph{all} $(n,m)$: a finite-description operator that, for \emph{every} $(n,m)$, returns a sequence whose fold
is $S$'s yield would stand in for the grammar itself, which no CF generator can. Our \textbf{experimental
result} is strictly weaker and finite: the frozen operator succeeds on the tested ladder
$n{=}m\in\{5,\dots,14\}$ and selected asymmetric $(n,m)$ up to $(14,3)$, at parameters $3$--$4\times$ beyond
those it was fixed on (\S\ref{sec:exp}). The \textbf{inference} between them is therefore an extrapolation, and
we label it as such: the result is \emph{consistent with} uniform coverage of the family and rules out a
lookup table of the development targets, but it does not establish the definition's universally quantified
claim. A finite program with length-parameterized templates would pass the same ladder without realizing every
member, and the three asymmetric-extreme failures show the operator is not in fact uniform at the tested
budget. We never write that the operator realizes $\mathcal L_{\mathrm H}$ uniformly; we write that it
transfers to tested members well beyond its development range.

\smallskip\noindent\emph{Worked derivation ($n{=}m{=}2$, linker $\ell{=}\texttt{....}$).} Eq.~\eqref{eq:mcfg}
derives $T_{2,2}=\texttt{((....[[....))....]]}$ in four steps: two applications each
build the stem pairs $(aa,cc)$ and $(bb,dd)$, then the start rule interleaves them across the four slots
separated by $\ell$. A CFG cannot reproduce this table: its single-yield nonterminals cannot carry the two
ends of a stem to separated, interleaved positions.

\smallskip\noindent\textbf{The pseudoknot classes form an MCFG-dimension hierarchy.} The dimension-2 grammar
above captures the H-type family (one crossing stem). Richer pseudoknots (kissing-hairpin interactions,
complex $\ge3$-stem entanglements) are widely expected to require higher MCFG fan-out, but the minimal fan-out
for each specific class, and the inclusion/incomparability relations among pseudoknot classes as formal
languages, are delicate questions we do not settle here \citep{rivaseddy2000,nebelweinberg,katosekikasami2006}.
We therefore assert the dimension-2 witness \emph{only} for the H-type family we design against and make
\emph{no} per-class dimension-hierarchy claim: in particular we do not claim that MCFG dimension grows with
the crossing-stem count. The grades order targets by a structural difficulty statistic (crossing-stem count), so the empirical frontier organizes by grade, without
asserting that statistic is a formal grammar grade. The $\CF\!\to\!\MCF$ boundary, the object of this
paper, is crossed already at the first tier (one crossing stem); higher tiers are a finer hierarchy
\emph{within} \MCF{} that we do not attempt to climb.

\begin{table}[h]\centering\small
\setlength{\tabcolsep}{5pt}\renewcommand{\arraystretch}{1.2}
\caption{The oracles as maps $\Sigma^\ast\!\to\mathcal T$: folding model, range, and pseudoknot subclass
modeled.}\label{tab:oracles}
\begin{tabularx}{\linewidth}{@{}lLll@{}}
\toprule
oracle & folding model & range $\operatorname{im}\Omega$ & subclass\\
\midrule
$\Ocf$ (ViennaRNA) & Turner-2004 MFE, pseudoknot-free ensemble & $\subseteq\NESTED$ & none (nested)\\
pKiss & heuristic min-energy folding with knot motifs \citep{janssen2015pkiss} & $\subseteq\mathcal T$ & H-type, kissing hairpin\\
ProbKnot & max-expected-accuracy pairs from a partition function \citep{bellaousov2010probknot} & $\subseteq\mathcal T$ & pairs above a probability cut\\
ShapeKnots & MFE under SHAPE restraints (here run \emph{de novo}) \citep{hajdin2013shapeknots} & $\subseteq\mathcal T$ & low-order pseudoknots\\
\bottomrule
\end{tabularx}
\end{table}

\smallskip\noindent\textbf{The folding oracles (models and output spaces).} Table~\ref{tab:oracles} fixes each
oracle as a map $\Sigma^\ast\!\to\mathcal T$ with an explicit range. $\operatorname{im}\Ocf\subseteq\NESTED$ (generally strict, as undesignable nested structures are unrealizable) is exactly
what Prop.~\ref{prop:confine} (via Lemma~\ref{lem:bound}) requires; the three panel predictors admit crossing
outputs but model \emph{different} pseudoknot subclasses under different objectives, which is the mechanistic source
of the $[15,60]\%$ cross-oracle envelope of \S\ref{sec:oracle} and hence of the abstention rule.

\section{Reproducibility and Protocol Details}\label{app:repro}
Oracles are computed and replayable under a fixed configuration; harness, target lists, intrinsic grades,
and per-target JSON are released. The \CF{} ladder is ViennaRNA (Turner~2004, $37^\circ$C); the pseudoknot
panel is pKiss (reference) with ProbKnot and ShapeKnots for the oracle-robustness envelope and abstention.
The sealed-commitment step records a hash and timestamp of each prediction before the held-out adjudicator is run; the
$\pm$solver ablation toggles the external pseudoknot solver while holding the agent fixed; the minimal-hint
sweep (E3), a planned minimal-hint probe, parameterizes how much of the target's crossing structure is
revealed and reports the smallest hint under which the operator crosses. Every measured number in this manuscript
enters only from such a logged run (\S\ref{sec:exp}); the harness, target lists, intrinsic grades, and
per-target JSON are released for replay.

\smallskip\noindent\textbf{Software versions, commands, and release.} The oracles are pinned:
ViennaRNA~$2.7.2$ (Turner~2004 parameters, $37^\circ$C) for the \CF{} ladder and inverse folding; pKiss~$2.3.0$
(Bellman's GAP compiler \texttt{bellmans-gapc}~2024.01.12, bioconda) as the reference pseudoknot oracle;
RNAstructure~$6.6$ for ProbKnot and ShapeKnots (the latter run \emph{de novo}, no SHAPE constraints); all under
Python~$3.10$. Each oracle runs in its own pinned conda environment; the harness records the exact command line
and the raw predictor output for every call. Design seeds are the integers $0$--$4$.

\smallskip\noindent\textbf{Execution-level nondeterminism at a fixed seed, and the fix.} Two independent runs
of the same arm at the same nominal seed gave different counts. Our first diagnosis, subprocess timing under a
wall-clock cap, was incomplete: the operators also call \textsc{ViennaRNA}'s \texttt{inverse\_fold}, a
stochastic search on \textsc{ViennaRNA}'s \emph{own} global RNG, which is invisible to
\texttt{random.Random(seed)} and advances from call to call. The operators now expose
\texttt{deterministic=True}, which replaces both wall-clock caps with counted ones (iterations plus a $150$-call
pKiss allowance, above the $142.7$ the hungriest arm realized under the clock) and seeds \textsc{ViennaRNA}'s
generator per attempt. Validated against a $96$-process synthetic load: $4/4$ and $4/4$ targets bit-identical
for the two operators against $1/4$ and $0/4$ with the default. Independently confirmed by
\texttt{family\_grid\_par.py}, which re-runs a serially-produced seed under $50$-way parallelism and finds
per-target verdicts identical, so a $(\text{target},\text{seed})$ pair fixes the run regardless of what else is
executing.

\smallskip\noindent\textbf{How large the noise was, and why the paired reading survived it.} Pooling all six
independent executions per nominal seed (\texttt{exec\_spread.py}), the panel operator's $r_{\mathrm{mix}}$ has
a pooled within-seed s.d.\ of $\mathbf{2.04/60}$, the \emph{same order} as the execution-replicated method
effect of $2.87/60$: enough to make an unpaired comparison of two separately-reported counts meaningless, not
enough to swamp the paired one ($4.3$ standard errors from zero). The gamer is nearly deterministic ($0$--$1$, mean spread $0.3/60$), so pairing rescues the comparison though neither count is
individually stable. Two mechanisms drove the spread: subprocess timing jitter, and \emph{sweep composition},
since the ablation sweep ran six arms per worker and the paired sweep two, so the same nominal configuration got
through $87.3$ versus $130.7$ pKiss calls per target. A wall-clock-budgeted operator's result is thus a function
of what else was running beside it, which no seed records. Four consequences we state rather than hide.
Per-seed counts in Table~\ref{tab:sep} are single-execution draws, to be read with at least $\pm2/60$
granularity. A bootstrap over targets and seeds only understates uncertainty by one level, so the wall-clock
intervals quoted in \S\ref{sec:exp} use a three-level bootstrap. Every such comparison is paired \emph{within}
an execution, so this variation inflates intervals but cannot masquerade as a method effect. And no claim should
be read off a single count.

\smallskip\noindent\textbf{\CF{} floor measurement.} The floor is the \CF-confined REMC optimizer run on the
same targets under the same adjudicator, at the same nominal caps as the audited arms. It scores $0/60$ on
crossing targets, which is an empirical result at a finite budget and not a restatement of
Prop.~\ref{prop:confine}: the proposition bounds what the pseudoknot-free \emph{verifier} can certify, the
floor records what the prior \emph{policy} actually produced.

\smallskip\noindent\textbf{The three tests behind Table~\ref{tab:sep}.} Let $u^{(A)}_T\in\{0,1\}$ indicate
whether arm $A$ realizes held-out target $T$ panel-unanimously, over a pool of size $N$. \emph{Percentile
bootstrap}: a $95\%$ CI for a rate resamples the $N$ targets with replacement $B{=}20{,}000$ times.
\emph{Paired McNemar} (same targets, arms $A$ vs.\ gamer $G$): with
$b=\#\{T:u^{(A)}_T{=}1,u^{(G)}_T{=}0\}$ and $c=\#\{T:u^{(A)}_T{=}0,u^{(G)}_T{=}1\}$,
\begin{equation}\label{eq:mcnemar}
  p_{\mathrm{McN}}=\min\!\Big\{1,\ 2\!\!\sum_{k=0}^{\min(b,c)}\!\binom{b+c}{k}2^{-(b+c)}\Big\} .
\end{equation}
\emph{Fixed-baseline binomial}: against the gamer rate $p_0{=}1/60$, the one-sided $p$-value for $k$ successes
in $N$ trials is $\sum_{j=k}^{N}\binom{N}{j}p_0^{\,j}(1-p_0)^{N-j}$. The analysis is
\texttt{stats\_separation.py}, run offline on the saved per-target adjudications.

\smallskip\noindent\textbf{Statistical conventions.} Rates use a fixed denominator of the full pool: a target
on which the three predictors disagree counts as not-solved, so abstention withholds a positive claim rather
than dropping the target. Every comparison carries an effect size, never a $p$-value alone. Three levels of
variation matter (target, nominal seed, execution at a fixed seed) and are not interchangeable, so we treat the
\emph{deepest} interval run for a comparison as primary and label every interval with what was resampled. With
five nominal seeds the level-1 resampling distribution has few distinct states, so these are \emph{sensitivity
intervals} with no coverage claim, which is why the most robust statement about $\Delta_{\mathrm{gamer}}$ uses
no interval at all: the operator is strictly higher in $14$ of $15$ independent executions, and $15$ of $15$ on
the structurally disjoint pool.

\smallskip\noindent\textbf{Is the effect carried by a few lucky targets?} $r_{\mathrm{mix}}$ is a low-base-rate
conjunction, so we checked. Deleting each target in turn moves the paired $\Delta$ within
$[{+}0.068,{+}0.081]$ around ${+}0.080$; the most influential single target moves it by $-0.012$. The successes
do concentrate, $25$ target-seed wins on $13$ distinct targets with the top three accounting for $44\%$, which we report instead of treating as $25$ independent wins; but they are not one or two targets, and their mean
crossing-pair count ($5.46$) matches the pool ($5.48$). Targets are also not independent draws: single-linkage
clustering at $15\%$ collapses the $60$ into $18$ clusters, one holding $29$, so we resample clusters, which
widens the interval from $[{+}0.040,{+}0.127]$ to $[{+}0.024,{+}0.149]$ and still clears zero.

\smallskip\noindent\textbf{Quasi-held-out versus exploratory, stated once.} We do not use the word
``confirmatory'': we hold no timestamped pre-registration, and a confirmatory held-out analysis would need the
evaluation pool disjoint in the respect that matters for a structure-conditioned operator, namely structure.
\emph{Quasi-held-out}: the paired $r_{\mathrm{mix}}$ separation on the $125$ ID-disjoint targets.
\emph{Exploratory}: the ShapeKnots-only comparison, the crossing-pair-count association, the budget sweep, the
six-model invention draws and the cross-source pool. No family-wise correction is applied, since these answer
different questions on overlapping data; no exploratory $p$-value here is a test, and ``significant in $k/3$
seeds'' is a descriptive replication count.

\smallskip\noindent\textbf{The multi-level bootstrap, specified in full.} \emph{Estimand}: $\Delta$, the
mean over (target, seed, execution) cells of the paired indicator difference between two arms.
\emph{Pairing}: the unit is the paired difference formed \emph{inside} one cell, so both arms carry identical
resampling weights at every level. \emph{Level 1}: draw nominal seeds with replacement. \emph{Level 2}:
executions are nested within seed; we report \emph{full-nest} (three per drawn seed, matching the design run,
primary) and \emph{one-exec} (one per drawn seed, wider, a conservative sensitivity). \emph{Level 3}: draw
targets with replacement, \emph{synchronously} for every drawn cell and both arms, reported with the target as
the unit and with the structural cluster as the unit (single linkage at $15\%$ normalized dot-bracket
distance), since the pool clusters (App.~\ref{app:data}). \emph{Interval}: percentile, $2.5$/$97.5$ quantiles
of $B{=}20{,}000$ replicates. \emph{Why no coverage claim}: with three nominal seeds the level-1 resampling
distribution has few distinct states, so these are \emph{sensitivity intervals} and we do not use them as
tests. \emph{Implementation}: Python \texttt{random} seeded at $20260726$ in \texttt{pabrep\_stats.py},
\texttt{gvp\_stats.py} and \texttt{cluster\_boot.py}.

\smallskip\noindent\textbf{Is the index split a \emph{structural} split? (An audit, with a partly negative
answer.)} Splitting by target index guarantees no target is reused, but it does not guarantee the untouched
targets are structurally distant from the touched ones, and if they were near-duplicates, ``held-out
transfer'' would reduce to template-neighbour generalization. We measured it
(\texttt{runs/cc\_ascent/\allowbreak split\_audit.py}): for each of the $125$ ID-disjoint targets
(ids
$126$--$250$) we found its nearest neighbour among the $126$ touched targets (development $[0{:}60]$, inspected
pool $[60{:}110]$, agent-invention training $[110{:}126]$) under normalized dot-bracket edit distance.
The result is a genuine caveat: the nearest-neighbour distance has median $0.077$, and
\textbf{$39/125$ ID-disjoint targets are \emph{exact} structural duplicates} of a touched target, with $94/125$
within $15\%$. So the index split is \emph{not} a structural split, and it cannot carry a transfer claim
(Table~\ref{tab:pools}).

\smallskip\noindent\emph{Two observations that bound the problem without disposing of it.} First, structural
duplication is not sequence homology here: the IUPAC restraint strings of those nearest-neighbour pairs still
differ by a median $0.329$ normalized edit distance (minimum $0.243$), so a duplicated \emph{target} does not
come with a duplicated \emph{sequence} to copy. We are careful about what this does and does not exclude. The
operator's input is the target \emph{structure}, so a differing natural sequence rules out only the narrowest leakage channel, copying a neighbour's sequence, and leaves untouched every channel that runs through the
target: objective design, mutation heuristics, template construction, the stopping criterion, error weighting,
and researcher-level choices all saw dot-bracket strings that recur in the ID-disjoint pool. ``Different
sequence'' is therefore not evidence against template leakage.

\smallskip\noindent Second, success does not visibly concentrate on the duplicates, but the sample cannot
establish that it does not (\texttt{runs/cc\_ascent/cluster\_boot.py}). Pooled over three seeds, the panel
operator's $r_{\mathrm{mix}}$ is $2/117$ ($0.017$) on the $39$ \emph{exact} duplicates, $19/165$ ($0.115$) on
the $55$ near-but-not-exact neighbours, and $9/93$ ($0.097$) on the $31$ structurally distant targets. So if
anything the exact duplicates are the \emph{worst} stratum, the opposite of the memorization signature. Taking
near and exact together against distant, and making the \emph{structural cluster} the resampling unit inside
each stratum ($18$ clusters among the $94$ near-or-exact targets, $28$ among the $31$ distant ones),
$\Delta_{\mathrm{duplicate}}=p_{\mathrm{distant}}-p_{\mathrm{near}}={+}0.022$ with a cluster bootstrap $95\%$
interval of $[{-}0.048,{+}0.108]$. That interval excludes neither a substantial positive nor a substantial
negative difference, and the strata are not difficulty-matched by design (mean crossing-pair counts $5.41$,
$5.35$, $5.55$) though they happen to be close. The correct statement is therefore the narrow one: \emph{we did
not observe higher success on the near-duplicate subset, but the structurally non-disjoint split prevents a
clean transfer claim}, and not that the failure mode is absent. This pool therefore supports an ID-held-out
evaluation only.

\smallskip\noindent\textbf{The remedy, run.} The right design assigns whole structure clusters rather than
individual targets, and \S\ref{sec:exp} reports it: clustering all $251$ targets at the
same $15\%$ threshold and keeping only clusters with no touched member yields $24$ clusters holding $27$
targets, verified at minimum distance $0.167$ from the touched set. Two results. First, about the
\emph{benchmark}: $251$ crossing targets contain only $60$ structural clusters, of which $36$ touch our
development, so Pseudobase++ cannot furnish a well-powered structurally clean transfer test for an effect at a
$\sim\!5\%$ base rate, a constraint we suspect applies to other work reporting on this benchmark and which we
therefore state plainly. Second, about the \emph{effect}: on that pool the panel operator is
mixed-panel-consistent on $1.87/27$ ($6.9\%$, i.e.\ \emph{above} its $5.3\%$ development-pool rate) against the
gamer's $0.00/27$ in all fifteen executions, $b{=}28$ against $c{=}0$. The near-duplicate concern the audit
raised is thus tested directly and does not materialize. But on $27$ targets that settles direction, not size.

\section{Robustness Checks, the Adjudicator's Reference Class, and Materials}\label{app:support}

\subsection{Robustness checks in full}\label{app:robust}
\S\ref{sec:exp} states the two checks the argument uses. This appendix adds the three details it omits.

\smallskip\noindent\textbf{The skeleton negative control, and why we do not call it specificity.} Folding each
target's \CF{} skeleton under the panel, the panel credits the crossing on $0/60$ skeletons while recovering the
nested skeleton on $57/60$, whereas a naive $A_{\mathrm{nest}}$ evaluator over-accepts them. Skeletons are one
narrow negative class, and a real negative distribution would also contain near-miss designs, mispaired
variants and \emph{alternative} pseudoknot topologies on which the panel is untested. The defensible statement
is exactly \emph{no observed false positives on the skeleton negative control}; the false-positive rate is uncalibrated, so we treat $r_{\mathrm{mix}}$ as a high-stringency consistency measure and not a bound.

\smallskip\noindent\textbf{The cross-source $0/60$, per oracle.} It does not let us blame the strictest
predictor alone: DesiRNA forms $9/60$ crossings under pKiss but only $3/60$ survive ProbKnot and $0/60$ pass
ShapeKnots, so ProbKnot and not just de-novo ShapeKnots already rejects most, for the SOTA solver as much as for
ours (our operator reaches $1$--$2/60$ under pKiss\,$\wedge$\,ProbKnot across two runs of the same arm, and
$0/60$ either way). A larger repair budget ($240$\,s, $2500$ iterations) does not change it.

\smallskip\noindent\textbf{Budget, which established nothing.} Over $\{60,180,300\}$\,s on a fixed $20$-target
subset, panel-unanimous is $\{1,3,1\}/20$: no clean budget trend, dominated by search variance, consistent with
oracle-guided search rather than a budget-monotone capability.

\paragraph{Growing-parameter extrapolation: the earlier $16$-target version.} \emph{Superseded by the
$100$-target grid of \S\ref{sec:exp}, and kept only so the record shows what was run first.} A finite target set cannot
establish a class ascent, so we tested the strictly weaker property a finite set cannot exhibit at all:
freeze the operator on small stems ($n{=}m\le4$) and apply it unchanged to a larger unseen ladder of
$\mathcal L_{\mathrm H}$ (to $n{=}m{=}14$, length $68$\,nt). The \CF{} floor is $0/16$ throughout and the
frozen operator realizes $13/16$ under pKiss$\,\wedge\,$ProbKnot including the whole diagonal
$n{=}m{=}5\ldots14$, but strict three-oracle unanimity is $6/16$, and a finite program with
length-parameterized templates could produce exactly this. The result is therefore consistent with
unbounded-family coverage and equally consistent with a template, does not distinguish an oracle-agreement
ceiling from a search-budget limit, and no claim in this paper rests on it. We record it so that its absence
from the main text is not selective reporting.

\subsection{Adjudicator choice, abstention scoring, and what ShapeKnots costs}\label{app:adjud}
The choice of held-out adjudicator deserves scrutiny because the whole positive side rests on it. ShapeKnots is
the only crossing-aware predictor in the panel that neither operator optimizes, so it is the only one for which
``held out'' is true by construction rather than by convention. It also runs here in an \emph{unvalidated
de-novo mode}: the method was designed to fold under experimental SHAPE constraints
\citep{hajdin2013shapeknots} and we supply none, so we are using it outside the regime it was calibrated for
and say so wherever its numbers appear.

What that costs is stringency we cannot separate from strictness. Given its own reference class the judge is
not implausibly harsh: it recovers $11.3\%$ $[5.3,20.3]$ of native pseudoknots from their own sequences
(\texttt{heldout\_ceiling.py}) and confirms our panel operator's designs at $8.0\%$, about $70\%$ of that
ceiling, against the gamer's $3.7\%$. The low absolute rates are therefore substantially a property of the
adjudicator, and the \emph{ratio} between arms is the part we read.

\emph{Three caveats keep that an orientation rather than a corrected estimate.} The pooled interval ignores
that the same $125$ targets recur across seeds, so it is \textbf{too narrow}; the paired seed-level test is
unchanged and still separates in only one of three seeds; and natives and designs are different populations,
ours being $372/375$ H-type against $75/80$ natives carrying more than five crossing pairs. A sibling line
found precisely the unstratified version of this comparison to be confounded by topology, and our native set
carries no H-type/kissing label and only five low-crossing members, so we cannot stratify it. The
ceiling-relative reading therefore changes what $r_{\mathrm{ho}}$'s absolute size means; it does not change
its verdict.

\emph{Abstention scoring.} A target on which the three predictors disagree counts as not-solved against the
full-$60$ denominator, so abstention withholds a positive claim rather than dropping the target, and every
reported rate is conservative in the same direction. That is why $r_{\mathrm{mix}}$ is a high-stringency
consistency measure rather than a selective-prediction rate: we never report an accuracy conditioned on the
panel having agreed.

\section{Datasets, Intrinsic Grades, and Materials}\label{app:data}
Every target's difficulty \emph{grade} is a property of the target fixed before design: its \emph{formal-family
type} (nested vs.\ crossing) and, within crossing, its \emph{crossing-pair count} $g$, the number of base pairs
that cross another, computed in $O(|P|^2)$ from the dot-bracket. Table~\ref{tab:appdata} lists the pools.
\begin{table}[h]\centering\small
\setlength{\tabcolsep}{4.5pt}\renewcommand{\arraystretch}{1.15}
\caption{Datasets. Grade $g$ = crossing-pair count. EteRNA100 is the nested (\CF) competence baseline; the
three crossing pools are the crossing capability sets; see Table~\ref{tab:pools} for what each
supports.}\label{tab:appdata}
\begin{tabularx}{\linewidth}{@{}lLlL@{}}
\toprule
pool & source & $n$ / length & grade $g$ distribution\\
\midrule
held-out (primary) & Pseudobase++ \citep{rnainvbench} idx $0$--$59$ & $60$ / $21$--$137$\,nt & $g{:}n$ ---
 $3{:}4$, $4{:}6$, $5{:}22$, $6{:}18$, $7{:}7$, $8{:}1$, $9{:}2$\\
second pool & Pseudobase++ idx $60$--$109$ & $50$ / $21$--$137$\,nt & disjoint from primary\\
cross-source & RNAInvBench \emph{inverse\_rna\_folding} (bpRNA/ArchiveII) & $60$ / $21$--$120$\,nt & $g{:}n$ ---
 $1{:}4$, $4{:}13$, $5{:}6$, $6{:}16$, $7{:}8$, $8{:}5$, $9{:}3$, $10{:}3$, $12{:}1$, $13{:}1$\\
family $\mathcal L_{\mathrm H}$ & synthetic H-type (Eq.~\eqref{eq:family}) & $10$ / $24$--$44$\,nt & $n{=}m\in\{3,\dots,8\}$ $+$ asym\\
EteRNA100-V2 & \citep{koodli2021eterna100v2} & $100$ / $12$--$400$\,nt & nested (\CF)\\
\bottomrule
\end{tabularx}
\end{table}
The cross-source pool is reconstructed to crossing dot-bracket from the pair list (crossing pairs are encoded
as non-type-$0$; a greedy multi-bracket assignment yields $\le4$ bracket levels); multiplet structures are
excluded so each base is in $\le1$ pair (\texttt{runs/cc\_ascent/bprna\_pk.py}). The two pools overlap on the crossing-pair count, so the cross-source drop
(\S\ref{sec:exp}) is not explained by that one covariate. \emph{We do not call the comparison
``grade-fair''}: overlap on a single prespecified structural covariate is not distribution matching. The two
pools are \emph{not} matched on length, pseudoknot subtype, loop-length distribution, base-pair composition,
motif family, source database, or sequence homology, any of which could drive the drop. The cross-source
result should therefore be read as a distribution-shift observation with the shift only partly characterized.

\section{The Operator-Invention Protocol}\label{app:invent}
An LLM agent (GLM-5.2, temperature $0.4$, no output cap) is given a documented \emph{primitive API} and asked to
write one function \texttt{design(target,seed)}; it may not import anything or call an external pseudoknot
solver. The primitives are \texttt{fold\_cf} (ViennaRNA MFE, nested), \texttt{inverse\_fold\_cf},
\texttt{fold\_pkiss\_pairs} and \texttt{fold\_probknot\_pairs} (crossing-aware pair sets),
\texttt{target\_pairs}, \texttt{rand\_seq} and \texttt{enforce\_pairs}. Each round the agent's code is executed
on a $16$-target training pool disjoint from every evaluation pool, the per-target outcomes are returned as
feedback, and the agent rewrites. The best round by training score is frozen and adjudicated once on the
development/evaluation $60$; the frozen source is archived at
\texttt{runs/cc\_ascent/invent\_operator/invented\_design.py}.

Provenance and its limit. Runs call frontier models through a single provider (OpenRouter) at
temperature~$0.4$. Each \texttt{invent\_log.json} records the provider-qualified model identifier, per-round
prompt and completion tokens, per-round training score, and every compile or runtime error verbatim; each frozen
\texttt{design} program is stored in full in \texttt{result.json}. What the logs cannot contain is a
provider-side version hash, because the endpoint exposes none: a frontier model behind a routed alias is not a
fixed artifact. That is a reproducibility limit of the invention step and of nothing measured in this paper,
since every number here re-executes frozen source rather than re-querying a model.

The agent's training solve rate rises as it folds oracle feedback into successive rewrites, and the frozen best
round measures $9/60$ under all three predictors. The analogous frozen results for all six frontier models are
in Table~\ref{tab:models}, and every one peaks by round $\le5$.

\begin{figure}[t]\centering
\includegraphics[width=0.94\linewidth]{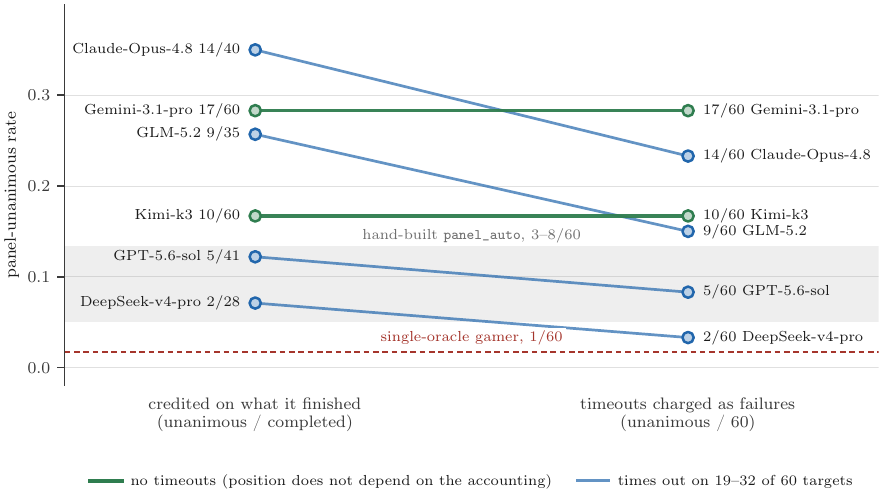}
\caption{\textbf{Why Table~\ref{tab:models} is alphabetical.} The same six agents under the two defensible
accountings of a timeout. Credit each model only on the targets it finished and Claude-Opus-4.8 leads;
charge timeouts as failures and Gemini-3.1-pro does. The crossing lines are the finding: an ordering that
reverses under a bookkeeping choice is not an ordering, so we report none. The two flat lines are the models
that never time out, and their position is the only one here that does not depend on the choice. What
survives both accountings is the single claim this appendix makes: all six clear the gamer and the decidable
floor (\texttt{model\_sensitivity.py}).}\label{fig:models}
\end{figure}

\begin{table}[t]\centering\small
\setlength{\tabcolsep}{6pt}\renewcommand{\arraystretch}{1.12}
\caption{\textbf{An agent invents the operator, across six frontier models} (three open, three closed). Each agent,
given only primitives and the two crossing oracles, writes and revises a \texttt{design} operator; the frozen
best is adjudicated by the panel on the development/evaluation $60$. Reference points: hand-built
\texttt{panel\_auto} $34$--$37$ / $3$--$8$; gamer $43$ / $1$; \CF{} floor $0$. \textbf{Rows are alphabetical,
not ranked}, and timeouts are charged as failures against the full $60$. Crediting models only on what they
finished would rank Claude-Opus-4.8 first ($14/40$); charging timeouts ranks Gemini-3.1-pro first ($17/60$).
That the order flips under an accounting change is why we claim no order, and Fig.~\ref{fig:models} is that flip. What survives both is the only claim
made here: all six clear the gamer and the floor (\texttt{model\_sensitivity.py}).}\label{tab:models}
\begin{tabularx}{\linewidth}{@{}Llccc@{}}
\toprule
agent & weights & train $/16$ & held-out pKiss$\wedge$ProbKnot & \textbf{panel-unanimous $/60$}\\
\midrule
Claude-Opus-4.8 & closed & $12$ & $40/60$ & $14$\\
DeepSeek-v4-pro & open & $5$ & $26/60$ & $2$\\
Gemini-3.1-pro & closed & $12$ & $38/60$ & $17$\\
GLM-5.2 & open & $12$ & $35/60$ & $9$\\
GPT-5.6-sol & closed & $12$ & $41/60$ & $5$\\
Kimi-k3 & open & $12$ & $42/60$ & $10$\\
\bottomrule
\end{tabularx}
\end{table}

\smallskip\noindent\textbf{Design limits of the six-model comparison.} One temperature-$0.4$ draw per model,
no seed replication, and the two models later replicated span $10$--$18$ on this metric across seeds, a range
wider than most gaps in Table~\ref{tab:models}. Four of the six also time out on $19$--$32$ of $60$ targets;
that table charges the timeouts as failures and shows the ordering flipping when they are not
(\texttt{model\_sensitivity.py}). The experiment therefore supports one statement only, that six single-run
agents produced nonzero cross-predictor consistency under a common invention protocol, and supports no ordering
among them. A model comparison would need several draws per model at several temperatures, a shared compute
budget, and a clustering of operator sources.

\section{Oracle Configurations, Hyperparameters, Compute, and Verifier Stringency}\label{app:config}
\noindent\textbf{Oracle configurations}, all deterministic and replayable with \texttt{DATAPATH} pinned to
the RNAstructure data tables: $\Omega_{\mathrm{cf}}$ is ViennaRNA Turner-2004 MFE at $37^\circ$C (env
\texttt{rna}, no timeout); pKiss is \texttt{pkiss mfe} \citep{janssen2015pkiss} (env \texttt{pk}, $60$\,s);
ProbKnot is RNAstructure MEA over the partition ensemble \citep{bellaousov2010probknot} (env \texttt{rnastr},
$60$\,s); ShapeKnots is RNAstructure MFE-based run \emph{de novo} with no SHAPE data
\citep{hajdin2013shapeknots} (env \texttt{rnastr}, $120$\,s).

\noindent\textbf{Operator hyperparameters.} REMC floor: wall $30$\,s. Design operators (fixed across all
targets, with size-independent budgets): skeleton budget $12$--$20$\,s, repair budget $90$--$120$\,s, $\le800$--$1500$
repair iterations, ProbKnot gate at $\mathrm{pk\_bp}\le3$. \textbf{Compute.} A single $256$-core host;
target-level parallelism at $\mathrm{NWORKERS}{=}60$ (the design operators are single-process, so this is
safe); the entire sweep suite is search-only and incurs \emph{no} model-inference cost, except the invention
agent (App.~\ref{app:invent}), a few dollars of GLM-5.2.

\smallskip\noindent\textbf{Positive-side stringency (the verifier ladder), and one rung beyond
panel-unanimity.} Panel-unanimity is already stronger than ``the target pairs appear somewhere'': it requires
the target to \emph{equal} each method's predicted structure, the MFE under pKiss and ShapeKnots and the
maximum-expected-accuracy structure over a partition-function ensemble under ProbKnot. Those are related but
not identical notions of a single dominant fold, since ProbKnot's MEA is not a unique-MFE criterion, so
panel-unanimity sits near the ``unique/dominant fold'' region of the
exact\,$\subseteq$\,unique-MFE\,$\subseteq$\,dominance ladder and not exactly on one rung, agreed by three predictors that are thermodynamically related rather than statistically independent.

\begin{figure}[t]\centering
\includegraphics[width=0.82\linewidth]{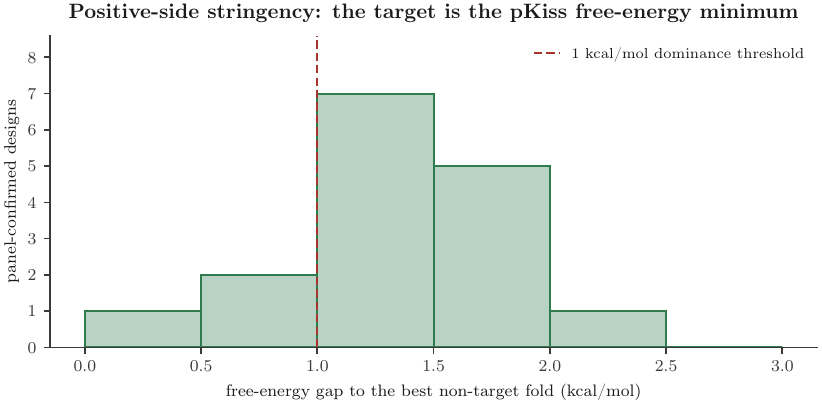}
\caption{\textbf{One rung above panel-unanimity, as a distribution rather than a median.} Free-energy gap
from the target fold to the best non-target pKiss suboptimal, for the $17$ panel-confirmed designs whose
enumeration completed. A median of $1.2$\,kcal/mol is compatible with a bimodal picture in which some designs
are marginal; the distribution shows it is not, with $14/17$ clearing $1$\,kcal and the left tail thin rather
than heavy. The scope limits stated in the text still apply: eight longer designs timed out, timeout
correlates with length, and the gap is a property of the pKiss energy model, not a model-independent
thermodynamic fact (\texttt{positive\_stringency.py}).}\label{fig:stringency}
\end{figure}

We climb one rung further and measure it (\texttt{runs/cc\_ascent/positive\_stringency.py}). Enumerating pKiss
suboptimals for the mixed-panel-consistent designs and taking the free-energy gap $\Delta$ from the target to
the best alternative fold: the target is the pKiss MFE for all $17$ designs whose enumeration completed within
budget, dominant by median $\Delta{=}1.2$\,kcal/mol (mean $1.44$), with $14/17$ clearing a $\ge1$\,kcal margin (Fig.~\ref{fig:stringency}).
Eight longer designs timed out, and since timeout correlates with length this rung covers the shorter
adjudicated designs only. The gap is a property of the pKiss energy model, not a model-independent
thermodynamic fact. Within that scope the positive verdict is not a marginal coincidence of pair sets: the
target is the free-energy-minimum fold by a real if modest margin.

Two further strengthenings are computable with the installed tooling: a
Boltzmann-probability threshold (RNAstructure's \texttt{partition} and \texttt{ProbabilityPlot}), and a
fourth, non-thermodynamic predictor. Physical confirmation by in-vitro probing is the
only route past model-relativity and is out of scope here.

\clearpage
\bibliographystyle{plainnat}
\bibliography{references}

\begin{thebibliography}{62}
\providecommand{\natexlab}[1]{#1}
\providecommand{\url}[1]{\texttt{#1}}
\expandafter\ifx\csname urlstyle\endcsname\relax
  \providecommand{\doi}[1]{doi: #1}\else
  \providecommand{\doi}{doi: \begingroup \urlstyle{rm}\Url}\fi

\bibitem[Aguirre et~al.(2011)Aguirre, Buld{\'u}, Stich, and
  Manrubia]{aguirre2011}
Jacobo Aguirre, Javier~M. Buld{\'u}, Michael Stich, and Susanna~C. Manrubia.
\newblock Topological structure of the space of phenotypes: The case of {RNA}
  neutral networks.
\newblock \emph{PLoS ONE}, 6\penalty0 (10):\penalty0 e26324, 2011.
\newblock \doi{10.1371/journal.pone.0026324}.

\bibitem[Amodei et~al.(2016)Amodei, Olah, Steinhardt, Christiano, Schulman, and
  Man{\'e}]{amodei2016concrete}
Dario Amodei, Chris Olah, Jacob Steinhardt, Paul Christiano, John Schulman, and
  Dan Man{\'e}.
\newblock Concrete problems in {AI} safety.
\newblock \emph{arXiv preprint arXiv:1606.06565}, 2016.

\bibitem[Anishchenko et~al.(2021)Anishchenko, Pellock, Chidyausiku,
  et~al.]{anishchenko2021hallucination}
Ivan Anishchenko, Samuel~J. Pellock, Tamuka~M. Chidyausiku, et~al.
\newblock De novo protein design by deep network hallucination.
\newblock \emph{Nature}, 600:\penalty0 547--552, 2021.

\bibitem[Baker and Pixley(1975)]{bakerpixley1975}
Kirby~A. Baker and Alden~F. Pixley.
\newblock Polynomial interpolation and the {Chinese} remainder theorem for
  algebraic systems.
\newblock \emph{Mathematische Zeitschrift}, 143\penalty0 (2):\penalty0
  165--174, 1975.
\newblock \doi{10.1007/BF01187059}.

\bibitem[Bellaousov and Mathews(2010)]{bellaousov2010probknot}
Stanislav Bellaousov and David~H. Mathews.
\newblock {ProbKnot}: Fast prediction of {RNA} secondary structure including
  pseudoknots.
\newblock \emph{RNA}, 16\penalty0 (10):\penalty0 1870--1880, 2010.
\newblock \doi{10.1261/rna.2125310}.

\bibitem[Bonnet et~al.(2020)Bonnet, Rz{\c{a}}{\.z}ewski, and
  Sikora]{bonnet2020hard}
{\'E}douard Bonnet, Pawe{\l} Rz{\c{a}}{\.z}ewski, and Florian Sikora.
\newblock Designing {RNA} secondary structures is hard.
\newblock \emph{Journal of Computational Biology}, 27\penalty0 (3):\penalty0
  302--316, 2020.
\newblock \doi{10.1089/cmb.2019.0420}.
\newblock prelim. RECOMB 2018; arXiv:1710.11513.

\bibitem[Busch and Backofen(2006)]{busch2006inforna}
Anke Busch and Rolf Backofen.
\newblock {INFO-RNA}---a fast approach to inverse {RNA} folding.
\newblock \emph{Bioinformatics}, 22\penalty0 (15):\penalty0 1823--1831, 2006.
\newblock \doi{10.1093/bioinformatics/btl194}.

\bibitem[Chen et~al.(2026)Chen, Mao, and Wang]{verifiereconomics}
Wenhui Chen, Qingqing Mao, and Chenghua Wang.
\newblock ``solved'' is a choice of verifier, not a fact: Deterministic
  cost--robustness frontiers of verifier choice in {RNA} and {SAT}.
\newblock Technical report, InceptLabs-Shanghai, 2026.
\newblock companion technical report.

\bibitem[Chomsky(1956)]{chomsky1956}
Noam Chomsky.
\newblock Three models for the description of language.
\newblock \emph{IRE Transactions on Information Theory}, 2\penalty0
  (3):\penalty0 113--124, 1956.
\newblock \doi{10.1109/TIT.1956.1056813}.

\bibitem[Del{\'e}tang et~al.(2023)Del{\'e}tang, Ruoss, Grau-Moya, Genewein,
  Wenliang, Catt, Cundy, Hutter, Legg, Veness, and Ortega]{deletang2023}
Gr{\'e}goire Del{\'e}tang, Anian Ruoss, Jordi Grau-Moya, Tim Genewein, Li~Kevin
  Wenliang, Elliot Catt, Chris Cundy, Marcus Hutter, Shane Legg, Joel Veness,
  and Pedro~A. Ortega.
\newblock Neural networks and the {Chomsky} hierarchy.
\newblock In \emph{ICLR}, 2023.
\newblock arXiv:2207.02098.

\bibitem[Dowell and Eddy(2004)]{dowelleddy2004}
Robin~D. Dowell and Sean~R. Eddy.
\newblock Evaluation of several lightweight stochastic context-free grammars
  for {RNA} secondary structure prediction.
\newblock \emph{BMC Bioinformatics}, 5:\penalty0 71, 2004.
\newblock \doi{10.1186/1471-2105-5-71}.

\bibitem[Durbin et~al.(1998)Durbin, Eddy, Krogh, and Mitchison]{durbin1998}
Richard Durbin, Sean~R. Eddy, Anders Krogh, and Graeme Mitchison.
\newblock \emph{Biological Sequence Analysis}.
\newblock Cambridge University Press, 1998.

\bibitem[Fawzi et~al.(2022)Fawzi, Balog, Huang, Hubert, Romera-Paredes,
  Barekatain, Novikov, Ruiz, Schrittwieser, Swirszcz, Silver, Hassabis, and
  Kohli]{alphatensor}
Alhussein Fawzi, Matej Balog, Aja Huang, Thomas Hubert, Bernardino
  Romera-Paredes, Mohammadamin Barekatain, Alexander Novikov, Francisco J.~R.
  Ruiz, Julian Schrittwieser, Grzegorz Swirszcz, David Silver, Demis Hassabis,
  and Pushmeet Kohli.
\newblock Discovering faster matrix multiplication algorithms with
  reinforcement learning.
\newblock \emph{Nature}, 610\penalty0 (7930):\penalty0 47--53, 2022.
\newblock \doi{10.1038/s41586-022-05172-4}.

\bibitem[Gao et~al.(2023)Gao, Schulman, and Hilton]{gao2023overopt}
Leo Gao, John Schulman, and Jacob Hilton.
\newblock Scaling laws for reward model overoptimization.
\newblock In \emph{ICML}, PMLR 202, pages 10835--10866, 2023.
\newblock arXiv:2210.10760.

\bibitem[Gottweis et~al.(2025)Gottweis, Weng, Daryin, et~al.]{coscientist2025}
Juraj Gottweis, Wei-Hung Weng, Alexander Daryin, et~al.
\newblock Towards an {AI} co-scientist, 2025.
\newblock arXiv:2502.18864.

\bibitem[Gr{\"u}nwald(2007)]{grunwald2007mdl}
Peter~D. Gr{\"u}nwald.
\newblock \emph{The Minimum Description Length Principle}.
\newblock MIT Press, 2007.

\bibitem[Hajdin et~al.(2013)Hajdin, Bellaousov, Huggins, Leonard, Mathews, and
  Weeks]{hajdin2013shapeknots}
Christine~E. Hajdin, Stanislav Bellaousov, Wayne Huggins, Christopher~W.
  Leonard, David~H. Mathews, and Kevin~M. Weeks.
\newblock Accurate {SHAPE}-directed {RNA} secondary structure modeling,
  including pseudoknots.
\newblock \emph{PNAS}, 110\penalty0 (14):\penalty0 5498--5503, 2013.
\newblock \doi{10.1073/pnas.1219988110}.

\bibitem[Hopcroft and Ullman(1979)]{hopcroftullman1979}
John~E. Hopcroft and Jeffrey~D. Ullman.
\newblock \emph{Introduction to Automata Theory, Languages, and Computation}.
\newblock Addison-Wesley, 1979.

\bibitem[Huang et~al.(2025)Huang, Jin, Li, Li, Cand{\`e}s, and
  Leskovec]{huang2025popper}
Kexin Huang, Ying Jin, Ryan Li, Michael~Y. Li, Emmanuel Cand{\`e}s, and Jure
  Leskovec.
\newblock Automated hypothesis validation with agentic sequential
  falsifications.
\newblock \emph{arXiv preprint arXiv:2502.09858}, 2025.
\newblock ICML 2025.

\bibitem[Janssen and Giegerich(2015)]{janssen2015pkiss}
Stefan Janssen and Robert Giegerich.
\newblock The {RNA} shapes studio.
\newblock \emph{Bioinformatics}, 31\penalty0 (3):\penalty0 423--425, 2015.
\newblock \doi{10.1093/bioinformatics/btu649}.

\bibitem[Joshi(1985)]{joshi1985tag}
Aravind~K. Joshi.
\newblock Tree adjoining grammars: How much context-sensitivity is required to
  provide reasonable structural descriptions?
\newblock In David~R. Dowty, Lauri Karttunen, and Arnold~M. Zwicky, editors,
  \emph{Natural Language Parsing}, pages 206--250. Cambridge University Press,
  1985.

\bibitem[Kato et~al.(2006)Kato, Seki, and Kasami]{katosekikasami2006}
Yuki Kato, Hiroyuki Seki, and Tadao Kasami.
\newblock {RNA} pseudoknotted structure prediction using stochastic multiple
  context-free grammar.
\newblock \emph{IPSJ Digital Courier}, 2:\penalty0 655--664, 2006.
\newblock \doi{10.2197/ipsjdc.2.655}.

\bibitem[Knudsen and Hein(1999)]{knudsenhein1999}
Bjarne Knudsen and Jotun Hein.
\newblock {RNA} secondary structure prediction using stochastic context-free
  grammars and evolutionary history.
\newblock \emph{Bioinformatics}, 15\penalty0 (6):\penalty0 446--454, 1999.
\newblock \doi{10.1093/bioinformatics/15.6.446}.

\bibitem[Koodli et~al.(2021)Koodli, Rudolfs, Wayment-Steele, and
  Das]{koodli2021eterna100v2}
Rohan~V. Koodli, Boris Rudolfs, Hannah~K. Wayment-Steele, and Rhiju Das.
\newblock Redesigning the {Eterna100} for the {Vienna} 2 folding engine.
\newblock \emph{bioRxiv}, 2021.
\newblock \doi{10.1101/2021.08.26.457839}.
\newblock preprint.

\bibitem[Koohestani et~al.(2025)Koohestani, Li, Podkopaev, and
  Izadi]{agentsautomata2025}
Roham Koohestani, Ziyou Li, Anton Podkopaev, and Maliheh Izadi.
\newblock Are agents just automata? on the formal equivalence between agentic
  {AI} and the {Chomsky} hierarchy.
\newblock \emph{arXiv preprint arXiv:2510.23487}, 2025.

\bibitem[Lakatos(1978)]{lakatos1978}
Imre Lakatos.
\newblock \emph{The Methodology of Scientific Research Programmes}.
\newblock Cambridge University Press, 1978.

\bibitem[Lightman et~al.(2023)Lightman, Kosaraju, Burda, Edwards, Baker, Lee,
  Leike, Schulman, Sutskever, and Cobbe]{lightman2023verify}
Hunter Lightman, Vineet Kosaraju, Yura Burda, Harri Edwards, Bowen Baker, Teddy
  Lee, Jan Leike, John Schulman, Ilya Sutskever, and Karl Cobbe.
\newblock Let's verify step by step.
\newblock \emph{arXiv preprint arXiv:2305.20050}, 2023.
\newblock ICLR 2024.

\bibitem[Liu et~al.(2024)Liu, Liu, Zhu, Lei, Yang, Zhang, Li, and
  Liu]{liu2024aigs}
Zijun Liu, Kaiming Liu, Yiqi Zhu, Xuanyu Lei, Zonghan Yang, Zhenhe Zhang, Peng
  Li, and Yang Liu.
\newblock {AIGS}: Generating science from {AI}-powered automated falsification.
\newblock \emph{arXiv preprint arXiv:2411.11910}, 2024.

\bibitem[Lorenz et~al.(2011)Lorenz, Bernhart, H{\"o}ner~zu Siederdissen, Tafer,
  Flamm, Stadler, and Hofacker]{lorenz2011}
Ronny Lorenz, Stephan~H. Bernhart, Christian H{\"o}ner~zu Siederdissen, Hakim
  Tafer, Christoph Flamm, Peter~F. Stadler, and Ivo~L. Hofacker.
\newblock {ViennaRNA} package 2.0.
\newblock \emph{Algorithms for Molecular Biology}, 6\penalty0 (1):\penalty0 26,
  2011.
\newblock \doi{10.1186/1748-7188-6-26}.

\bibitem[Lyngs{\o} and Pedersen(2000)]{lyngso2000pknot}
Rune~B. Lyngs{\o} and Christian N.~S. Pedersen.
\newblock {RNA} pseudoknot prediction in energy-based models.
\newblock \emph{Journal of Computational Biology}, 7\penalty0 (3-4):\penalty0
  409--427, 2000.
\newblock \doi{10.1089/106652700750050862}.

\bibitem[Mankowitz et~al.(2023)Mankowitz, Michi, Zhernov, Gelmi, Selvi,
  Paduraru, et~al.]{alphadev}
Daniel~J. Mankowitz, Andrea Michi, Anton Zhernov, Marco Gelmi, Marco Selvi,
  Cosmin Paduraru, et~al.
\newblock Faster sorting algorithms discovered using deep reinforcement
  learning.
\newblock \emph{Nature}, 618\penalty0 (7964):\penalty0 257--263, 2023.
\newblock \doi{10.1038/s41586-023-06004-9}.

\bibitem[McCaskill(1990)]{mccaskill1990}
John~S. McCaskill.
\newblock The equilibrium partition function and base pair binding
  probabilities for {RNA} secondary structure.
\newblock \emph{Biopolymers}, 29\penalty0 (6-7):\penalty0 1105--1119, 1990.
\newblock \doi{10.1002/bip.360290621}.

\bibitem[Nebel and Weinberg(2012)]{nebelweinberg}
Markus~E. Nebel and Frank Weinberg.
\newblock Algebraic and combinatorial properties of common {RNA} pseudoknot
  classes.
\newblock \emph{Journal of Computational Biology}, 19\penalty0 (10):\penalty0
  1134--1150, 2012.
\newblock PMC3469209.

\bibitem[Novikov et~al.(2025)Novikov, V{\~u}, Eisenberger, Dupont, Huang,
  Wagner, Shirobokov, et~al.]{alphaevolve}
Alexander Novikov, Ng{\^a}n V{\~u}, Marvin Eisenberger, Emilien Dupont, Po-Sen
  Huang, Adam~Zsolt Wagner, Sergey Shirobokov, et~al.
\newblock {AlphaEvolve}: A coding agent for scientific and algorithmic
  discovery.
\newblock \emph{arXiv preprint arXiv:2506.13131}, 2025.

\bibitem[Nussinov and Jacobson(1980)]{nussinov1980}
Ruth Nussinov and Ann~B. Jacobson.
\newblock Fast algorithm for predicting the secondary structure of
  single-stranded {RNA}.
\newblock \emph{PNAS}, 77\penalty0 (11):\penalty0 6309--6313, 1980.
\newblock \doi{10.1073/pnas.77.11.6309}.

\bibitem[Ogden(1968)]{ogden1968}
William Ogden.
\newblock A helpful result for proving inherent ambiguity.
\newblock \emph{Mathematical Systems Theory}, 2\penalty0 (3):\penalty0
  191--194, 1968.
\newblock \doi{10.1007/BF01694004}.

\bibitem[Popper(1959)]{popper1959}
Karl~R. Popper.
\newblock \emph{The Logic of Scientific Discovery}.
\newblock Hutchinson, London, 1959.
\newblock transl. of Logik der Forschung, 1934.

\bibitem[Reeder and Giegerich(2004)]{reedergiegerich2004}
Jens Reeder and Robert Giegerich.
\newblock Design, implementation and evaluation of a practical pseudoknot
  folding algorithm based on thermodynamics.
\newblock \emph{BMC Bioinformatics}, 5:\penalty0 104, 2004.
\newblock \doi{10.1186/1471-2105-5-104}.

\bibitem[Rivas and Eddy(1999)]{rivas1999pknots}
Elena Rivas and Sean~R. Eddy.
\newblock A dynamic programming algorithm for {RNA} structure prediction
  including pseudoknots.
\newblock \emph{Journal of Molecular Biology}, 285\penalty0 (5):\penalty0
  2053--2068, 1999.
\newblock \doi{10.1006/jmbi.1998.2436}.

\bibitem[Rivas and Eddy(2000)]{rivaseddy2000}
Elena Rivas and Sean~R. Eddy.
\newblock The language of {RNA}: A formal grammar that includes pseudoknots.
\newblock \emph{Bioinformatics}, 16\penalty0 (4):\penalty0 334--340, 2000.
\newblock \doi{10.1093/bioinformatics/16.4.334}.

\bibitem[RNAInvBench()]{rnainvbench}
RNAInvBench.
\newblock {RNAInvBench} and {Pseudobase++}: Pseudoknotted {RNA} inverse-design
  targets, 2024.
\newblock inverse-design benchmark suite.

\bibitem[Romera-Paredes et~al.(2024)Romera-Paredes, Barekatain, Novikov, Balog,
  Kumar, Dupont, Ruiz, Ellenberg, Wang, Fawzi, Kohli, and Fawzi]{funsearch}
Bernardino Romera-Paredes, Mohammadamin Barekatain, Alexander Novikov, Matej
  Balog, M.~Pawan Kumar, Emilien Dupont, Francisco J.~R. Ruiz, Jordan~S.
  Ellenberg, Pengming Wang, Omar Fawzi, Pushmeet Kohli, and Alhussein Fawzi.
\newblock Mathematical discoveries from program search with large language
  models.
\newblock \emph{Nature}, 625\penalty0 (7995):\penalty0 468--475, 2024.
\newblock \doi{10.1038/s41586-023-06924-6}.

\bibitem[Runge et~al.(2019)Runge, Stoll, Falkner, and Hutter]{runge2019learna}
Frederic Runge, Danny Stoll, Stefan Falkner, and Frank Hutter.
\newblock Learning to design {RNA}.
\newblock In \emph{ICLR}, 2019.
\newblock arXiv:1812.11951.

\bibitem[Sakakibara et~al.(1994)Sakakibara, Brown, Hughey, Mian, Sj{\"o}lander,
  Underwood, and Haussler]{sakakibara1994}
Yasubumi Sakakibara, Michael Brown, Richard Hughey, I.~Saira Mian, Kimmen
  Sj{\"o}lander, Rebecca~C. Underwood, and David Haussler.
\newblock Stochastic context-free grammars for {tRNA} modeling.
\newblock \emph{Nucleic Acids Research}, 22\penalty0 (23):\penalty0 5112--5120,
  1994.
\newblock \doi{10.1093/nar/22.23.5112}.

\bibitem[Schaefer(1978)]{schaefer1978}
Thomas~J. Schaefer.
\newblock The complexity of satisfiability problems.
\newblock In \emph{STOC}, pages 216--226, 1978.
\newblock \doi{10.1145/800133.804350}.

\bibitem[Schuster et~al.(1994)Schuster, Fontana, Stadler, and
  Hofacker]{schuster1994}
Peter Schuster, Walter Fontana, Peter~F. Stadler, and Ivo~L. Hofacker.
\newblock From sequences to shapes and back: A case study in {RNA} secondary
  structures.
\newblock \emph{Proceedings of the Royal Society of London B}, 255\penalty0
  (1344):\penalty0 279--284, 1994.
\newblock \doi{10.1098/rspb.1994.0040}.

\bibitem[Searls(1992)]{searls1992}
David~B. Searls.
\newblock The linguistics of {DNA}.
\newblock \emph{American Scientist}, 80\penalty0 (6):\penalty0 579--591, 1992.

\bibitem[Searls(2002)]{searls2002}
David~B. Searls.
\newblock The language of genes.
\newblock \emph{Nature}, 420\penalty0 (6912):\penalty0 211--217, 2002.
\newblock \doi{10.1038/nature01255}.

\bibitem[Seki et~al.(1991)Seki, Matsumura, Fujii, and Kasami]{seki1991mcfg}
Hiroyuki Seki, Takashi Matsumura, Mamoru Fujii, and Tadao Kasami.
\newblock On multiple context-free grammars.
\newblock \emph{Theoretical Computer Science}, 88\penalty0 (2):\penalty0
  191--229, 1991.
\newblock \doi{10.1016/0304-3975(91)90374-B}.

\bibitem[Skalse et~al.(2022)Skalse, Howe, Krasheninnikov, and
  Krueger]{skalse2022}
Joar Skalse, Nikolaus H.~R. Howe, Dmitrii Krasheninnikov, and David Krueger.
\newblock Defining and characterizing reward hacking.
\newblock In \emph{NeurIPS}, volume~35, 2022.
\newblock arXiv:2209.13085.

\bibitem[Vijay-Shanker and Weir(1994)]{vijayshanker1994}
K.~Vijay-Shanker and David~J. Weir.
\newblock The equivalence of four extensions of context-free grammars.
\newblock \emph{Mathematical Systems Theory}, 27\penalty0 (6):\penalty0
  511--546, 1994.
\newblock \doi{10.1007/BF01191624}.

\bibitem[Wang and Buehler(2026)]{selfrevising}
Fiona~Y. Wang and Markus~J. Buehler.
\newblock Self-revising discovery systems for science: A categorical framework
  for agentic {AI}.
\newblock \emph{arXiv preprint arXiv:2606.01444}, 2026.

\bibitem[Wang et~al.(2023)Wang, Xie, Jiang, et~al.]{voyager2023}
Guanzhi Wang, Yuqi Xie, Yunfan Jiang, et~al.
\newblock Voyager: An open-ended embodied agent with large language models,
  2023.
\newblock arXiv:2305.16291.

\bibitem[Watson et~al.(2023)Watson, Juergens, Bennett,
  et~al.]{watson2023rfdiffusion}
Joseph~L. Watson, David Juergens, Nathaniel~R. Bennett, et~al.
\newblock De novo design of protein structure and function with {RFdiffusion}.
\newblock \emph{Nature}, 620:\penalty0 1089--1100, 2023.

\bibitem[Wei et~al.(2025)Wei, Yang, Zhang, et~al.]{agenticscience2025survey}
Jiaqi Wei, Yuejin Yang, Xiang Zhang, et~al.
\newblock From {AI} for science to agentic science: A survey on autonomous
  scientific discovery, 2025.
\newblock arXiv:2508.14111.

\bibitem[Wirecki et~al.(2025)Wirecki, Lach, Badepally, Moafinejad, Jaryani,
  Klaudel, Nec, Baulin, and Bujnicki]{wirecki2025desirna}
Tomasz~K. Wirecki, Grzegorz Lach, Nagendar~Goud Badepally, S.~Naeim Moafinejad,
  Farhang Jaryani, Gaja Klaudel, Kalina Nec, Eugene~F. Baulin, and Janusz~M.
  Bujnicki.
\newblock {DesiRNA}: Structure-based design of {RNA} sequences with a replica
  exchange monte carlo approach.
\newblock \emph{Nucleic Acids Research}, 53\penalty0 (2), 2025.
\newblock \doi{10.1093/nar/gkae1306}.

\bibitem[Xu et~al.(2025)]{falsify2025}
Ce~Xu et~al.
\newblock Position: Falsify, don't just discover---{AI}-generated discoveries
  are not born scientific.
\newblock In \emph{ICML, Position Paper Track}, 2025.
\newblock OpenReview SlgXCLZFj3.

\bibitem[Zadeh et~al.(2011)Zadeh, Steenberg, Bois, Wolfe, Pierce, Khan, Dirks,
  and Pierce]{zadeh2011nupack}
Joseph~N. Zadeh, Conrad~D. Steenberg, Justin~S. Bois, Brian~R. Wolfe,
  Marshall~B. Pierce, Asif~R. Khan, Robert~M. Dirks, and Niles~A. Pierce.
\newblock {NUPACK}: Analysis and design of nucleic acid systems.
\newblock \emph{Journal of Computational Chemistry}, 32\penalty0 (1):\penalty0
  170--173, 2011.
\newblock \doi{10.1002/jcc.21596}.

\bibitem[Zhou et~al.(2023)Zhou, Dai, Li, Ward, Mathews, and
  Huang]{zhou2023samfeo}
Tianshuo Zhou, Ning Dai, Sizhen Li, Max Ward, David~H. Mathews, and Liang
  Huang.
\newblock {RNA} design via structure-aware multifrontier ensemble optimization.
\newblock \emph{Bioinformatics}, 39\penalty0 (Supplement 1):\penalty0
  i563--i571, 2023.
\newblock \doi{10.1093/bioinformatics/btad252}.

\bibitem[Zhou et~al.(2024)Zhou, Tang, Mathews, and Huang]{zhou2024rigende}
Tianshuo Zhou, Wei~Yu Tang, David~H. Mathews, and Liang Huang.
\newblock Undesignable {RNA} structure identification via rival structure
  generation and structure decomposition.
\newblock In \emph{RECOMB}, LNCS 14758, pages 270--287, 2024.
\newblock \doi{10.1007/978-1-0716-3989-4_17}.
\newblock arXiv:2311.08339.

\bibitem[Zhou et~al.(2025)Zhou, Malik, Tang, Mathews, and
  Huang]{zhou2025motifs}
Tianshuo Zhou, Apoorv Malik, Wei~Yu Tang, David~H. Mathews, and Liang Huang.
\newblock Scalable and interpretable identification of minimal undesignable
  {RNA} structure motifs with rotational invariance.
\newblock In \emph{RECOMB}, 2025.
\newblock arXiv:2402.17206.

\bibitem[Zuker and Stiegler(1981)]{zuker1981}
Michael Zuker and Patrick Stiegler.
\newblock Optimal computer folding of large {RNA} sequences using
  thermodynamics and auxiliary information.
\newblock \emph{Nucleic Acids Research}, 9\penalty0 (1):\penalty0 133--148,
  1981.
\newblock \doi{10.1093/nar/9.1.133}.

\end{thebibliography}

\end{document}